\documentclass{article}
\usepackage{iclr2024_conference}
\usepackage{times}

\usepackage{listings}
\usepackage{color}

\definecolor{dkgreen}{rgb}{0,0.6,0}
\definecolor{gray}{rgb}{0.5,0.5,0.5}
\definecolor{mauve}{rgb}{0.58,0,0.82}

\usepackage{hyperref}
\usepackage{xcolor}
\hypersetup{
    colorlinks,
    linkcolor={red!50!black},
    citecolor={blue!50!black},
    urlcolor={blue!80!black}
}

\usepackage{macros_packages}
\usepackage{booktabs}

\usepackage{caption}
\usepackage{subcaption}

\title{Masked Completion via Structured Diffusion with White-Box Transformers}

\author{%
    Druv Pai\thanks{Correspondence to: Druv Pai, \texttt{druvpai@berkeley.edu}.} \\ UC Berkeley
    \And
    Ziyang Wu \\ UC Berkeley
    \And
    Sam Buchanan \\ TTIC
    \And
    Yaodong Yu \\ UC Berkeley
    \And
    Yi Ma \\ UC Berkeley \& HKU %
}

\iclrfinalcopy %
\begin{document}

\maketitle

\begin{abstract}
    Modern learning frameworks often train deep neural networks with massive amounts of unlabeled data to learn representations by solving simple pretext tasks, then use the representations as foundations for downstream tasks. These networks are empirically designed; as such, they are usually not interpretable, their representations are not structured, and their designs are potentially redundant. White-box deep networks, in which each layer explicitly identifies and transforms structures in the data, present a promising alternative. However, existing white-box architectures have only been shown to work at scale in supervised settings with labeled data, such as classification. In this work, we provide the first instantiation of the white-box design paradigm that can be applied to large-scale unsupervised representation learning. We do this by exploiting a fundamental connection between diffusion, compression, and (masked) completion, deriving a deep transformer-like masked autoencoder architecture, called \ours{}, in which the role of each layer is mathematically fully interpretable: they transform the data distribution to and from a structured representation. Extensive empirical evaluations confirm our analytical insights. \ours{} demonstrates highly promising performance on large-scale imagery datasets while using only \(\sim\)30\% of the parameters compared to the standard masked autoencoder with the same model configuration. The representations learned by \ours{} have explicit structure and also contain semantic meaning. Code is available on \href{https://github.com/Ma-Lab-Berkeley/CRATE}{GitHub}.
\end{abstract}

\vspace{-0.25in}
\section{Introduction}\label{sec:intro}

In recent years, deep learning has been called upon to process continually larger quantities of high-dimensional, noisy, and unlabeled data. A key property which makes these ever-larger tasks tractable is that the high-dimensional data tends to have \textit{low-dimensional geometric and statistical structure}. Modern deep networks tend to learn (implicit or explicit) representations of this structure, which are then used to efficiently perform downstream tasks. Learning these representations is thus of central importance in machine learning, and there are so far several common methodologies for this task. 
We focus our attention below on approaches that incrementally transform the data towards the end representation with \textit{simple, mathematically-interpretable primitives}. Discussion of popular alternatives is postponed to \Cref{app:related_work}.

\textbf{Denoising-diffusion models for high-dimensional data.} A popular method for learning implicit representations of high-dimensional data is \textit{learning to denoise}: a model that can \textit{denoise}, i.e., remove noise from a corrupted 
observation from the data distribution (to the extent information-theoretically possible), can be chained across noise levels to transform the data distribution to and from certain highly structured distributions, such as an isotropic Gaussian,
enabling efficient sampling
\citep{Hyvarinen2005-fi,Vincent2011-dr,Sohl-Dickstein2015-kz,ho2020denoising,NEURIPS2021_6e289439,Song2020-xo,Song2023-rs}.
Crucially, in the case of data with low-dimensional structure---including the highly nonlinear structure characteristic of natural images---these models can be learned efficiently
\citep{chen2023score,oko2023diffusion,Moitra2020-lb}, and as a result this
framework has significant practical impact \citep{Rombach2022-he}. Despite this
progress, these techniques have been largely limited to use in generative
modeling; a key reason is the unstructured nature of the final `noisy' state of
the diffusion process, which makes it challenging to control and interpret the
model's learned implicit representation of the data.

\textbf{White-box models and structured representation learning.} In contrast,
\textit{white-box} models are designed to produce explicit and structured
representations of the data distribution according to a desired parsimonious
configuration, such as sparsity \citep{gregor2010learning,zhai2020complete} or
(piecewise) linearity \citep{chan2022redunet}. Recent work
\citep{chan2022redunet,Yu2023-ig} has built white-box deep networks via unrolled
optimization: namely, to obtain representations with a desired set of
properties, one constructs an objective function which encourages these
desiderata, then constructs a deep network where each layer is designed to
iteratively optimize the objective. This builds deep networks as a chain of
operators, representing well-understood optimization primitives, which
iteratively transform the representations to the desired structure. For example,
\citet{Yu2023-ig} uses an information-theoretic objective promoting lossy
compression of the data towards a fixed statistical structure to build
a transformer-like architecture named \oursbase{} in the above manner. However,
such-obtained deep networks have yet to be constructed for most unsupervised
contexts. The fundamental difficulty here is that decoder networks must map from
representations to data, and hence \textit{invert} (in a distributional sense)
the transformations made to the data distribution by the encoder. This renders
the unrolled optimization approach used to construct white-box encoders such as
\oursbase{} infeasible for constructing decoders, and instead demands
a fine-grained understanding of the operators that implement the encoder, and
their (distributional) inverses.

\textbf{Our contributions.} To overcome this difficulty and extend
the applicability of white-box models to unsupervised settings, we demonstrate
in this work that these two
paradigms have more in common than previously appreciated. First, we show quantitatively that under certain natural regimes, \textit{denoising} and \textit{compression} are highly similar primitive data processing operations: when the target distribution has low-dimensional structure, both operations implement a projection operation onto this structure. Second, using this insight, we demonstrate a quantitative connection between unrolled discretized diffusion models and unrolled optimization-constructed deep networks.
This leads to a significant expansion of the existing conceptual toolkit for developing white-box neural network architectures, which we use to derive white-box transformer-like encoder and decoder architectures that together form an autoencoding model that we call \ours{}, illustrated in  Fig.\ref{fig:crate_mae_pipeline}. We evaluate \ours{} on the challenging masked autoencoding task \citep{he2022masked} and demonstrate promising performance with large parameter savings over traditional masked autoencoders, along with many side benefits such as emergence of semantic meaning in the representations.

\section{Approach}\label{sec:approach}
\subsection{Setup and Notation}\label{sub:setup}
We use the same notation and basic problem setup as in \cite{Yu2023-ig}. Namely, we have some matrix-valued random variable \(\vX = \mat{\vx_{1},\dots,\vx_{N}} \in \R^{D \times N}\) representing the data, where the \(\vx_{i} \in \R^{D}\) are called ``tokens'' and may be arbitrarily correlated. To obtain representations of the input, we learn an \textit{encoder} \(f \colon \R^{D \times N} \to \R^{d \times N}\); our representations are denoted by the random variable \(\vZ = f(\vX) = \mat{\vz_{1}, \dots, \vz_{N}} \in \R^{d \times N}\), where the token representations are \(\vz_{i} \in \R^{d}\). In the autoencoding setup, we also learn a \textit{decoder} \(g \colon \R^{d \times N} \to \R^{D \times N}\), such that \(\vX \approx \vXh = [\vxh_{1}, \dots, \vxh_{N}] \doteq g(\vZ)\).

Our encoder and decoder will be deep neural networks, and as such they will be composed of several, say \(L\), \textit{layers} each. Write \(f = f^{L} \circ \cdots \circ f^{1} \circ f^{\pre}\) and \(g = g^{\post} \circ g^{L - 1} \circ \cdots \circ g^{0}\), where \(f^{\ell} \colon \R^{d \times N} \to \R^{d \times N}\) and \(g^{\ell} \colon \R^{d \times N} \to \R^{d \times N}\) are the \(\ell\th\) layer of the encoder and decoder respectively, and \(f^{\pre} \colon \R^{D \times N} \to \R^{d \times N}\) and \(g^{\post} \colon \R^{d \times N} \to \R^{D \times N}\) are the pre-~and post-processing layers respectively. The \textit{input} to the \(\ell\th\) layer of the encoder is denoted \(\vZ^{\ell} \doteq \mat{\vz_{1}^{\ell}, \dots, \vz_{N}^{\ell}} \in \R^{d \times N}\), and the \textit{input} to the \(\ell\th\) layer of the decoder is denoted \(\vY^{\ell} \doteq \mat{\vy_{1}^{\ell}, \dots, \vy_{N}^{\ell}} \in \R^{d \times N}\). 

\begin{figure}[t!]
    \begin{center}
        \includegraphics[width=0.95\textwidth]{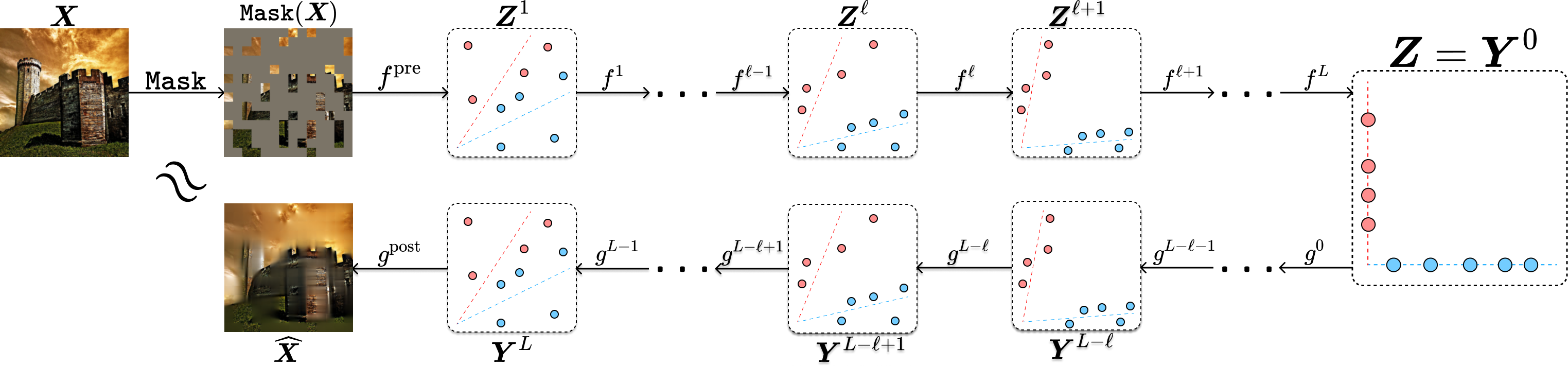}
    \end{center}
    \caption{\small \textbf{Diagram of the overall  white-box \ours{} pipeline, illustrating the end-to-end (masked) autoencoding process.} The token representations are transformed iteratively towards a parsimonious (e.g., compressed and sparse) representation by each encoder layer \(f^{\ell}\). Furthermore, such representations are transformed back to the original image by the decoder layers \(g^{\ell}\). Each encoder layer \(f^{\ell}\) is meant to be (partially) inverted by a corresponding decoder layer \(g^{L - \ell}\).}
    \label{fig:crate_mae_pipeline}
\end{figure}

\subsection{Desiderata, Objective, and Optimization}\label{sub:objective}
Our goal is to use the encoder \(f\) and decoder \(g\) to learn \textit{representations} \(\vZ\) which are \textit{parsimonious} \citep{ma2022principles} and \textit{invertible}; namely, they have \textit{low-dimensional, sparse, (piecewise) linear} geometric and statistical structure, and are (approximately) bijective with the original data \(\vX\).
\citet{Yu2023-ig} proposes to implement these desiderata by positing a \textit{signal model} for the representations:

\begin{codebook}
    Let \(\vZ = \mat{\vz_{1}, \dots, \vz_{N}} \in \R^{d \times N}\) be a random matrix. We impose the following statistical model on \({\vZ}\), parameterized by orthonormal bases \(\vU_{[K]} = (\vU_{k})_{k \in [K]} \in (\R^{d \times p})^{K}\): each token \({\vz}_{i}\) has marginal distribution given by
    \begin{equation}\label{model:gaussian_tokens}
        \vz_{i} \stackrel{d}{=} \vU_{s_{i}}\valpha_{i}, \quad \forall i \in [N]
    \end{equation}
    where \((s_{i})_{i \in [N]} \in [K]^{N}\) are random variables corresponding to the subspace indices, and \((\valpha_{i})_{i \in [N]} \in (\R^{p})^{N}\) are zero-mean Gaussian variables. If we optionally specify a noise parameter \(\sigma \geq 0\), we mean that we ``diffuse'' the tokens with Gaussian noise: by an abuse of notation, each token \({\vz}_{i}\) has marginal distribution given by
    \begin{equation}\label{model:gaussian_tokens_noise}
        \vz_{i} \stackrel{d}{=} \vU_{s_{i}}\valpha_{i} + \sigma \vw_{i}, \quad
        \forall i \in [N]
    \end{equation}
    where \((\vw_{i})_{i \in [N]} \in (\R^{d})^{N}\) are i.i.d.~standard Gaussian variables, independent of \(s_{i}\) and \(\valpha_{i}\).
\end{codebook}

If the \(\vU_{k}\) are sufficiently incoherent and axis-aligned, we expect such representations to maximize the \textit{sparse rate reduction} objective function \citep{Yu2023-ig}:
\begin{equation}\label{eq:sparse_rr}
    \mathbb{E}_{\vZ}[\Delta R(\vZ \mid \vU_{[K]}) - \lambda \norm{\vZ}_{0}] = \mathbb{E}_{\vZ}[R(\vZ) - R^{c}(\vZ \mid \vU_{[K]}) - \lambda \norm{\vZ}_{0}],
\end{equation}
where \(R\) and \(R^{c}\) are \textit{lossy coding rates}, or \textit{rate distortions} \citep{cover1999elements}, which are estimates for the number of bits required to encode the sample up to precision \(\epsilon > 0\) using a Gaussian codebook, both unconditionally (for \(R\)), and conditioned on the samples being drawn from \(\vU_{k}\) summed over all \(k\) (for \(R^{c}\)). Closed-form estimates \citep{ma2007segmentation,Yu2023-ig} for such rate distortions are:
\begin{align}
    R(\vZ) \label{eq:r_def}
    &= \frac{1}{2}\log\det\rp{\vI_{N} + \alpha \vZ^{\top}\vZ}, \qquad \alpha \doteq \frac{d}{N\epsilon^{2}} \\
    R^{c}(\vZ \mid \vU_{[K]}) \label{eq:rc_def}
    &= \frac{1}{2}\sum_{k = 1}^{K}\log\det\rp{\vI_{N} + \beta (\vU_{k}^{\top}\vZ)^{\top}(\vU_{k}^{\top}\vZ)}, \qquad \beta \doteq \frac{p}{N\epsilon^{2}}.
\end{align}
Notably, \(R^{c}\) is a measure of \textit{compression against our statistical
structure} --- it measures how closely the overall distribution of tokens in
\(\vZ\) fit a Gaussian mixture on \(\vU_{[K]}\). Meanwhile, the other two terms
\(R\) and \(\norm{\,\cdot\,}_{0}\) ensure non-collapse and sparsity of the representations, respectively.

Following \cite{Yu2023-ig}, one then constructs a deep network that
\textit{incrementally optimizes the sparse rate reduction} in order to transform
the data distribution towards the desired parsimonious configuration
\Cref{model:gaussian_tokens}.
Specifically, \cite{Yu2023-ig} proposed to construct the deep neural network \(f\)
as a two-step alternating optimization procedure which compresses the input against the (learned) local signal model \(\vU_{[K]}^{\ell}\) at layer \(\ell\), by taking a step of gradient descent on \(R^{c}(\vZ \mid \vU_{[K]}^{\ell})\), and subsequently taking a step of proximal gradient descent on a LASSO objective \citep{tibshirani1996regression,wright2022high} to sparsify the data in a (learned) dictionary \(\vD^{\ell} \in \R^{d \times d}\):
\begin{align}\label{eq:grad_rc_mssa}
    \vZ^{\ell + 1/2}
    &= \vZ^{\ell} + \MSSA(\vZ^{\ell} \mid \vU_{[K]}^{\ell}) \approx \vZ^{\ell} - \kappa \nabla_{\vZ}R^{c}(\vZ^{\ell} \mid \vU_{[K]}^{\ell})  \\
    \vZ^{\ell + 1}
    &= \ISTA(\vZ^{\ell + 1/2} \mid \vD^{\ell}) \approx \argmin_{\vZ \geq \vZero}\left[\frac{1}{2}\norm{\vZ^{\ell + 1/2} - \vD^{\ell}\vZ}_{2}^{2} + \lambda \norm{\vZ}_{1}\right],
\end{align}
where \(\MSSA(\cdot)\), the \textbf{M}ulti-head \textbf{S}ubspace \textbf{S}elf-\textbf{A}ttention block~\citep{Yu2023-ig}, is defined as
\begin{equation}\label{eq:mssa}
    \MSSA(\vZ \mid \vU_{[K]}) \doteq \frac{p}{N\epsilon^{2}}\mat{\vU_{1} & \cdots & \vU_{K}}\mat{(\vU_{1}^{\top}\vZ)\softmax((\vU_{1}^{\top}\vZ)^{\top}(\vU_{1}^{\top}\vZ)) \\ \vdots \\ (\vU_{K}^{\top}\vZ)\softmax((\vU_{K}^{\top}\vZ)^{\top}(\vU_{K}^{\top}\vZ))},
\end{equation}
and \(\ISTA(\cdot)\), the \textbf{I}terative \textbf{S}hrinkage-\textbf{T}hresholding \textbf{A}lgorithm block~\citep{Yu2023-ig}, is defined as 
\begin{equation}\label{eq:ista}
    \ISTA(\vZ \mid \vD) \doteq \ReLU(\vZ - \eta \vD^{\top}(\vD\vZ - \vZ) - \eta \lambda \bm{1}).
\end{equation}
The \(\MSSA\) block is exactly the same as a multi-head self-attention block in
a transformer, with the changes that the
\(\bm{Q}_{k}\)/\(\bm{K}_{k}\)/\(\bm{V}_{k}\) blocks are replaced by a single
matrix \(\vU_{k}\) in each head \(k\). The resulting layer \(f^{\ell}\) thus
bears significant resemblance to a transformer-like block, and so the
\oursbase{} model is a white-box transformer-like architecture constructed via
unrolled optimization. Such \oursbase{} models obtain competitive performance on
standard tasks while enjoying many side benefits
\citep{Yu2023-ig,yu2023emergence}, yet they have so far only been trained for
supervised learning. In the sequel, we introduce a paradigm to obtain fully white-box
networks for unsupervised learning, such as autoencoding, through
a novel understanding of the \oursbase{} model's distributional layerwise
inverse.

\begin{figure}
    \begin{center}
    \includegraphics[width=0.8\textwidth]{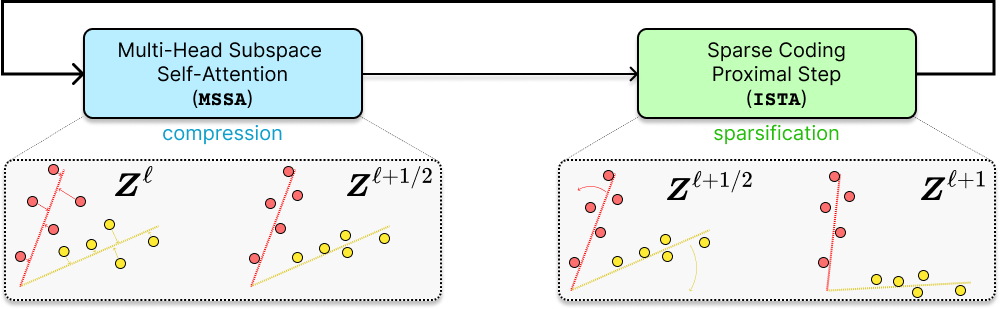}
    \end{center}
    \caption{\small \textbf{The compression-sparsification iteration implemented by each layer of \oursbase{}, and each encoder layer of \ours{}.} The compression step, implemented by the \(\MSSA\) operator, projects the tokens \(\vZ^{\ell}\) towards the subspace model \(\vU_{[K]}^{\ell}\) to form \(\vZ^{\ell + 1/2}\). The sparsification step, implemented by the \(\ISTA\) operator, rotates the tokens in \(\vZ^{\ell + 1/2}\) towards the coordinate axes, using the sparsifying dictionary \(\vD^{\ell}\), to get \(\vZ^{\ell + 1}\). The steps are performed in sequence and comprise a single of the \ours{} encoder.}
    \label{fig:crate_compress_sparsification}
\end{figure}

\subsection{Unifying Compression and Denoising}\label{sub:unification}

To transform our representations to the idealized signal model given by
\Cref{model:gaussian_tokens}, we seek to iteratively remove the disturbances or
deviations of each sample from this signal model. One way to perform this task
is to perform \textit{lossy data compression}
\citep{ma2007segmentation,yu2020learning,psenka2023representation,Yu2023-ig}:
compressed versions of the data, without ancillary disturbances, form the
representations. This approach has been favored in the construction of previous
white-box deep networks, such as \oursbase{} described above, due to the
existence of explicit information-theoretic criteria for compression. In this
case, the term \(R^{c}(\vZ \mid \vU_{[K]})\), defined in \Cref{eq:rc_def},
measures the lossy compression of the representations \(\vZ\) against the class
of statistical models given by \Cref{model:gaussian_tokens}. Thus, an operation
to minimize \(R^{c}\), such as \Cref{eq:grad_rc_mssa}, implements a step of
compression to learn better representations.

Another way to remove disturbances from the signal model
\Cref{model:gaussian_tokens}, especially if the perturbed model has the noisy
structure given in \Cref{model:gaussian_tokens_noise}, is to \textit{denoise}.
When the data is highly structured or low-dimensional, one-step denoising
becomes statistically and computationally difficult
\citep{pedregosa2023convergence}. Hence the modern solution to this problem is
via \textit{denoising diffusion models}, which take many small denoising steps
towards the data distribution at progressively decreasing noise levels
\citep{ho2020denoising,Song2020-xo,karras2022elucidating}.
Such models use estimates of the so-called \textit{score function} \(\nabla\log
p_{\sigma}\) \citep{Hyvarinen2005-fi}, where \(p_{\sigma}\) is the
probability density function of the noised input when the noise has standard
deviation \(\sigma > 0\). At all sufficiently small values of $\sigma$, the score function \(\nabla \log
p_{\sigma}(\widetilde{\vZ})\) for a particular noised input \(\widetilde{\vZ}\)
points towards the closest point to \(\widetilde{\vZ}\) on
the data distribution support
\citep{chen2023score,lu2023mathematical,Yu2023-ig}, or more generally the modes
of the true data distribution, which guides the denoising diffusion model to
project \(\vZ\) onto the support of the data distribution and diffuse it within
this support.\footnote{A more mathematical exposition of diffusion models may be
found in \Cref{app:diffusion}.}

\begin{figure}
    \centering 
    \includegraphics[width=0.7\textwidth]{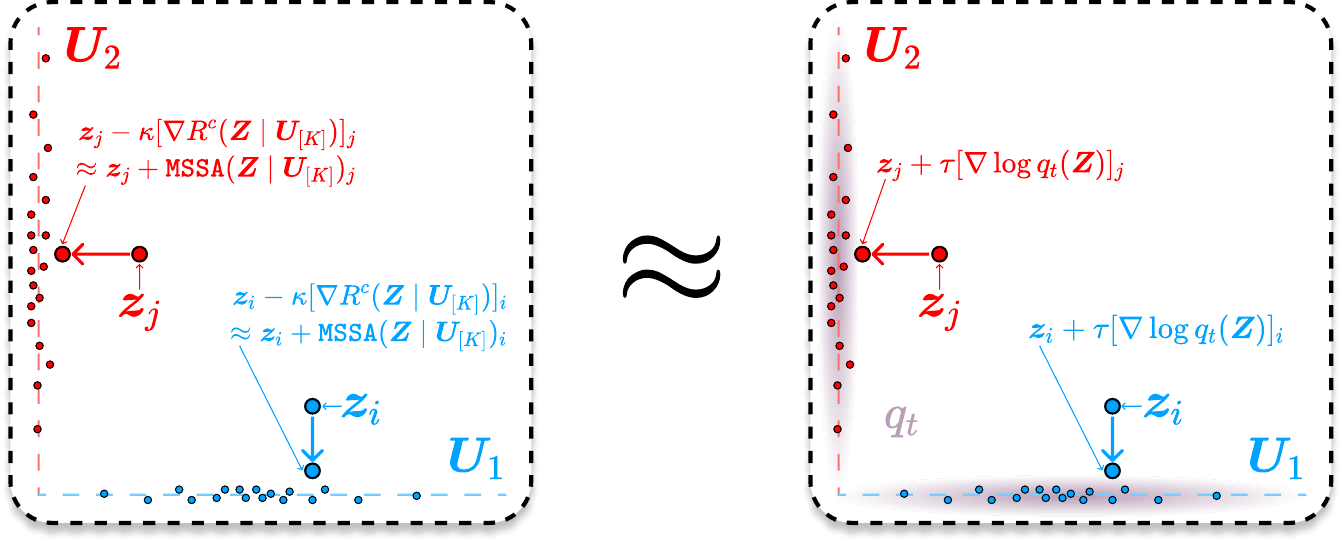}
    \caption{\small \textbf{Compression and denoising against the low-dimensional Gaussian mixture token model \Cref{model:gaussian_tokens} are equivalent.} \textit{Left:} the effect of compression against the low-dimensional Gaussian mixture model for tokens \Cref{model:gaussian_tokens}, i.e., taking gradient steps on the coding rate \(R^{c}(\cdot \mid \vU_{[K]})\) --- or equivalently, using the \(\MSSA(\cdot \mid \vU_{[K]})\) operator --- which is shown in \Cref{thm:informal_rate_score} to be equivalent to projecting onto the \(\vU_{[K]}\). \textit{Right:} the effect of denoising against \Cref{model:gaussian_tokens}, i.e., taking gradient steps on the score function of the noisy model \Cref{model:gaussian_tokens_noise} at small noise levels \(\sigma\), or equivalently small times \(t\). Up to scaling factors (not pictured), these two operations are equivalent, and have similar geometric and statistical interpretations as a projection onto the support of the data distribution. This connection motivates our structured denoising-diffusion framework, as elaborated in \Cref{sub:unification}.}
    \label{fig:structured_denoising}
\end{figure}

In the context of \Cref{model:gaussian_tokens,model:gaussian_tokens_noise}, both
denoising and compression operations conceptually remove additive disturbances
from the data, as visualized in \Cref{fig:structured_denoising}. In the
following result, we make this qualitative observation mathematically precise:
we show that under a simplified version of the signal model
\Cref{model:gaussian_tokens}, taking a gradient step on \(R^{c}\), a compression
primitive, acts as a projection onto the local signal model \(\vU_{[K]}\),
just as with the denoising primitive of taking a gradient step on \(\log
p_{\sigma}\).

\begin{theorem}[Informal version of \Cref{lem:inverse-term} in \Cref{app:computations_rr_gradient}]\label{thm:informal_rate_score}
    Suppose \(\vZ\) follows the noisy Gaussian codebook model \Cref{model:gaussian_tokens_noise}, with infinitesimal noise level \(\sigma^{\ell} > 0\) and subspace memberships \(s_{i}\) distributed as i.i.d.\ categorical random variables on the set of subspace indices \(\{1, \dots, K\}\), independently of all other sources of randomness. Suppose in addition that the number of tokens \(N\), the representation dimension \(d\), the number of subspaces \(K\), and the subspace dimensions \(p\) have relative sizes matching those of practical transformer architectures including the \ours{} encoder (specified in detail in \Cref{ass:parameter_config}). Then the negative compression gradient \(-\nabla_{\vz_{i}}R^{c}(\vZ^{\ell} \mid \vU_{[K]}^{\ell})\) points from \(\vz_{i}^{\ell}\) to the nearest \(\vU_{k}^{\ell}\).
\end{theorem}

\Cref{thm:informal_rate_score} establishes in a representative special case of the Gaussian codebook model \Cref{model:gaussian_tokens} that at low noise levels, \textit{compression against the signal model \Cref{model:gaussian_tokens} is equivalent to denoising against \Cref{model:gaussian_tokens}}. In the sequel, we use this connection to understand the \(\MSSA\) operators of the \ours{} encoder, derived in \Cref{sub:objective} from a different perspective, as realizing an incremental transformation of the data distribution towards the signal model \Cref{model:gaussian_tokens} via approximate denoising. This important property guarantees that a corresponding deterministic diffusion process---namely, the time reversal of the denoising process---implies an inverse operator for the compression operation implemented by \(\MSSA\). Because these approximate denoising processes transform the data towards a parametric structure, we call them \textit{structured denoising-diffusion processes}.

\subsection{Constructing a Distributionally-Invertible Transformer Layer}\label{sub:crate_mae_layer}

In \Cref{sub:setup}, we described a method to construct a white-box transformer-like encoder network via unrolled optimization meant to compress the data against learned geometric and statistical structures, say against a distribution of tokens where each token is marginally distributed as a Gaussian mixture supported on \(\vU_{[K]}\). In \Cref{sub:unification}, we described in general terms an approach that relates denoising and compression to yield a conceptually similar network using the formalism of diffusion models, this time trainable via autoencoding. In this section, we carry out this procedure concretely to obtain an encoder and decoder layer with similarly interpretable operational characteristics.

To measure compression, we use the \(R^{c}\) function defined in \Cref{eq:rc_def}. By using a standard (reverse-time) diffusion process with a scaling of \(R^{c}\) as a drop-in replacement for the score (see \Cref{app:discretization} for  details), we obtain that such a denoising diffusion process may be described by the following stochastic differential equation (SDE) \citep{Song2020-xo}.
\begin{equation}\label{eq:compression_sde}
    \odif{\vZ(t)} = -\frac{1}{T - t}\nabla R^{c}(\vZ(t) \mid \vU_{[K]})\odif{t} + \sqrt{2}\odif{\vB(t)}, \qquad \forall t \in [0, T],
\end{equation}
where \((\vB(t))_{t \in [0, T]}\) is a Brownian motion. \textit{As a design choice}, we wish to assert that our encoder and decoder ought to be deterministic, in particular preferring that our encoder-decoder architecture achieves \textit{sample-wise autoencoding} as opposed to \textit{distribution-wise autoencoding} or \textit{generation}. Thus we need to construct some ordinary differential equation (ODE) which transports the input probability distribution in the same way as \Cref{eq:compression_sde}. Such an equation is readily obtained as the \textit{probability flow ODE} \citep{Song2020-xo}, which itself is commonly used for denoising and sampling \citep{Song2020-xo,lu2022dpm,Song2023-rs} and has the form
\begin{equation}\label{eq:compression_ode}
    \odif{\vZ(t)} = -\frac{1}{2(T - t)}\nabla R^{c}(\vZ(t) \mid \vU_{[K]})\odif{t}, \qquad \forall t \in [0, T].
\end{equation}
In particular, the \(\vZ(t)\) generated by \Cref{eq:compression_sde,eq:compression_ode} have the same law. A first-order discretization (see \Cref{app:discretization}) of \Cref{eq:compression_ode} with step size \(\kappa\) obtains the iteration:
\begin{equation}\label{eq:compression}
    \vZ^{\ell + 1/2} =  \vZ^{\ell} + \MSSA(\vZ^{\ell} \mid \vU_{[K]}^{\ell}) \approx \vZ^{\ell} - \kappa \nabla R^{c}(\vZ^{\ell} \mid \vU_{[K]}^{\ell}),
\end{equation}
where \(\MSSA(\cdot)\) was defined in \Cref{eq:mssa}. Similar to \cite{Yu2023-ig}, in order to optimize the sparse rate reduction of the features, and in particular to sparsify them, we instantiate a learnable dictionary \(\vD^{\ell} \in \R^{d \times d}\) and sparsify against it, obtaining 
\begin{equation}\label{eq:encoder_sparsification}
    \vZ^{\ell + 1} = \ISTA(\vZ^{\ell + 1/2} \mid \vD^{\ell}),
\end{equation}
where \(\ISTA(\cdot)\) was defined in \Cref{eq:ista}. Thus, we obtain a two step iteration for the \(\ell\th\) encoder layer \(f^{\ell}\), where \(\vZ^{\ell + 1} = f^{\ell}(\vZ^{\ell})\):
\begin{equation}\label{eq:encoder_layer}
    \vZ^{\ell + 1/2} = \vZ^{\ell} + \MSSA(\vZ^{\ell} \mid \vU_{[K]}^{\ell}), \qquad \vZ^{\ell + 1} = \ISTA(\vZ^{\ell + 1/2} \mid \vD^{\ell}).
\end{equation}
This is the same layer as in \oursbase{}, whose conceptual behavior is illustrated in \Cref{fig:crate_compress_sparsification}. This equivalence stems from the fact that the diffusion probability flow \Cref{eq:compression_ode} is conceptually and mechanically similar to gradient flow on the compression objective in certain regimes, and so it demonstrates a useful conceptual connection between discretized diffusion and unrolled optimization as iteratively compressing or denoising the signal against the learned data structures.

Note that we parameterized a \textit{different} local signal model \(\vU_{[K]}^{\ell}\) and dictionary \(\vD^{\ell}\) at each layer, despite the continuous-time flows in \Cref{eq:compression_ode} using only one (i.e., the final) local signal model. This is because the sparsification step \Cref{eq:encoder_sparsification} transforms the data distribution, and so we require a different local signal model at each layer to model the new (more sparse) data distribution; see \Cref{fig:crate_mae_pipeline} for intuition on the iterative transformations. Also, having a different signal model at each layer may allow for more efficient iterative linearization and compression of highly nonlinear structures.

Now that we have shown how the structured diffusion approach can recover the original \oursbase{} architecture \citep{Yu2023-ig} as an encoder in our autoencoding problem, we use our new approach to construct a novel matching decoder. The time reversal of the ODE \Cref{eq:compression_ode} is:
\begin{equation}\label{eq:compression_ode_reverse}
    \odif{\vY(t)} = \frac{1}{2t}\nabla R^{c}(\vY(t) \mid \vU_{[K]})\odif{t}, \qquad \forall t \in [0, T],
\end{equation}
in the sense that the \(\vY(T - t)\) generated by \Cref{eq:compression_ode_reverse} has the same law as the \(\vZ(t)\) generated by \Cref{eq:compression_ode}, assuming compatible initial conditions. A first-order discretization of \Cref{eq:compression_ode_reverse} obtains the iteration:
\begin{equation}
    \vY^{\ell + 1} = \vY^{\ell + 1/2} - \MSSA(\vY^{\ell + 1/2} \mid \vV_{[K]}^{\ell}) \approx \vY^{\ell + 1/2} + \kappa \nabla R^{c}(\vY^{\ell + 1/2} \mid \vV_{[K]}^{\ell}),
\end{equation}
where \(\vV_{[K]}^{\ell} = (\vV_{1}^{\ell}, \dots, \vV_{K}^{\ell})\) and each \(\vV_{k}^{\ell} \in \R^{d \times p}\) are the bases of the subspaces to ``anti-compress'' against. In our work, we treat them as \textit{different} from the corresponding \(\vU_{k}^{L - \ell}\), because the discretization of \Cref{eq:compression_ode,eq:compression_ode_reverse} is imperfect, and thus we should not expect a 1-1 correspondence between local signal models in the encoder and decoder.
To invert the effect of a sparsifying \(\ISTA\) step, which our mental model in \Cref{fig:crate_compress_sparsification} portrays as a rotation of the subspace supports to a more incoherent configuration, we multiply by another learnable dictionary \(\vE^{\ell} \in \R^{d \times d}\), obtaining
\begin{equation}
    \vY^{\ell + 1/2} = \vE^{\ell}\vY^{\ell}, \qquad \vY^{\ell + 1} = \vY^{\ell + 1/2} - \MSSA(\vY^{\ell + 1/2} \mid \vV_{[K]}^{\ell}).
\end{equation}
This constructs the \((\ell + 1)\st\) layer \(g^{\ell}\) of our decoder. 
In the implementation, we add layer normalizations to ensure that the features are roughly constant-size so that the above approximations hold. \Cref{fig:crate_mae_layers} has a graphical depiction of the encoder and decoder layers.

\begin{figure}[t!]
    \centering
    \includegraphics[width=0.99\textwidth]{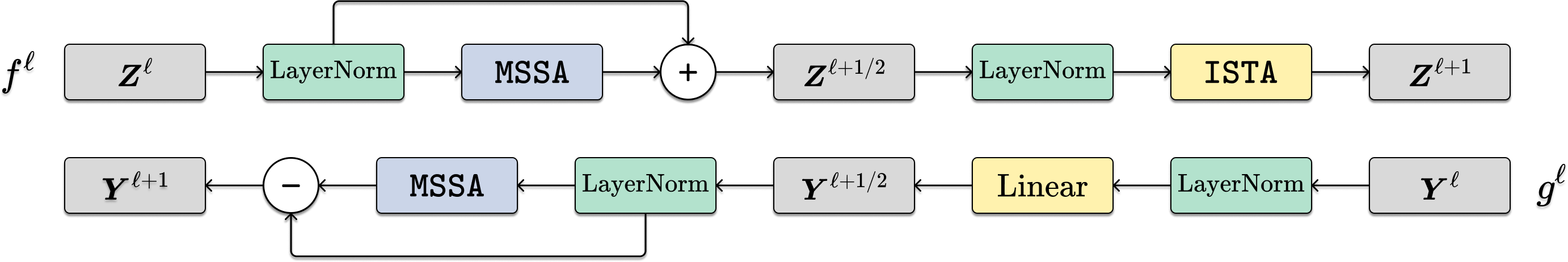}
    \caption{\small \textbf{Diagram of each encoder layer (\textit{top}) and decoder layer (\textit{bottom}) in \ours{}.} Notice that the two layers are highly anti-parallel --- each is constructed to do the operations of the other in reverse order. That is, in the decoder layer \(g^{\ell}\), the \(\ISTA\) block of \(f^{L - \ell}\) is partially inverted first using a linear layer, then the \(\MSSA\) block of \(f^{L - \ell}\) is reversed; this order unravels the transformation done in \(f^{L - \ell}\).}
    \label{fig:crate_mae_layers}
\end{figure}

\subsection{A Complete White-Box Transformer-Like Architecture for Autoencoding}\label{sub:crate_mae_arch}

As previously discussed, the encoder is the concatenation of a preprocessing map \(f^{\pre} \colon \R^{D \times N} \to \R^{d \times N}\), which has learnable parameters \(\vW^{\pre} \in \R^{d \times D}\) and \(\vE^{\pos} \in \R^{d \times N}\), and has the form:
\begin{equation}\label{eq:preprocessing}
    f^{\pre}(\vX) \doteq \vW^{\pre}\vX + \vE^{\pos},
\end{equation}
and \(L\) transformer-like layers \(f^{\ell} \colon \R^{d \times N} \to \R^{d \times N}\) given by
\begin{equation}
    f^{\ell}(\vZ^{\ell}) \doteq \ISTA(\vZ^{\ell} + \MSSA(\vZ^{\ell} \mid \vU_{[K]}^{\ell}) \mid \vD^{\ell}), \qquad \forall \ell \in [L],
\end{equation}
omitting normalization for simplicity. The decoder is the concatenation of \(L\) transformer-like layers \(g^{\ell} \colon \R^{d \times N} \to \R^{d \times N}\) given by
\begin{equation}
    g^{\ell}(\vY^{\ell}) \doteq \vE^{\ell}\vY^{\ell} - \MSSA(\vE^{\ell}\vY^{\ell} \mid \vV_{[K]}^{\ell}), \qquad \forall \ell \in [L] - 1,
\end{equation}
with a postprocessing map \(g^{\post} \colon \R^{d \times N} \to \R^{D \times N}\) which is a learnable linear map \(\vW^{\post} \in \R^{D \times d}\):
\begin{equation}\label{eq:postprocessing}
    g^{\post}(\vY^{L}) \doteq \vW^{\post}\vY^{L}.
\end{equation}
A full diagram of the autoencoding procedure is given in
\Cref{fig:crate_mae_pipeline}. %

Our training procedure seeks to learn and represent the structures in the data distribution. For this, we use a pretext task that measures the degree to which these structures have been learned: \textit{masked autoencoding} \citep{he2022masked}, which ``masks out'' a large percentage of randomly selected tokens in the input \(\vX\) %
and then attempts to reconstruct the whole image, measuring success by the resulting autoencoding performance. 
Conceptually, masked autoencoding can be seen as a nonlinear generalization of
the classical \textit{matrix completion} task, which exploits low-dimensional
structure to impute missing entries in incomplete data; classical matrix completion can be solved efficiently if and only if the data have low-dimensional structure \citep{Candes2009-kb,Amelunxen2014-yo,wright2022high}. 
The masked autoencoding loss writes
\begin{equation}\label{eq:mae_loss_unreg}
    L_{\mathrm{MAE}}(f, g) \doteq \mathbb{E}\big[\norm{(g \circ f)(\operatorname{\texttt{Mask}}(\vX)) - \vX}_{2}^{2}].
\end{equation}
Further implementation details of this architecture are discussed in \Cref{app:experiment_clarifications,app:experiment_pseudocode}.

\section{Empirical Evaluations}\label{sec:experiments}

In this section, we conduct experiments to evaluate \ours{} on real-world datasets and both supervised and unsupervised tasks. Similarly to \citet{Yu2023-ig}, \ours{} is built using simple design choices that we do not claim are optimal. We also do not claim that our results are optimally engineered; in particular, \textit{we do not use the extreme amount of computational resources required to obtain state-of-the-art performance using vision transformer-backed masked autoencoders (MAEs)~\citep{he2022masked}}.
Our goals in this section are to verify that our white-box masked autoencoding model \ours{} has promising performance and learns semantically meaningful representations, and that each operator in \ours{} aligns with our theoretical design.
We provide additional experimental details in
\Cref{app:experiment_clarifications,app:experiment_pseudocode}.

\begin{table}[t!]
    \centering
    \caption{\small \textbf{Model configurations for different sizes of \ours{}, parameter counts, and comparisons to ViT-MAE models from \citet{he2022masked,gandelsman2022test}.} We observe that \ours{}-Base uses around 30\% of the parameters of ViT-MAE-Base, and a similar number of parameters as ViT-MAE-Small.}
    \label{tab:model_configs}
    \small
    \setlength{\tabcolsep}{13.6pt}
    \resizebox{0.99\textwidth}{!}{%
        \begin{tabular}{@{}lcccccc@{}}
            \toprule
            \textbf{Model Configuration} & \(L\) & \(d\) & \(K\) & \(N\) & \ours{} \# Parameters & ViT-MAE \# Parameters
            \\ 
            \midrule
            \midrule
            Small (-S) & 12 & 576 & 12 & 196 & 25.4M & 47.6M \\
            \midrule
            Base (-B) & 12 & 768 & 12 & 196 & 44.6M & 143.8M \\
            \bottomrule
        \end{tabular}%
    }
\end{table}

\textbf{Network architecture and training configuration.} We implement the encoder and decoder
architectures described in \Cref{sec:approach}, with a few changes detailed in
\Cref{app:experiment_clarifications}. We consider different model sizes of
\ours{} by varying the token dimension \(d\), number of heads \(K\), and number
of layers \(L\); such parameters will be kept the same for the encoder and
decoder, which is contrary to \citet{he2022masked} but in line with our white-box derivation. 
\Cref{tab:model_configs} displays the \ours{} model configurations and number of
parameters, and compares with equivalent ViT-MAE model sizes
\citep{he2022masked,gandelsman2022test}, showing that \textit{\ours{} uses
around 30\% of the parameters of MAE with the same model configuration}. %
We consider ImageNet-1K \citep{deng2009imagenet} as the main experimental setting for our architecture. We apply the AdamW \citep{loshchilov2017decoupled} optimizer to train \ours{} models for both pre-training and fine-tuning. When fine-tuning, we also use several commonly used downstream datasets: CIFAR10, CIFAR100 \citep{krizhevsky2009learning}, Oxford Flowers \citep{nilsback2008automated}, and Oxford-IIT-Pets \citep{parkhi2012cats}. 

\begin{figure}[t!]
    \centering
    \begin{subfigure}[b]{0.48\textwidth}
        \centering
        \includegraphics[width=\textwidth]{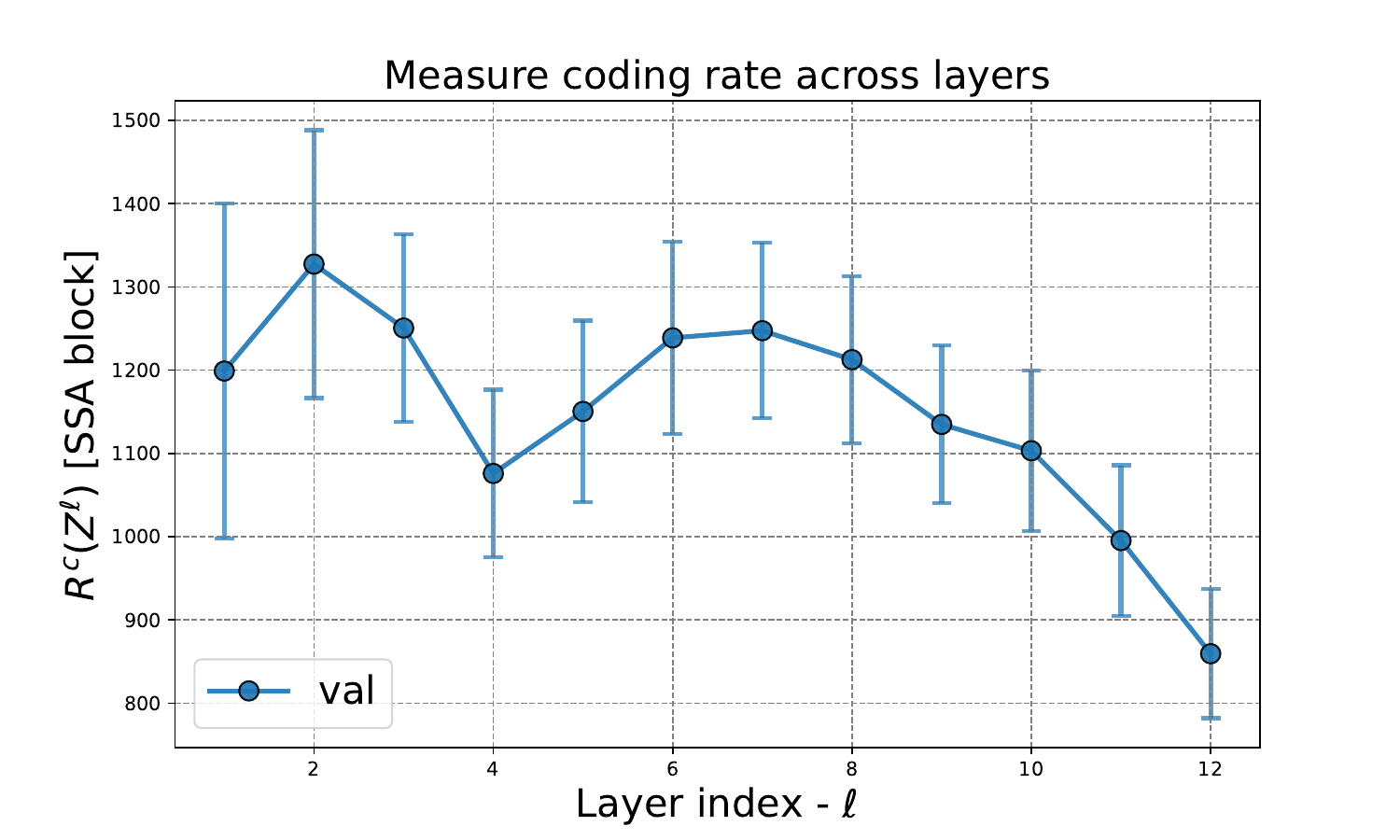}
    \end{subfigure}
    \begin{subfigure}[b]{0.48\textwidth}
        \centering
        \includegraphics[width=\textwidth]{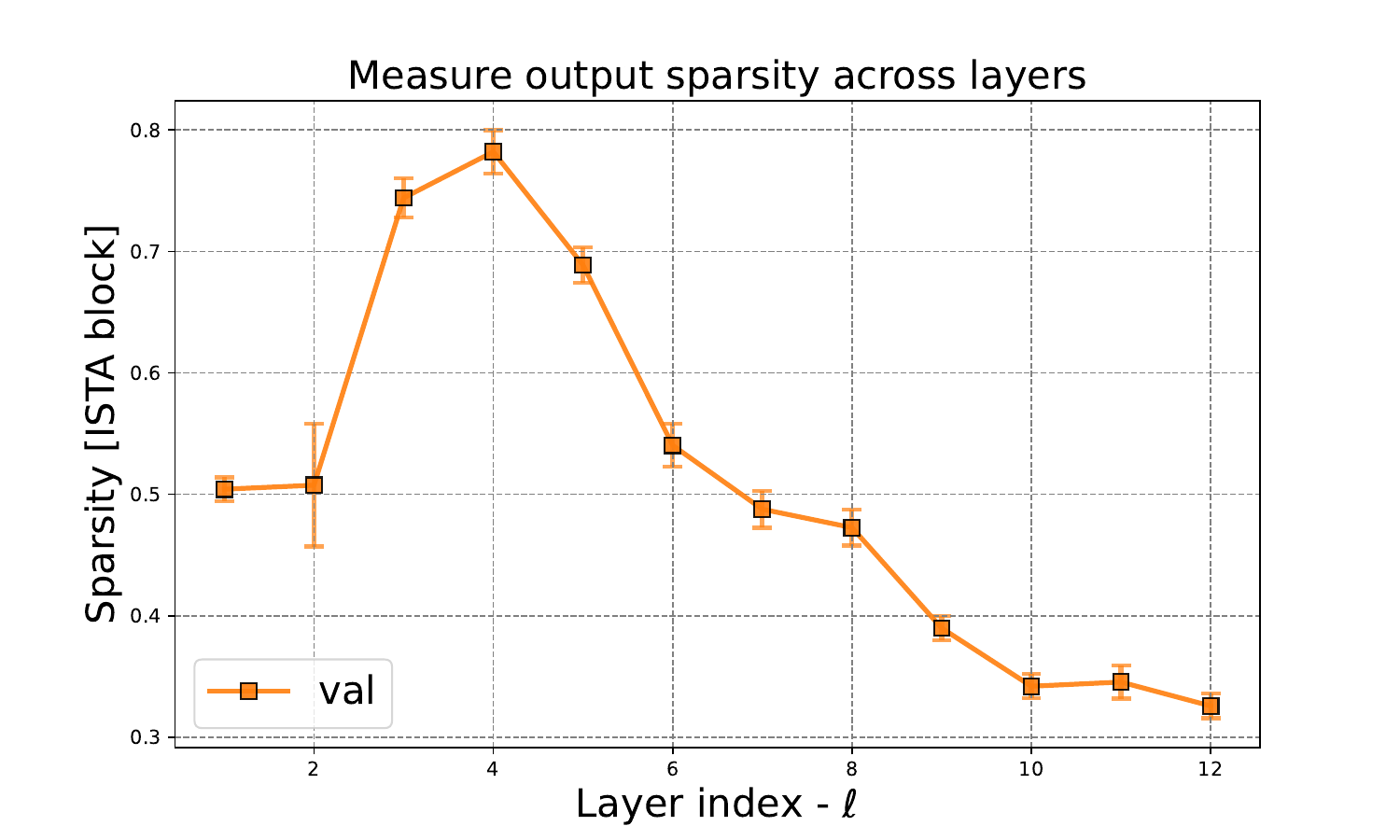}
    \end{subfigure}
    \caption{\small \textbf{\textit{Left}: The compression measure \(R^{c}(\vZ^{\ell+1/2} \mid \vU_{[K]}^{\ell})\) at different layers of the encoder. \textit{Right}: the sparsity measure \(\|\vZ^{\ell+1}\|_0 / (d\cdot N)\), at different layers of the encoder.} Measurements were collected from \ours{}-Base averaged over \(10000\) randomly chosen ImageNet samples. We observe that the compression and sparsity improve consistently over each layer and through the whole network.}
    \label{fig:exp-rc-sparsity-base}
\end{figure}

\textbf{Layer-wise function analysis.} First, we confirm that our model actually does do layer-wise compression and sparsification, confirming our conceptual understanding as described in \Cref{sec:approach}. In \Cref{fig:exp-rc-sparsity-base}, we observe that each layer of the encoder tends to compress and sparsify the input features, confirming our theoretical designing of the role of each operator in the network.

\begin{figure}[t!]
    \centering
    \includegraphics[width=0.99\textwidth]{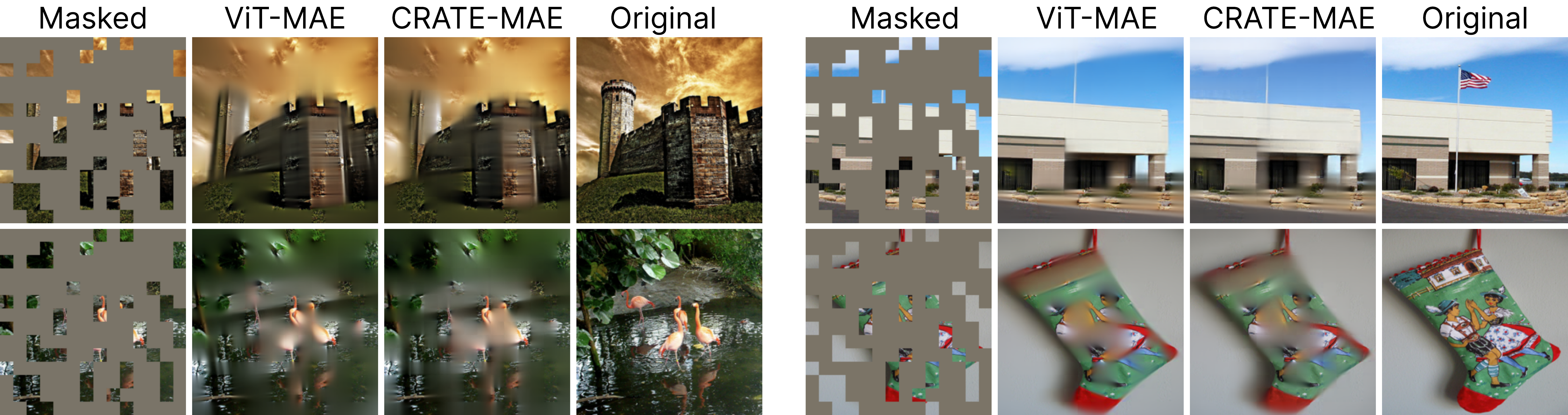}
    \caption{\small \textbf{Autoencoding visualizations of \ours{}-Base and ViT-MAE-Base \citep{he2022masked} with 75\% patches masked.} We observe that the reconstructions from \ours{}-Base are on par with the reconstructions from ViT-MAE-Base, despite using \(< 1/3\) of the parameters.
    }
    \label{fig:mae_autoencoding}
\end{figure}
\textbf{Autoencoding performance.} In \Cref{fig:mae_autoencoding}, we qualitatively compare the masked autoencoding performance of \ours{}-Base to ViT-MAE-Base \citep{he2022masked}. We observe that both models are able to reconstruct the data well, despite \ours{} using less than a third of the parameters of ViT-MAE. In \Cref{tab:reconstruction_loss} (deferred to \Cref{app:additional_experiments}) we display the average reconstruction loss of \ours{}-Base and ViT-MAE-Base, showing a similar quantitative conclusion.

\begin{table*}[t!]
    \centering 
    \caption{\small \textbf{Top-\(1\) classification accuracy of \ours{} models when pre-trained on ImageNet-1K and evaluated via fine-tuning or linear probing for various datasets.} We compare \ours{} to standard ViT-MAE models with many more parameters. Our results show that \ours{} achieves competitive performance on this transfer learning task when either fine-tuning the whole model or just the classification head.}
    \label{tab:linear_probing_finetuning}
    \small
    \setlength{\tabcolsep}{13.6pt}
    \resizebox{0.99\textwidth}{!}{%
        \begin{tabular}{@{}lcc|cc@{}}
            \toprule 
            \textbf{Classification Accuracy} & \ours{-S} & \ours{-B} & {\color{gray} ViT-MAE-S} & {\color{gray} ViT-MAE-B} \\
            \midrule
            \midrule
            \underline{\textit{Fine-Tuning}} & & & \\
            CIFAR10 & 96.2 & 96.8 & {\color{gray} 97.6} & {\color{gray} 98.5} \\
            CIFAR100 & 79.0 & 80.3 & {\color{gray} 83.0} & {\color{gray} 87.0} \\
            Oxford Flowers-102 & 71.7 & 78.5 & {\color{gray} 84.2} & {\color{gray} 92.5} \\
            Oxford-IIIT-Pets & 73.7 & 76.7 & {\color{gray} 81.7} & {\color{gray} 90.3} \\
            \midrule 
            \underline{\textit{Linear Probing}} & & & \\
            CIFAR10 & 79.4 & 80.9 & {\color{gray} 79.9} & {\color{gray} 87.9} \\
            CIFAR100 & 56.6 & 60.1 & {\color{gray} 62.3} & {\color{gray} 68.0} \\
            Oxford Flowers-102 & 57.7 & 61.8 & {\color{gray} 66.8} & {\color{gray} 66.4} \\
            Oxford-IIIT-Pets & 40.6 & 46.2 & {\color{gray} 51.8} & {\color{gray} 80.1} \\
            \bottomrule
        \end{tabular}
    }
\end{table*}

\begin{figure}[t!]
    \centering
    \begin{subfigure}[b]{0.475\textwidth}
        \centering
        \includegraphics[width=\textwidth]{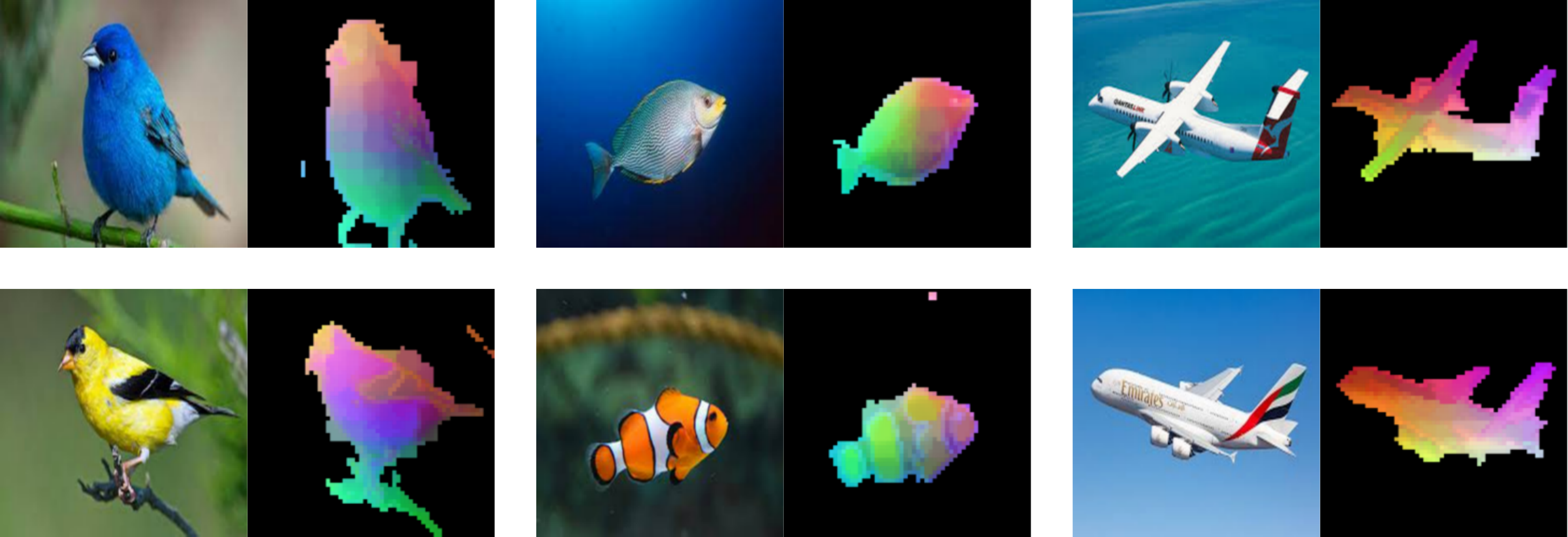}
        \caption{Visualizing PCA of token representations.}
        \label{fig:pca}
    \end{subfigure}
    \hspace{0.025\textwidth}
    \begin{subfigure}[b]{0.475\textwidth}
        \centering
        \includegraphics[width=\textwidth]{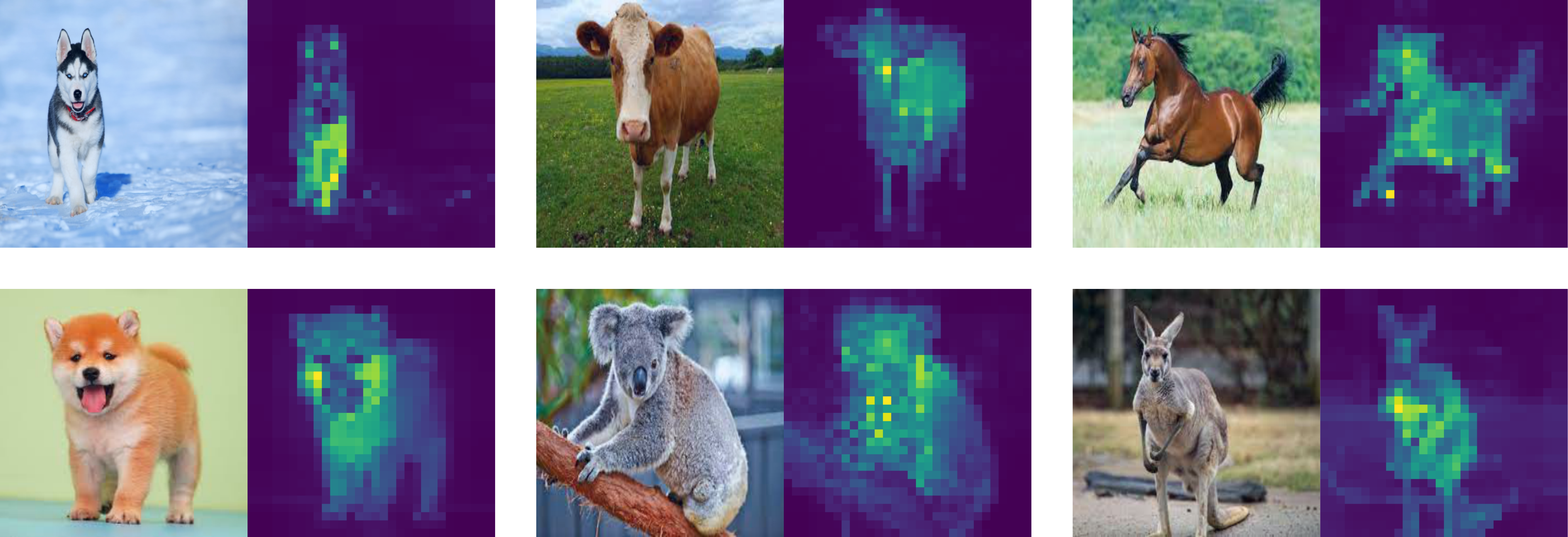}
        \caption{Visualizing selected attention head outputs.}
        \label{fig:attention_maps}
    \end{subfigure}
    \caption{\small \textbf{\textit{Left:} Visualizations of the alignment of each image's token representations with the top three principal components (in red, blue, and green respectively) of all token representations of images in the given class. \textit{Right:} Visualizations of hand-picked attention map across all attention heads in the last layer of the \ours{} encoder for each image.} We observe in \Cref{fig:pca} that the top three principal components are aligned with tokens from parts of the image that carry its semantics, and in \Cref{fig:attention_maps} that the attention maps correctly ``attend to'' (activate strongly on) exactly the parts of the image which are semantically meaningful.}
    \label{fig:pca_attn}
\end{figure}

\textbf{Representation learning and emerging semantic properties.}
In
\Cref{tab:linear_probing_finetuning}, we display the performance of \ours{}
models when fine-tuned or linear probed for supervised classification (precise
method in \Cref{app:experiment_clarifications}) on a variety of datasets.
We observe that the classification accuracies of \ours{} models are competitive
with much larger ViT-MAE models.
Moreover, the learned representations of \ours{} carry useful semantic content.
By taking the alignment of the
representations of each token with the top few principal components of the
representations of tokens in each class (precise details in
\Cref{app:viz_methodology}), we observe in \Cref{fig:pca_attn} (left) that the
representations are linearized, and that the top few principal components carry
semantic structure.
In \Cref{fig:pca_attn}
(right), we observe that the attention heads in the \(\MSSA\) operator in \ours{}
capture the semantics of the input images.
These properties have previously been observed in white-box models trained with
supervised cross-entropy losses \citep{yu2023emergence}; our results demonstrate
that they are consequences of the white-box architecture, rather than the
loss function.

\section{Conclusion}\label{sec:conclusion}

In this work, we uncover a quantitative connection between denoising and compression, and use it to design a conceptual framework for building white-box (mathematically interpretable) transformer-like deep neural networks which can learn using unsupervised pretext tasks, such as masked autoencoding. We show that such models are more parameter-efficient over their empirically designed cousins, achieve promising performance on large-scale real-world imagery datasets, and learn structured representations that contain semantic meaning. This work demonstrates the potential and practicality of white-box networks derived from first principles for tasks outside supervised classification. We thus believe that this work helps to bridge the theory and practice of deep learning, by unifying on both the conceptual and technical level many previously separated approaches including, but not limited to, diffusion and denoising, compression and rate reduction, transformers, and (masked) autoencoding.

\subsubsection*{Acknowledgments}
The authors would like to acknowledge help from Tianzhe Chu in preparing the manuscript. Druv Pai acknowledges support from a UC Berkeley College of Engineering fellowship. Yaodong Yu acknowledges support from the joint Simons Foundation-NSF DMS grant \#2031899 and AI Community Mini Grant from Future of Life Institute. Yi Ma acknowledges support from the joint Simons Foundation-NSF DMS grant \#2031899, the ONR grant N00014-22-1-2102, and the University of Hong Kong.

\bibliography{references.fixed.bib}
\bibliographystyle{iclr2024_conference}

\clearpage
\appendix
\section{Other Related Work}\label{app:related_work}

In \Cref{sec:intro}, we described the approaches to unsupervised learning of low-dimensional structures in the data that were most relevant to the rest of the work. Here, we discuss some other popular alternatives for completeness.

\paragraph{Black-box unsupervised representation learning.} On the other end from white-box models, which learn representations from data that have \textit{a priori} desired structures, are black-box unsupervised learning methods which learn fully data-driven representations. One implementation of this principle includes contrastive learning, which learns representations from computing the statistics of multiple augmentations of the same data point \citep{chen2020simple,bardes2021vicreg}. Another angle is to seek a representation with desirable characteristics for a specific task, such as classification; prior works have considered diffusion models as ``representation learners'' from this angle \citep{Chen2024-tb,Xiang2023-wq}. The notion of representation learning we are interested in in this work, namely the transformation of the data distribution towards a structured form that preserves its essential information content, is different from the notion in this latter group of works. Still another implementation is that of \textit{autoencoding} models, the most recently popular of which is the \textit{masked autoencoder} (MAE) \citep{he2022masked}. Autoencoders attempt to build low-dimensional representations of the data and use them to reconstruct input data \citep{Tishby-ITW2015,Rezende2014-vr,Kingma2013-sb,Hinton-NIPS1993}; masked autoencoders specifically mask the input data in training and attempt to impute the missing entries through reconstruction.

The common point in all such unsupervised learning methods so far is that they use \textit{black-box} neural networks, such as ResNets \citep{chen2020simple} or transformers \citep{caron2021emerging}, as their back-end. Thus, although they sometimes develop semantically meaningful representations of the data \citep{chen2020simple,caron2021emerging,bardes2021vicreg}, they are uninterpretable, and their training procedures and internal mechanisms are opaque.

\paragraph{Deep networks and stochastic dynamics.} There are many quantitative rapprochements of deep learning and stochastic dynamics. The most well-known of these is diffusion models, which can be modeled as discretizations of It\"{o} diffusion processes \citep{Song2020-xo}. The neural network is usually trained to estimate the so-called \textit{score function}. Diffusion models can be thought of as implementing a particular approximation to optimal transport between a template distribution and the true data distribution \citep{khrulkov2022understanding}. Different types of stochastic dynamics useful for generative modeling may be derived from optimal transport between the data distribution and a pre-specified template \citep{de2021diffusion,albergo2023stochastic}. However, diffusion models are unique among these methods in that they have an \textit{iterative denoising} interpretation \citep{karras2022elucidating}, which this work draws on. Such an interpretation has previously been used to construct deep denoising networks from unrolled diffusion processes \citep{mei2023deep}, instead of just using the deep networks to do black-box estimation of the score function. Similar studies have interpreted deep networks as discretizations of diffusion processes without this particular denoising interpretation \citep{li2022neural}, but the aforementioned unrolled iterative denoising strategy is what we draw upon in this work.

\paragraph{Other related work.} Here we also discuss some related work with regards to different modifications of the transformer architecture and training procedures which interface well with our white-box design. For example, \citet{kitaev2020reformer} suggests that sharing the \(Q\) and \(K\) matrices in the regular transformer is a mechanism to make the transformer more efficient at no performance cost. This choice is a heuristic, whereas our white-box design suggests that \(Q\), \(K\), and \(V\) should be set equal, and as we see in the paper this comes with some small tradeoffs. Also, since white-box models are derived such that each layer has a defined and understood role, it is natural to ask if such models can be trained layer-wise, i.e., one layer at a time \citep{bengio2006greedy}. While this is also possible, we leave it to future work; our experiments show that with end-to-end training, a vastly more common method to train deep networks, each layer still follows the role it was designed for.

\subsection{An Overview of Diffusion Processes}\label{app:diffusion}

In this section, we give an overview of the basics of time-reversible It\^{o} diffusion processes, the mathematical foundation for diffusion models. This is to make this paper more self-contained by providing knowledge about general diffusion processes that we will apply to our special models. The coverage adapts that of \cite{millet1989integration,Song2020-xo,karras2022elucidating}. 

Consider a generic It\^{o} diffusion process \((\vz(t))_{t \in [0, T]}\), where \(\vz(t)\) is an \(\R^{m}\)-valued random variable, given by the SDE
\begin{equation}\label{eq:forward_general_diffusion}
    \odif{\vz(t)} = b(\vz(t), t)\odif{t} + \Sigma(\vz(t), t)\odif{\vw(t)}, \qquad \vz(0) \sim P, \qquad \forall t \in [0, T]
\end{equation}
where \(\vw\) is a Brownian motion and \(P\) is some probability measure on \(\R^{m}\) (in this case representing the data distribution). Here the \textit{drift coefficient} \(b \colon \R^{m} \times \R \to \R^{m}\) and \textit{diffusion coefficient} \(\Sigma \colon \R^{m} \times \R \to \R^{m \times m}\) are functions. To make sense of \Cref{eq:forward_general_diffusion} and also verify the existence of strong (i.e., pathwise well-defined) solutions, we need some regularity on them, and we choose the following assumption:
\begin{enumerate}[label={A\arabic*.}, leftmargin=*]
    \item \(b\) and \(\Sigma\) have some spatial smoothness and do not grow too fast, i.e., there is a constant \(K \geq 0\) such that for all \(\vx, \wt{\vz} \in \R^{m}\) we have
    \begin{align}
        &\sup_{t \in [0, T]}\bs{\norm{\Sigma(\vx, t) - \Sigma(\wt{\vz}, t)}_{F} + \norm{b(\vx, t) - b(\wt{\vz}, t)}_{2}} \leq K\norm{\vx - \wt{\vz}}_{2} \\
        &\sup_{t \in [0, T]}\bs{\norm{\Sigma(\vx, t)}_{F} + \norm{b(\vx, t)}_{2}} \leq K(1 + \norm{\vx}_{2}).
    \end{align}
\end{enumerate}
In general, \(\vz(t)\) may not have a density w.r.t.~the Lebesgue measure on \(\R^{m}\). For example, suppose that \(P\) is supported on some low-dimensional linear subspace (or even a Dirac delta measure), and take \(\Sigma\) to be the orthoprojector onto this subspace. Then \(\vz(t)\) will be supported on this subspace for all \(t\) and thus not have a density w.r.t.~the Lebesgue measure. Thus, when further discussing processes of the type \Cref{eq:forward_general_diffusion}, we make the following assumption
\begin{enumerate}[resume, label={A\arabic*.}, leftmargin=*]
    \item \(\vz(t)\) has a probability density function \(p(\cdot, t)\) for all \(t > 0\). 
    
    This is guaranteed by either of the following conditions \citep{millet1989integration}:
    \begin{enumerate}[label={A2.\arabic*}]
        \item \(b\) and \(\Sigma\) are differentiable in \((\vx, t)\) and have H\"{o}lder-continuous derivatives, and \(P\) has a density w.r.t.~the Lebesgue measure;
        \item The event 
        \begin{equation}
            \{\text{\(\rank(\Sigma(\vz(s), s)) = m\) for all \(s\) in some neighborhood of \(0\)}\}
        \end{equation}
        happens \(P\)-almost surely.
    \end{enumerate}
\end{enumerate}

Define \(\Psi \colon \R^{m} \times \R \to \R^{m \times m}\) by
\begin{equation}
    \Psi(\vx, t) \doteq \Sigma(\vx, t)\Sigma(\vx, t)\adj.
\end{equation}
To discuss time-reversibility, we also need the following local integrability condition, which is another measure of sharp growth of the coefficients (or precisely their derivatives):
\begin{enumerate}[resume, label={A\arabic*.}, leftmargin=*]
    \item The functions \((\vx, t) \mapsto \nabla_{\vx} \cdot (\Psi(\vx, t)p(\vx, t))\) are integrable on sets of the form \(D \times [t_{0}, 1]\) for \(t_{0} > 0\) and \(D\) a bounded measurable subset of \(\R^{m}\):
    \begin{align}
        \int_{t_{0}}^{1}\int_{D}\norm{\nabla_{\vx} \cdot (\Psi(\vx, t)p(\vx, t))}_{2}\odif{\vx}\odif{t} < \infty.
    \end{align}
\end{enumerate}
To write the notation out more explicitly, 
\begin{align}
    &\nabla_{\vx} \cdot (\Psi(\vx, t)p(\vx, t)) = \mat{\nabla_{\vx} \cdot (\Psi^{1}(\vx, t)p(\vx, t)) \\ \vdots \\ \nabla_{\vx} \cdot (\Psi^{m}(\vx, t)p(\vx, t))}  \\
    \text{where} \qquad 
    &\nabla_{\vx} \cdot (\Psi^{i}(\vx, t)p(\vx, t)) = \sum_{j = 1}^{m}\pdv*{[\Psi^{ij}(\vx, t)p(\vx, t)]}{x_{j}}
\end{align}
where \(\Psi^{i}\) is the \(i\th\) row of \(\Psi\) transposed to a column, and \(\Psi^{ij}\) is the \((i, j)\th\) entry of \(\Psi\). Note that \cite{millet1989integration} phrases this in terms of an local integrability condition on each \(\abs{\nabla_{\vx} \cdot (\Psi^{i}(\vx, t)p(\vx, t))}\), which would naturally give a local integrability condition on \(\norm{\nabla_{\vx}\cdot (\Psi(\vx, t)p(\vx, t))}_{\infty}\). However, all norms on \(\R^{m}\) are equivalent, and so this leads to a local integrability condition for \(\norm{\nabla_{x} \cdot (\Psi(\vx, t)p(\vx, t))}_{2}\) as produced. Note that the assumptions do not guarantee that the involved derivatives exist, in which case they are taken in the distributional (e.g., weak) sense, whence they should exist \citep{millet1989integration}.

Under assumptions A1---A3, \cite{millet1989integration} guarantees the existence of another process \((\wt{\vz}(t))_{t \in [0, T]}\) such that the laws of \(\vz(t)\) and \(\wt{\vz}(T - t)\) are the same for all \(t \in [0, T]\). This process \((\wt{\vz}(t))_{t \in [0, T]}\) is called the \textit{time reversal} of \((\vz(t))_{t \in [0, T]}\), and is shown to have law given by 
\begin{equation}\label{eq:backward_general_diffusion}
    \odif{\wt{\vz}(t)} = b^{\leftarrow}(\wt{\vz}(t), t)\odif{t} + \Sigma^{\leftarrow}(\wt{\vz}(t), t)\odif{\vw^{\leftarrow}(t)}, \qquad \wt{\vz}(0) \sim p(\cdot, T), \qquad \forall t \in [0, T]
\end{equation}
where \(\vw^{\leftarrow}(t)\) is an independent Brownian motion and
\begin{align}
    b^{\leftarrow}(\vx, t)
    &= -b(\vx, T - t) + \frac{\nabla_{\vx}\cdot[\Psi(\vx, T - t)p(\vx, T - t)]}{p(\vx, T - t)} \\
    &= -b(\vx, T - t) + \nabla_{\vx} \cdot \Psi(\vx, T - t) + \Psi(\vx, T - t)[\nabla_{\vx}\log p(\vx, T - t)], \\
    \Sigma^{\leftarrow}(\vx, t)
    &= \Sigma(\vx, T - t).
\end{align}

We would next like to develop an ODE which transports the probability mass \(P\) in the same way as \Cref{eq:forward_general_diffusion} --- namely, find another process \((\ol{\vz}(t))_{t \in [0, T]}\) which has deterministic dynamics, yet has the same law as \((\vz(t))_{t \in [0, T]}\). \cite{Song2020-xo} looks at the Fokker-Planck equations (which can be defined, at least in a weak sense, under assumptions A1--A2) and manipulates them to get the following dynamics for \(\ol{\vz}(t)\):
\begin{align}\label{eq:general_diffusion_ode}
    \odif{\ol{\vz}(t)} 
    &= \ol{b}(\ol{\vz}(t), t)\odif{t}, \qquad \ol{\vz}(0) \sim P, \qquad \forall t \in [0, T], \\
    \text{where} \qquad \ol{b}(\vx, t)
    &= b(\vx, t) - \frac{1}{2}\cdot \frac{\nabla_{\vx} \cdot [\Psi(\vx, t)p(\vx, t)]}{p(\vx, t)} \\
    &= b(\vx, t) - \frac{1}{2}\nabla_{\vx}\cdot \Psi(\vx, t) - \frac{1}{2}\Psi(\vx, t)[\nabla_{x}\log p(\vx, t)].
\end{align}
Now to get a similar process for \(\wt{\vz}(t)\), namely a process \((\ol{\wt{\vz}}(t))_{t \in [0, T]}\) which evolves deterministically yet has the same law as \((\wt{\vz}(t))_{t \in [0, T]}\), we may either take the time reversal of \Cref{eq:general_diffusion_ode} or apply the Fokker-Planck method to \Cref{eq:backward_general_diffusion}, in both cases obtaining the same dynamics:
\begin{align}\label{eq:backward_general_diffusion_ode}
    \small
    \odif{\ol{\wt{\vz}}(t)}
    &= \ol{b}^{\leftarrow}(\ol{\wt{\vz}}(t), t)\odif{t}, \qquad \ol{\wt{\vz}}(0) \sim p(\cdot, T), \qquad \forall t \in [0, T], 
\end{align}
where
\begin{align}
    \ol{b}^{\leftarrow}(\vx, t)
    &= -\ol{b}(\vx, T - t) \\
    &= -b(\vx, T - t) + \frac{1}{2}\cdot \frac{\nabla_{\vx} \cdot [\Psi(\vx, T - t)p(\vx, T - t)]}{p(\vx, T - t)} \\
    &= -b(\vx, t) + \frac{1}{2}\nabla_{\vx}\cdot \Psi(\vx, T - t) + \frac{1}{2}\Psi(\vx, T - t)[\nabla_{\vx}\log p(\vx, T - t)].
\end{align}
The quantity \(\nabla_{\vx}\log p(\vx, t)\) is of central importance; it is denoted the \textit{score at time \(t\)}, and we use the notation \(s(\vx, t) \doteq \nabla_{x}\log p(\vx, t)\) for it. With this substitution, we have the following dynamics for our four processes:
\begin{align}
    \odif{\vz(t)} 
    &= b(\vz(t), t)\odif{t} + \Sigma(\vz(t), t)\odif{\vw(t)}, \quad \vz(0) \sim P \\
    \odif{\wt{\vz}(t)} 
    &= [-b(\wt{\vz}(t), T - t) + \nabla_{\vx} \cdot \Psi(\wt{\vz}(t), T - t) + \Psi(\wt{\vz}(t), T - t)s(\wt{\vz}(t), T - t)]\odif{t} \\
    &\qquad + \Sigma(\wt{\vz}(t), T - t)\odif{\vw^{\leftarrow}(t)}, \quad \wt{\vz}(0) \sim p(\cdot, T) \\
    \odif{\ol{\vz}(t)}
    &= \bs{b(\ol{\vz}(t), t) - \frac{1}{2}\nabla_{\vx}\cdot \Psi(\ol{\vz}(t), t) - \frac{1}{2}\Psi(\ol{\vz}(t), t)s(\ol{\vz}(t), t)}\odif{t}, \quad \ol{\vz}(0) \sim P \\
    \odif{\ol{\wt{\vz}}(t)}
    &= \bigg[-b(\ol{\wt{\vz}}(t), T - t) + \frac{1}{2}\nabla_{\vx}\cdot \Psi(\ol{\wt{\vz}}(t), T - t) \\
    &\qquad + \frac{1}{2}\Psi(\ol{\wt{\vz}}(t), T - t)s(\ol{\wt{\vz}}(t), T - t)\bigg]\odif{t}, \quad \ol{\wt{\vz}}(0) \sim p(\cdot, T).
\end{align}

In practice, one fits an estimator for \(s(\cdot, \cdot)\) and estimates \(p(\cdot, T)\) and runs a discretization of either \Cref{eq:backward_general_diffusion} or \Cref{eq:backward_general_diffusion_ode} to sample approximately from \(P\). One common instantiation used in diffusion models \citep{karras2022elucidating} is the so-called \textit{variance-exploding} diffusion process, which has the coefficient settings
\begin{equation}
    b(\vx, t) = 0, \qquad \Sigma(\vx, t) = \sqrt{2}\vI
\end{equation}
which implies that
\begin{equation}
    \Psi(\vx, t) = 2\vI.
\end{equation}
This means that the four specified processes are of the form
\begin{align}
    \odif{\vz(t)} 
    &= \sqrt{2}\odif{\vw(t)}, \quad \vz(0) \sim P \\
    \odif{\wt{\vz}(t)} 
    &= s(\wt{\vz}(t), T - t)\odif{t} + \sqrt{2}\odif{\vw^{\leftarrow}(t)}, \quad \wt{\vz}(0) \sim p(\cdot, T) \\
    \odif{\ol{\vz}(t)}
    &= s(\ol{\vz}(t), t)\odif{t}, \quad \ol{\vz}(0) \sim P \\
    \odif{\ol{\wt{\vz}}(t)}
    &= -s(\ol{\wt{\vz}}(t), T - t), \quad \ol{\wt{\vz}}(0) \sim p(\cdot, T).
\end{align}
Notice that the determinstic flows are actually gradient flows on the score, which concretely reveals a connection between sampling and optimization, and thus between diffusion models (precisely those which use the probability flow ODE to sample) and unrolled optimization networks.

In this context, we can also establish the connection between diffusion networks and iterative denoising. In the variance-exploding setting, we have
\begin{equation}
    \vz(t) \sim \mathcal{N}(\vz(0), 2t\vI),
\end{equation}
which can be easily computed using results from, e.g., \cite{sarkka2019applied}. Thus \(\vz(t)\) is a noisy version of \(\vz(0)\), with noise level increasing monotonically with \(t\), and sampling \(\vz(0)\) from \(\vz(t)\) conceptually removes this noise. Concretely, \textit{Tweedie's formula} \citep{efron2011tweedie} says that the optimal denoising function \(\E{\vz(0) \mid \vz(t)}\) has a simple form in terms of the score function:
\begin{equation}
    \E{\vz(0) \mid \vz(t)} = \vz(t) + 2t\cdot s(\vz(t), t).
\end{equation}
In other words, the score function \(s\) points from the current iterate \(\vz(t)\) to the value of the optimal denoising function, so it is a negative multiple of the conditionally-expected noise. Following the score by (stochastic) gradient flow or its discretization is thus equivalent to iterative denoising.

\subsection{Companion to \Cref{sub:unification}} \label{app:computations_rr_gradient}

In this section, we prove a formal version of the result \Cref{thm:informal_rate_score} stated in \Cref{sub:unification}. That is, we examine a basic yet representative instantiation of the signal model \Cref{model:gaussian_tokens_noise}, and show that under this model, in a
natural regime of parameter scales motivated by the architecture of \ours{}
applied to standard image classification benchmarks, the operation implemented
by taking a gradient step on the compression term of the sparse rate reduction
objective \Cref{eq:sparse_rr} corresponds to a projection operation at
quantization scales $\veps^2$ proportional to the size of the deviation.
This leads us in particular to a formal version of the result
\Cref{thm:informal_rate_score}.

\paragraph{Signal model.}

We consider an instantiation of the model \Cref{model:gaussian_tokens_noise}, elaborated here. That is, we fix a distribution over tokens $\vZ \in \bbR^{d \times N}$ induced by
the following natural signal model: each token $\vz_i$ is drawn independently from the
normalized isotropic Gaussian measure on one of $K$ $p$-dimensional
subspaces with orthonormal bases $\vU_1, \dots, \vU_K \in
\bbR^{d \times p}$,\footnote{More precisely, $\vz_i$ is distributed according to the
pushforward of the normalized isotropic Gaussian measure $\sN(\vZero, \tfrac{1}{p}\vI)$ on $\bbR^p$ by the bases $\vU_k$.}
which comprise the low-dimensional structure in the observed tokens, then
corrupted with i.i.d.\ Gaussian noise $\sN(\vZero, \tfrac{\sigma^2}{d} \vI)$;
the subspace each token is drawn from is selected uniformly at random,
independently of all other randomness in the problem.
This signal model therefore corresponds to the setting of uncorrelated tokens,
with maximum entropy coordinate distributions within subspaces.
It is natural to first develop our
theoretical understanding of the connection between compression and the score
function in the uncorrelated setting, although in general, the ability of \ours{} to capture correlations in the
data through the MSSA block is essential. In connection with the latter issue, we note that our proofs will generalize straightforwardly to the setting of ``well-dispersed'' correlated tokens: see the discussion in \Cref{rmk:assumptions}.

We make the further following assumptions within this model:
\begin{enumerate}
    \item Inspired by an ablation in \citet{Yu2023-ig}, which suggests that the learned
        \ours{} model on supervised classification on ImageNet has signal models $\vU_k$
        which are near-incoherent, we will assume that the subspaces $\vU_k$ have pairwise orthogonal column spaces.
        Our proofs will generalize straightforwardly to the setting where the subspaces are merely incoherent: see the discussion in \Cref{rmk:assumptions}.
    \item We assume that the relative scales of these parameters conform to the
        \ours{}-Base settings, trained on ImageNet: from
        \Cref{tab:model_configs}, these parameters are
        \begin{enumerate}
            \item $d = 768$;
            \item $N = 196$;
            \item $K = 12$;
            \item $p = d / K = 64$.
        \end{enumerate}
        In particular, $d \gg N \gg p$ and $Kp=d$. 
\end{enumerate}
These precise parameter values will not play a role in our analysis. We merely require the following quantitative relationships between the parameter values, which are more general than the above precise settings.
\begin{assumption}\label{ass:parameter_config}
    We have $\veps \leq 1$, \(\vU_{k}^{\top}\vU_{k^{\prime}} = \Ind{k = k^{\prime}}\vI\) for all \(k \neq k^{\prime}\), and the following parameter settings and scales:
    \begin{itemize}
        \item \(d \geq N \geq p \geq K \geq 2\);
        \item \(Kp = d\);
        \item \(C_{1}\sqrt{N \log N} \leq \frac{1}{2}N/K\), where \(C_{1}\) is the same as the universal constant \(C_{1}\) in the statement of \Cref{prop:binom_concentration};
        \item \(6C_{2}^{2}N \leq d\), where \(C_{2}\) is the same as the universal constant \(C_{3}\) in the statement of \Cref{prop:gaussian_covariance_concentration_opnorm};
        \item \(2C_{4}^{2}N \leq d\), where \(C_{4}\) is the same as the universal constant \(C_{1}\) in \Cref{prop:xktxk_concentration_opnorm};
    \end{itemize}
\end{assumption}
\textbf{Note:} there is no self-reference, as the third inequality is not used to prove \Cref{prop:binom_concentration}, the fourth is not used to prove \Cref{prop:gaussian_covariance_concentration_opnorm}, and the fifth is not used to prove \Cref{prop:xktxk_concentration_opnorm}.

The first and second inequalities together imply in particular that \(p \geq N / K\). The third inequality implies that \(C_{1}\sqrt{N \log N} < N/K\).  The first, second, and and third inequalities together imply that \(p > C_{1}\sqrt{N \log N}\), and that \(0 < N/K - C_{1}\sqrt{N \log N} < N/K < N/K + C_{1}\sqrt{N \log N} < N\).

These inequalities are verifiable in practice if one wishes to explicitly compute the absolute constants \(C_{1}, C_{2}, C_{3}, C_{4}\), and indeed they hold for our \ours{}-Base model.

Formally, let $\mu(K, p, \sigma^2)$ denote the probability measure on $\bbR^{d \times N}$ 
corresponding to the noisy Gaussian mixture distribution specified above.
We let $\vZ_{\base} \sim \mu$ denote an observation distributed according to
this signal model: formally, there exists a (random) map $i \mapsto s_{i}$, for $i
\in [N]$ and $s_{i} \in [K]$, such that
\begin{equation}
    \vz_{\base i} = \vU_{s_{i}} \valpha_i + \vdelta_i, \quad i = 1, \dots, n,
    \label{eq:signal-model-vector-form}
\end{equation}
where $\vDelta = \begin{bmatrix} \vdelta_1 & \hdots & \vdelta_{N} \end{bmatrix}
\simiid \sN(\vZero, \tfrac{\sigma^2}{d} \vI)$, and (independently) $\valpha_i
\simiid \sN(\vZero,
\tfrac{1}{p}\vI)$.
It is convenient to write this observation model in block form. 
To this end, let $K_k = \sum_{i=1}^N \Ind{s_{i} = k}$ for $k \in [K]$
denote the number of times the $k$-th subspace is represented amongst the
columns of $\vZ_{\base}$ (a random variable).
Then by rotational invariance of the Gaussian distribution, we have
\begin{equation}
    \vZ_{\base}
    \equid
    \begin{bmatrix}
        \vU_1 \vA_1 & \hdots & \vU_K \vA_K
    \end{bmatrix}
    \vPi
    + \vDelta,
    \label{eq:signal-model-block-form}
\end{equation}
where $\equid$ denotes equality in distribution, $\vPi \in \bbR^{N \times N}$ is
a uniformly random permutation matrix, and each $\vA_k \in \bbR^{p \times K_k}$. We also %
define \(\vX_{\natural}\) to be the noise-free version of \(\vZ_{\base}\).

Because of this equality in distribution, we will commit the mild abuse of
notation of identifying the block representation
\Cref{eq:signal-model-block-form} with the observation model
\Cref{eq:signal-model-vector-form} that follows the distribution $\mu$.

\paragraph{Denoising in the uncorrelated tokens model.}
In the uncorrelated tokens model \Cref{eq:signal-model-block-form}, the marginal
distribution of each column of $\vZ_{\base}$ is identical, and equal to an
equiproportional mixture of (normalized) isotropic Gaussians on the subspaces
$\vU_1, \dots \vU_k$, convolved with the noise distribution $\sN(\vZero,
\tfrac{\sigma^2}{d} \vI)$. 
This marginal distribution was studied in \citet{Yu2023-ig}, where it was shown that
when the perturbation level $\sigma^2 \to 0$, the score function for this
marginal distribution approximately implements a projection operation onto the
nearest subspace $\vU_k$.

Hence, we can connect compression, as implemented in the MSSA block of the
\ours{} architecture, to denoising in the uncorrelated tokens model by showing
that at similar local scales, and for suitable settings of the model parameters,
the compression operation implements a projection onto the low-dimensional
structure of the distribution, as well.

\paragraph{Compression operation.} 
The MSSA block of the \ours{} architecture arises from taking an (approximate)
gradient step on the $R^c$ term of the sparse rate reduction objective
\Cref{eq:sparse_rr}. This term writes
\begin{equation}
    R^c(\vZ \mid \vU_{[K]}) = \frac{1}{2} \sum_{k=1}^K \log\det \left( \vI + \beta
    (\vU_k \adj \vZ)\adj \vU_k\adj \vZ \right),
\end{equation}
where
\begin{align}
    \beta &= \frac{p}{N \veps^2},
\end{align}
and $\veps > 0$ is the quantization error. 
Calculating the gradient, we have
\begin{equation}
    \nabla_{\vZ} R^c(\vZ \mid \vU_{[K]})
    =
    \sum_{k=1}^K \vU_k\vU_k\adj \vZ \left(
        \beta\inv \vI + (\vU_k\adj \vZ)\adj \vU_k\adj \vZ
    \right)\inv.
    \label{eq:deltar-gradient}
\end{equation}
Minimizing the sparse rate reduction objective corresponds to taking a gradient
descent step on $R^c(\spcdot \mid \vU_{[K]})$. Performing this operation at the
observation from the uncorrelated tokens model $\vZ_{\base}$, the output can be
written as
\begin{equation}
    \vZ^+ = \vZ_{\base} - \eta \nabla R^c(\vZ_{\base} \mid \vU_{[K]}),
    \label{eq:tokens-compressed}
\end{equation}
where $\eta > 0$ is the step size.

\paragraph{Main result on projection.} We will see shortly that the behavior of
the compression output \Cref{eq:tokens-compressed} depends on the relative
scales of the perturbation about the low-dimensional structure $\sigma^2$ and
the target quantization error $\veps^2$.

\begin{theorem}\label{lem:inverse-term}
    There are universal constants $C_{1}, C_{2}, C_{3}, C_{4} > 0$ such that the following holds.
    Suppose \Cref{ass:parameter_config} holds, and moreover suppose that 
    $\sigma \leq 1$ and $C_{1}\beta \sigma \leq \half$.
    Then with probability at least $1 - K C_{2}\bigl( e^{-C_{3} d} + e^{-C_{4}N/K} + 
    N^{-2}\bigr)$,
    it holds
    \begin{align}
        &\norm*{
        \vZ^{+}
        -
        \left[
            \left(\vDelta - \eta\sP_{\vU_{[K]}}(\beta\vDelta \vPi\adj) \vPi\right) + \frac{1 +\beta\inv - \eta}{1 + \beta^{-1}}\vX_{\natural}
        \right]}
        \\
        &\leq C_{5}
        K\eta\left(
        \sigma^2 \beta^2 + \sigma(1 + \sqrt{N/d}) + \sqrt{K} \beta
        \sigma^2(1 + \sqrt{N/d}) + \sqrt{N/d}
        \right).
    \end{align}
    Here,
    $\sP_{\vU_{[K]}}$ implements a projection onto the relevant subspaces for
    each token in the limiting case as $\veps \to 0$, and is precisely defined
    in \Cref{eq:proj-defn-1,eq:proj-defn-2}.
\end{theorem}

We give the proof of \Cref{lem:inverse-term} below. First, we make three remarks
on interpreting the result, our technical assumptions, and our analysis.

\begin{remark}
    \Cref{lem:inverse-term} admits the following interesting interpretation in
    an asymptotic setting, where we can identify the leading-order behavior of
    the gradient and confirm our hypothesis about the connection between
    compression and score-following.
    Choose $\eta = \beta\inv$, so that the guarantee in \Cref{lem:inverse-term}
    incurs some cancellation, and moreover delineate more precise dependencies
    on the RHS of the guarantee:
    \begin{align}
        &\norm*{
        \vZ^{+}
        -
        \left[
            \left(\vDelta - \sP_{\vU_{[K]}}(\vDelta \vPi\adj)
            \vPi\right) + \frac{1}{1 + \beta^{-1}}\vX_{\natural}
        \right]}
        \\
        &\quad\ltsim
        \frac{NK^2 \veps^2}{d}\left(
        \frac{\sigma^2 d^2}{N^2K^2 \veps^4} + \sigma(1 + \sqrt{N/d}) +
        \frac{d\sigma^2}{N\sqrt{K}\veps^2} (1 + \sqrt{N/d}) + \sqrt{N/d}
        \right)
        \\
        &\quad\ltsim
        K^{3/2}\sigma^2 
        + \frac{\sigma^2 d}{N\veps^2} 
        + \frac{NK^2}{d}\left(
        \sigma + \sqrt{\frac{N}{d}}
        \right)\veps^2,
    \end{align}
    where we used \Cref{ass:parameter_config}, which implies $p = d/K$ and
    $N/d\leq 1$.
    We will check in due course whether we have satisfied the hypotheses of
    \Cref{lem:inverse-term}, so that this guarantee indeed applies.
    To this end, we optimize this bound as a function of $\veps > 0$, since this is a
    parameter of the compression model.
    The optimal $\veps$ is straightforward to compute using calculus: it
    satisfies
    \begin{align}
        \veps^2 &= \sqrt{\left.\frac{\sigma^2 d}{N}\right/\frac{K^2 N}{d} \left(
        \sigma + \sqrt{\frac{N}{d}}\right)}
        \\
        &=
        \frac{\sigma d}{N K \sqrt{\sigma + \sqrt{\frac{N}{d}}}},
    \end{align}
    and the value of the residual arising from \Cref{lem:inverse-term} with this
    choice of $\veps$ is no larger than an absolute constant multiple of
    \begin{equation}
        K^{3/2}\sigma^2 
        + \sqrt{
            \frac{K^2 \sigma^2 d}{N} \left(
                \frac{N\sigma}{d} + \left( \frac{N}{d} \right)^{3/2}
            \right)
        }
        =
        K \sigma \left(
        \sqrt{K} \sigma
        + \sqrt{\sigma + \sqrt{\frac{N}{d}}}
        \right).
    \end{equation}
    Moreover, with this choice of $\veps$, $\beta$ satisfies
    \begin{equation}
        \beta\inv = \frac{\veps^2 NK}{d} = \sqrt{\frac{\sigma}{1 +
        \sqrt{\frac{N}{d\sigma^2}}}}.
    \end{equation}
    In particular, the condition $\beta \sigma \ltsim 1$ in
    \Cref{lem:inverse-term} demands
    \begin{equation}
        \sqrt{\sigma + \sqrt{\frac{N}{d}}} \ltsim 1,
    \end{equation}
    which holds for sufficiently small $\sigma$ and sufficiently large $d \geq
    N$, showing that \Cref{lem:inverse-term} can be nontrivially applied in this
    setting.
    If we consider a simplifying limiting regime where $N, d \to +\infty$ such
    that $N/d \to 0$ and $N/K \to +\infty$, we observe the following asymptotic
    behavior of the guarantee of \Cref{lem:inverse-term}:
    \begin{align}
        \norm*{
        \vZ^{+}
        -
        \left[
            \left(\vDelta - \sP_{\vU_{[K]}}(\vDelta \vPi\adj)
            \vPi\right) + \frac{1}{1 + \sqrt{\sigma}}\vX_{\natural}
        \right]}
        \ltsim
        K  \sigma^{3/2} \left( 1 + \sqrt{K\sigma} \right).
    \end{align}
    This demonstrates that a gradient step on $R^c$ performs denoising: there is
    a noise-level-dependent shrinkage effect applied to the signal
    $\vX_{\natural}$, which vanishes as $\sigma \to 0$, and meanwhile the noise
    term $\vDelta$ is reduced.

    Moreover, as $\sigma \to 0$, we can express the limiting form of
    $\sP_{\vU_{[K]}}$ exactly as an orthogonal projection, since this drives
    $\beta\inv \to 0$:
    following \Cref{eq:proj-defn-1,eq:proj-defn-2}, we have here
    \begin{equation}
        \sP_{U_{[K]}}
        =
        \begin{bmatrix}
            \sP_1 & \hdots & \sP_K   
        \end{bmatrix},
    \end{equation}
    where
    \begin{equation}
        \sP_k \to 
        \sum_{k' \neq k}
        \vU_{k'} \proj{\im(\vA_{k'})^\perp} \vU_{k'}\adj.
    \end{equation}
    This shows that, in an asymptotic sense, a gradient step on $R^c$ serves to \textit{suppress
    the effect of the perturbation applied to the observations $\vZ_{\base}$
    about the local signal model $\vX_{\natural}$.} This verifies our claim
    previously that in this setting, there is a correspondence between a
    score-following algorithm and a compression-based approach: locally, both
    project the observations onto the structures of the signal model.

    It can be shown moreover that the shrinkage effect on $\vX_{\natural}$
    demonstrated here appears as a consequence of using the $R^c$
    ``compression'' term for the gradient step in \ours{}; when the gradient
    step is taken instead on the full $\Delta R$ rate reduction objective (which
    is computationally prohibitive, of course), there is zero shrinkage, and
    perfect denoising is performed for a wider variety of step sizes $\eta$ than the choice made here. We see the introduction of this shrinkage
    effect this as the price of constructing an efficient and interpretable
    network architecture. In practice, the ISTA block of \ours{} counteracts
    this shrinkage effect, which is anyways minor at reasonable parameter scales.
\end{remark}

\begin{remark}\label{rmk:assumptions}
     We have made two assumptions which may not hold exactly in practice: namely, we have assumed that the \(\vU_{k}\)'s have orthogonal columns, namely \(\vU_{k}^{\top}\vU_{k^{\prime}} = \Ind{k = k^{\prime}}\vI\), and we have assumed that the linear combination coefficients \(\vA_{k}\) that form the matrix \(\vX_{\natural}\) are i.i.d.~samples from Gaussian distributions. Both these assumptions can be made more realistic, at the cost of additional (non-instructive) complexity in the analysis; we briefly go over how.
     
     Relaxing the orthogonality condition \(\vU_{k}^{\top}\vU_{k^{\prime}} = \Ind{k = k^{\prime}}\vI\) to near-orthogonality, namely \(\norm{\vU_{k}^{\top}\vU_{k^{\prime}} - \Ind{k = k^{\prime}}\vI} \leq \nu\) for a small \(\nu\), as observed in practice \citep{Yu2023-ig} would introduce additional small error terms in the proof, say polynomial in \(\nu\). The magnitudes of these errors could in principle be precisely tracked, whence one could obtain a similar result to \Cref{lem:inverse-term}.

    Secondly, we have assumed that the \(\vA_{k}\)'s have independent columns which are sampled from (the same) Gaussian distribution. However, in the conceptual framework for \ours{}, we exploit the joint distribution (and in particular the correlations) between the tokens in order to obtain good performance for our model. Our analysis is not completely agnostic to this fact; as we will see, the proof of \Cref{lem:inverse-term} only leverages the independence of the columns of each \(\vA_{k}\)'s in order to obtain high-probability upper bounds on the smallest and largest singular value of the token matrices. If these bounds were ensured by some other method, such as appropriate normalization and incoherence, a similar conclusion to \Cref{lem:inverse-term} could hold in the more realistic correlated tokens model. Going beyond well-conditioned token matrices for each subspace would require additional modeling assumptions, and additional investigative experimental work to determine a realistic basis for such assumptions.
\end{remark}

\begin{remark}
    We have not attempted to optimize constants or rates of concentration in the
    proof of \Cref{lem:inverse-term}, preferring instead to pursue a
    straightforward analysis that leads to a qualitative interpretation of the
    behavior of the rate reduction gradient in our model problem. Minor
    improvements to the concentration analysis would enable the parameter
    scaling requirements in \Cref{ass:parameter_config} to be relaxed slightly,
    and the probability bound in \Cref{lem:inverse-term} that scales as $K /
    N^2$ can easily be improved to any positive power of $1/N$.
\end{remark}

\begin{proof}[Proof of \Cref{lem:inverse-term}]
    We start by noticing that, by orthonormality of the subspaces $\vU_k$,
    we have by \Cref{eq:signal-model-block-form}
    \begin{equation}
        \vU_k \adj \vZ_{\base}
        =
        \begin{bmatrix}
            \vZero & \hdots & \vA_k & \hdots & \vZero
        \end{bmatrix}
        \vPi
        + \vU_k\adj \vDelta,
    \end{equation}
    so that
    \begin{equation}
    \small
        \left(
        \beta\inv \vI + (\vU_k\adj \vZ_{\base})\adj \vU_k\adj
        \vZ_{\base}
        \right)\inv
        =
        \vPi\adj
        \left(
        \underbrace{
            \begin{bmatrix}
                \beta\inv\vI  &&&& \\
                &\ddots &&& \\
                && \beta\inv \vI + \vA_k\adj \vA_k && \\
                &&&\ddots & \\
                &&&&\beta\inv \vI  \\
            \end{bmatrix}
        }_{\vD_k}
        + \vXi_k
        \right)\inv
        \vPi,\label{eq:ak-defn}
    \end{equation}
    because permutation matrices are orthogonal matrices, and where the
    perturbation $\vXi_k$ is defined by
    \begin{equation}
        \vXi_k = 
        \vPi \vDelta\adj \vU_k \vU_k\adj \vDelta\vPi\adj
        +
        \begin{bmatrix}
            \vZero & \hdots & \vDelta_1\adj \vU_k \vA_k & \hdots & \vZero \\
            \vdots & & \vdots & & \vdots \\
            \vA_k\adj \vU_k\adj \vDelta_1 & \hdots & \vDelta_k\adj \vU_k \vA_k +
            \vA_k\adj \vU_k\adj \vDelta_k & \hdots & \vA_k\adj \vU_k\adj \vDelta_K\\
            \vdots & & \vdots & & \vdots \\
            \vZero & \hdots & \vDelta_K\adj \vU_k \vA_k & \hdots & \vZero \\
        \end{bmatrix},
        \label{eq:xik-defn}
    \end{equation}
    and where we have defined (implicitly) in addition
    \begin{equation}
        \begin{bmatrix}
            \vDelta_1 & \hdots & \vDelta_K
        \end{bmatrix}
        = \vDelta \vPi\adj.
    \end{equation}
    The matrix $\vD_k \succ \vZero$, so we can write
    \begin{equation}
        \left(
        \beta\inv \vI + (\vU_k\adj \vZ_{\base})\adj \vU_k\adj
        \vZ_{\base}
        \right)\inv
        =
        \vPi\adj
        \vD_k\inv
        \left(
        \vI + \vXi_k\vD_k\inv
        \right)\inv \vPi,
    \end{equation}
    from which it follows
    \begin{align}
        &\vU_k\adj \vZ_{\base}
        \left(
        \beta\inv \vI + (\vU_k\adj \vZ_{\base})\adj \vU_k\adj
        \vZ_{\base}
        \right)\inv
        \\
        &\quad=
        \left(
        \begin{bmatrix}
            \vZero & \hdots & \vA_k(\beta\inv \vI + \vA_k\adj \vA_k)\inv &
            \hdots & \vZero
        \end{bmatrix}
        + \vU_k\adj \vDelta\vPi\adj \vD_k\inv
        \right)
        \left(
        \vI + \vXi_k\vD_k\inv
        \right)\inv \vPi.
    \end{align}
    The task before us is therefore to control $\norm{\vXi_k \vD_k \inv} <
    1$, in order to apply the Neumann series to further simplify this
    expression. We will do this in stages: first, we invoke several
    auxiliary lemmas to construct a high-probability event on which the
    random quantities in the preceding expression are controlled about their
    nominal values; next, we show that the Neumann series can be applied on
    this event and a main term extracted; finally, we simplify this main
    term further in order to establish the claimed expression.

    \paragraph{High-probability event construction.} In order to achieve the appropriate control on all random quantities, we would like to construct a high-probability event on which the random quantities are not too large. By \Cref{prop:delta_concentration_opnorm,prop:xk_concentration_opnorm,prop:xktxk_concentration_opnorm} and union bound, there exist universal constants \(C_{i} > 0\) for which
    \begin{equation}
        \Pr*{
            \begin{array}{rl}
                & \norm{\vDelta} \leq \sigma(C_{1} + \sqrt{N/d}) \\ 
                \forall k \in [K] \colon& \norm{\vA_{k}} \leq 1 + C_{2}\sqrt{N/d} \\ 
                \forall k \in [K] \colon& \norm{\vA_{k}^{\top}\vA_{k} - \vI} \leq C_{3}\sqrt{N/d}
            \end{array}
        } \geq 1 - C_{4}K(e^{-C_{5}d} + e^{-C_{6}N/K} + N^{-2}).
    \end{equation}
    The event we compute the probability of, which we denote \(E^{\star}\), is precisely the good event that we want. Formally, 
    \begin{equation}
        E^{\star} \doteq \bc{
            \begin{array}{rl}
                & \norm{\vDelta} \leq \sigma(C_{1} + \sqrt{N/d}) \\ 
                \forall k \in [K] \colon& \norm{\vA_{k}} \leq 1 + C_{2}\sqrt{N/d} \\ 
                \forall k \in [K] \colon& \norm{\vA_{k}^{\top}\vA_{k} - \vI} \leq C_{3}\sqrt{N/d}
            \end{array}
        }. \label{eq:good_event}
    \end{equation}
    We know that \(E^{\star}\) occurs with high probability, and are able to strongly control the random quantities to the degree desired.

    \paragraph{Main term extraction.}
    By \Cref{lem:neumann-op-norm} and our hypotheses on the problem parameters, we have on \(E^{\star}\) that
    \begin{equation}
        \norm{\vXi_k \vD_k \inv} 
        \leq
        C \beta \sigma
        <
        1.
    \end{equation}  
    We can therefore apply the Neumann series to obtain
    \begin{align}
        &\vU_k\adj \vZ_{\base}
        \left(
        \beta\inv \vI + (\vU_k\adj \vZ_{\base})\adj \vU_k\adj
        \vZ_{\base}
        \right)\inv
        \\
        =&
        \scalebox{0.8}{\(\left(
        \begin{bmatrix}
            \vZero & \hdots & \vA_k(\beta\inv \vI + \vA_k\adj \vA_k)\inv &
            \hdots & \vZero
        \end{bmatrix}
        + \vU_k\adj \vDelta\vPi\adj \vD_k\inv
        \right)
        \left(
        \vI - 
        \vXi_k \vD_k^{-1}
        +
        \sum_{j=2}^\infty
        (-1)^j \left( \vXi_k \vD_k^{-1} \right)^j
        \right) \vPi\)}.
    \end{align}
    Again on \(E^{\star}\), we have
    \begin{equation}
        \norm*{
            \sum_{j=2}^\infty
            (-1)^j \left( \vXi_k \vD_k^{-1} \right)^j
        }
        \leq
        \sum_{j=2}^\infty
        \norm*{
            \vXi_k \vD_k^{-1}
        }^j
        \leq
        C (\beta \sigma)^2
        \frac{
            1
        }{
            1 - C \beta \sigma
        }
        \leq
        C' (\beta \sigma)^2.
    \end{equation}
    Moreover, as in the proof of \Cref{lem:neumann-op-norm}, we have on the previous event
    that
    \begin{equation}
        \norm*{
            \vU_k\adj \vDelta\vPi\adj \vD_k\inv
        }
        \leq
        C \beta \sigma.
    \end{equation}
    Thus, if we define a ``main term''
    \begin{equation}
        \vM_k
        =
        \left[
            \begin{bmatrix}
                \vZero & \hdots & \vA_k(\beta\inv \vI + \vA_k\adj \vA_k)\inv &
                \hdots & \vZero
            \end{bmatrix}
            \left(
            \vI - 
            \vXi_k \vD_k^{-1}
            \right)
            + \vU_k\adj \vDelta\vPi\adj \vD_k\inv
            \right]\vPi,
    \end{equation}
    we have on the same event as previously
    \begin{equation}
        \norm*{
            \vU_k\adj \vZ_{\base}
            \left(
            \beta\inv \vI + (\vU_k\adj \vZ_{\base})\adj \vU_k\adj
            \vZ_{\base}
            \right)\inv
            - \vM_k
        }
        \leq C (\beta \sigma)^2.
    \end{equation}
    To conclude, we need only study this main term, since $\vU_k$ has operator norm $1$.

    \paragraph{Simplifying the main term.}
    Our approach will be to control the main term $\vM_k$ around a simpler
    expression, using basic perturbation theory; by the triangle inequality
    for the operator norm, this will give control of the desired gradient
    term.
    After distributing, $\vM_k$ is a sum of three terms; we will start with
    the simplest term. We first compute
    \begin{equation}
        \vU_k\adj \vDelta\vPi\adj \vD_k\inv
        =
        \vU_k\adj
        \begin{bmatrix}
            \beta \vDelta_1 & \hdots 
            \vDelta_k \left(
            \beta\inv \vI + \vA_k\adj \vA_k\right)\inv
            & \hdots & \beta \vDelta_K
        \end{bmatrix}.
    \end{equation}
    We are going to argue that the residual
    \begin{equation}
        \norm*{
            \vU_k\adj \vDelta\vPi\adj \vD_k\inv
            -
            \vU_k\adj
            \begin{bmatrix}
                \beta \vDelta_1 & \hdots &
                \vZero
                & \hdots & \beta \vDelta_K
            \end{bmatrix}
        }
    \end{equation}
    is small. To this end, note that by the fact that $\vU_k$ has unit
    operator norm,
    \begin{align}
        &\norm*{
            \vU_k\adj \vDelta\vPi\adj \vD_k\inv
            -
            \vU_k\adj
            \begin{bmatrix}
                \beta \vDelta_1 & \hdots &
                \vZero
                & \hdots & \beta \vDelta_K
            \end{bmatrix}
        } \\
        &\leq
        \norm*{
            \begin{bmatrix}
                \vZero & \hdots &
                \vDelta_k \left(
                \beta\inv \vI + \vA_k\adj \vA_k\right)\inv
                & \hdots & \vZero
            \end{bmatrix}
        }
        \\
        &=
        \norm*{
            \vDelta_k \left(
            \beta\inv \vI + \vA_k\adj \vA_k\right)\inv
        } \\
        &\leq
        \norm*{
            \vDelta_k 
        }
        \norm*{
            \left(
            \beta\inv \vI + \vA_k\adj \vA_k\right)\inv
        }.
    \end{align}
    By \Cref{eq:good_event} and \Cref{eq:pseudoinverse-ipart-small} from
    \Cref{lem:blockwise-pinv}, the second term 
    here is controlled on $E^\star$.
    For the first term, we note that by definition and the fact that
    the unit sphere is invariant to rotations (and permutations are
    rotations),
    \begin{align}
        \norm*{\vDelta}
        =
        \sup_{\norm{\vu}_2 \leq 1}
        \norm*{\vDelta \vu}_2
        &=
        \sup_{\norm{\vu}_2 \leq 1}
        \norm*{\begin{bmatrix}
            \vDelta_1 & \hdots & \vDelta_K
        \end{bmatrix} \vu
        }_2
        \\
        &=
        \sup_{\norm{\vu}_2 \leq 1}
        \norm*{
            \sum_{i=1}^K
            \vDelta_i \vu_i
        }_2,
    \end{align}
    where $\vu_i$ are coordinate-subset-induced partitions of the vector
    $\vu$ induced by those of $\vDelta\vPi\adj$.
    This yields immediately
    \begin{equation}
        \norm*{\vDelta}
        \leq
        \sup_{\norm{\vu}_2 \leq 1}
        \sum_{i=1}^K 
        \norm*{
            \vDelta_i \vu_i
        }_2
        \leq
        \left( \max_{k \in [K]}\, \norm{\vDelta_k}\right)
        \sup_{\norm{\vu}_2 \leq 1}
        \sum_{i=1}^K \norm{\vu_i}_2
        \leq
        \sqrt{K}
        \left( \max_{k \in [K]}\, \norm{\vDelta_k}\right),
        \label{eq:control-op-by-blocks}
    \end{equation}
    by the triangle inequality and inequalities for $\ell^p$ norms.
    Similarly, choosing a specific $\vu$ in the operator norm expression,
    namely one that is supported entirely on one of the coordinate
    partitions $\vu_i$, shows that
    \begin{align}
        \norm*{\vDelta}
        &\geq
        \norm*{
            \vDelta_i \vu_i
        }_2
    \end{align}
    for any $i$, whence
    \begin{equation}
        \max_{k \in [K]}\, \norm{\vDelta_k}
        \leq
        \norm*{\vDelta}.
        \label{eq:blocks-opnorm-controlled-overall}
    \end{equation}
    It follows that we control the first term above on $E^\star$.
    Combining this reasoning, we conclude from the above
    \begin{align}
        &\norm*{
            \vU_k\adj \vDelta\vPi\adj \vD_k\inv
            -
            \vU_k\adj
            \begin{bmatrix}
                \beta \vDelta_1 & \hdots &
                \vZero
                & \hdots & \beta \vDelta_K
            \end{bmatrix}
        }\\
        \leq&
        \sigma(C + \sqrt{N/d})
        \left(
        \frac{1}{1 + \beta\inv}
        +
        \frac{C' \sqrt{N/d}}{1 + \beta\inv}
        \right)
        \\
         \ltsim &\sigma(1 + C\sqrt{N/d}),
    \end{align}
    where the second line uses \Cref{ass:parameter_config} to 
    remove the higher-order residual.

    Next, we recall that $\vXi_k$ is a sum of two terms; we will do one term
    at a time for concision. We have first
    \begin{align}
        &\begin{bmatrix}
            \vZero & \hdots & \vA_k(\beta\inv \vI + \vA_k\adj \vA_k)\inv &
            \hdots & \vZero
        \end{bmatrix}
        \vPi \vDelta\adj \vU_k \vU_k\adj \vDelta\vPi\adj
        \\
        &=
        \begin{bmatrix}
            \vZero & \hdots & \vA_k(\beta\inv \vI + \vA_k\adj \vA_k)\inv &
            \hdots & \vZero
        \end{bmatrix}
        \begin{bmatrix}
            \vDelta_1\adj \\
            \vdots\\
            \vDelta_K\adj
        \end{bmatrix}
        \vU_k \vU_k\adj
        \begin{bmatrix}
            \vDelta_1\adj \\
            \vdots\\
            \vDelta_K\adj
        \end{bmatrix}\adj
        \\
        &=
        \vA_k(\beta\inv \vI + \vA_k\adj \vA_k)\inv \vDelta_k\adj \vU_k
        \vU_k\adj
        \begin{bmatrix}
            \vDelta_1 &
            \hdots&
            \vDelta_K
        \end{bmatrix}.
    \end{align}
    We then multiply this term by $\vD_k\inv$ on the right to get the term
    that appears in $\vM_k$ (ignoring the multiplication on the right by
    $\vPi$, because it does not change operator norms). 
    In particular, we will control
    \begin{equation}
        \norm*{
            \vA_k(\beta\inv \vI + \vA_k\adj \vA_k)\inv \vDelta_k\adj \vU_k
            \vU_k\adj
            \begin{bmatrix}
                \vDelta_1 &
                \hdots&
                \vDelta_K
            \end{bmatrix}\vD_k\inv
        },
    \end{equation}
    showing that this term is small.
    We will accomplish this with the block diagonal structure of $\vD_k$:
    indeed, this gives that $\vD_k\inv$ is obtained by blockwise inversion,
    and hence
    \begin{align}
        &\norm*{
            \vA_k(\beta\inv \vI + \vA_k\adj \vA_k)\inv \vDelta_k\adj \vU_k
            \vU_k\adj
            \begin{bmatrix}
                \vDelta_1 &
                \hdots&
                \vDelta_K
            \end{bmatrix}\vD_k\inv
        }
        \\
        &\quad=
        \scalebox{0.9}{\(\norm*{
            \vA_k(\beta\inv \vI + \vA_k\adj \vA_k)\inv \vDelta_k\adj \vU_k
            \vU_k\adj
            \begin{bmatrix}
                \beta\vDelta_1
                &\hdots&
                \vDelta_k \left( \beta\inv \vI + \vA_k\adj \vA_k \right)\inv
                &\hdots&
                \beta \vDelta_K
            \end{bmatrix}
        }\)}
        \\
        &\quad\leq
        \sqrt{K}\max \Bigl\{
            \norm*{
                \vA_k(\beta\inv \vI + \vA_k\adj \vA_k)\inv \vDelta_k\adj \vU_k
                \vU_k\adj\vDelta_{k}(\beta\inv \vI + \vA_k\adj \vA_k)\inv
            },
            \\
            &\quad\qquad\qquad
            \max_{k' \neq k}\,
            \beta \norm*{
                \vA_k(\beta\inv \vI + \vA_k\adj \vA_k)\inv \vDelta_k\adj \vU_k
                \vU_k\adj\vDelta_{k'}
            }
            \Bigr\},
    \end{align}
    where the second line uses \Cref{eq:control-op-by-blocks}.
    We will give a coarse control of this term---the error could be improved
    further by exploiting more thoroughly independence of the blocks
    $\vDelta_k$ to show that the maximum over $k' \neq k$ in the last line
    of the preceding expression is smaller.
    We have by submultiplicativity of the operator norm
    and the triangle inequality
    \begin{align}
        &\norm*{
            \vA_k(\beta\inv \vI + \vA_k\adj \vA_k)\inv \vDelta_k\adj \vU_k
            \vU_k\adj\vDelta_{k}(\beta\inv \vI + \vA_k\adj \vA_k)\inv
        }\\
        \leq&
        \left(
        \frac{1}{1 + \beta \inv}
        +
        \frac{C \sqrt{N/d}}{1 + \beta \inv}
        \right)^2
        \norm*{
            \vDelta_k\adj \vU_k \vU_k\adj\vDelta_{k}
        }
        \\
        \leq&
        \left(1 + C \sqrt{N/d} \right)
        \norm*{
            \vDelta_k\adj \vU_k \vU_k\adj\vDelta_{k}
        },
    \end{align}
    where the first line uses \Cref{lem:blockwise-pinv}, and the second line
    uses \Cref{ass:parameter_config} to simplify the residual as above.
    We have meanwhile
    from the definition of $E^\star$
    \begin{align}
        \norm*{
            \vDelta_k\adj \vU_k \vU_k\adj\vDelta_{k}
        }
        &\leq
        \norm*{
            \vDelta_k
        }^2
        \ltsim \sigma^2\left(1 + \sqrt{N/d}\right),
    \end{align}
    because $\vU_k\vU_k\adj$ is an orthogonal projection, and again using
    \Cref{ass:parameter_config} to simplify the residual.
    We can argue analogously to simplify the other term in the maximum
    appearing above, and this yields
    \begin{align}
        &\norm*{
            \vA_k(\beta\inv \vI + \vA_k\adj \vA_k)\inv \vDelta_k\adj \vU_k
            \vU_k\adj
            \begin{bmatrix}
                \vDelta_1 &
                \hdots&
                \vDelta_K
            \end{bmatrix}\vD_k\inv
        }
        \\
        &\quad\ltsim
        \sqrt{K}\beta\sigma^2 \left(1 + C \sqrt{N/d} \right) \left( 1 +
        \sqrt{N/d} \right),
    \end{align}
    where we used the fact that $\veps \leq 1$ and the rest of
    \Cref{ass:parameter_config} implies that $\beta \geq 1$.
    This residual simplifies using \Cref{ass:parameter_config} to
    \begin{equation}
        \norm*{
            \vA_k(\beta\inv \vI + \vA_k\adj \vA_k)\inv \vDelta_k\adj \vU_k
            \vU_k\adj
            \begin{bmatrix}
                \vDelta_1 &
                \hdots&
                \vDelta_K
            \end{bmatrix}\vD_k\inv
        }
        \ltsim
        \sqrt{K}\beta\sigma^2 \left(1 + C \sqrt{N/d} \right).
    \end{equation}

    Next, we examine the last term, which is the other summand arising in
    the definition of $\vXi_k$. We have
    \begin{align}
        &\scalebox{0.85}{\(\begin{bmatrix}
            \vZero & \hdots & \vA_k(\beta\inv \vI + \vA_k\adj \vA_k)\inv &
            \hdots & \vZero
        \end{bmatrix}
        \begin{bmatrix}
            \vZero & \hdots & \vDelta_1\adj \vU_k \vA_k & \hdots & \vZero \\
            \vdots & & \vdots & & \vdots \\
            \vA_k\adj \vU_k\adj \vDelta_1 & \hdots & \vDelta_k\adj \vU_k \vA_k +
            \vA_k\adj \vU_k\adj \vDelta_k & \hdots & \vA_k\adj \vU_k\adj \vDelta_K\\
            \vdots & & \vdots & & \vdots \\
            \vZero & \hdots & \vDelta_K\adj \vU_k \vA_k & \hdots & \vZero \\
        \end{bmatrix}
        \)}
        \\
        &=
        \scalebox{0.95}{\(\vA_k(\beta\inv \vI + \vA_k\adj \vA_k)\inv
        \begin{bmatrix}
            \vA_k\adj \vU_k\adj \vDelta_1
            & \hdots 
            & 
            \left(
            \vDelta_k\adj \vU_k \vA_k + \vA_k\adj\vU_k\adj \vDelta_k
            \right)
            & \hdots 
            &%
            \vA_k\adj \vU_k\adj \vDelta_K
        \end{bmatrix}\)}.
    \end{align}
    Now multiplying on the right by $\vD_k\inv$ gives the term (again
    ignoring the right-multiplication by $\vPi$, which does not affect the
    operator norm)
    \begin{equation}
        \scalebox{0.8}{\(\vA_k(\beta\inv \vI + \vA_k\adj \vA_k)\inv
        \begin{bmatrix}
            \beta
            \vA_k\adj \vU_k\adj \vDelta_1
            & \hdots 
            & 
            \left(
            \vDelta_k\adj \vU_k \vA_k + \vA_k\adj\vU_k\adj \vDelta_k
            \right)
            \left(
            \beta\inv \vI + \vA_k \adj \vA_k
            \right)\inv
            & \hdots 
            &%
            \beta \vA_k\adj \vU_k\adj \vDelta_K
        \end{bmatrix}\)}.
        \label{eq:other-term-of-xik}
    \end{equation}
    We will argue that this term is close to the term
    \begin{equation}
        \vA_k(\beta\inv \vI + \vA_k\adj \vA_k)\inv
        \begin{bmatrix}
            \beta
            \vA_k\adj \vU_k\adj \vDelta_1
            & \hdots 
            & 
            \vZero
            & \hdots 
            &%
            \beta \vA_k\adj \vU_k\adj \vDelta_K
        \end{bmatrix}
        \label{eq:other-term-of-xik-nom}
    \end{equation}
    in operator norm. The argument is similar to the preceding arguments:
    for this, it suffices to bound
    \begin{equation}
        \scalebox{0.9}{\(\norm*{
            \begin{bmatrix}
                \vZero
                & \hdots 
                & 
                \vA_k(\beta\inv \vI + \vA_k\adj \vA_k)\inv
                \left(
                \vDelta_k\adj \vU_k \vA_k + \vA_k\adj\vU_k\adj \vDelta_k
                \right)
                \left(
                \beta\inv \vI + \vA_k \adj \vA_k
                \right)\inv
                & \hdots 
                &%
                \vZero
            \end{bmatrix}
        }\)},
    \end{equation}
    which is the same as controlling the operator norm of the nonzero block.
    Using submultiplicativity and \Cref{lem:blockwise-pinv} along with the
    simplifications we have done above leveraging
    \Cref{ass:parameter_config}, we obtain
    \begin{align}
        &\norm*{
            \vA_k(\beta\inv \vI + \vA_k\adj \vA_k)\inv
            \left(
            \vDelta_k\adj \vU_k \vA_k + \vA_k\adj\vU_k\adj \vDelta_k
            \right)
            \left(
            \beta\inv \vI + \vA_k \adj \vA_k
            \right)\inv
        }
        \\
        &\quad\leq
        \left( 1 + C \sqrt{N/d} \right)
        \norm*{
            \vDelta_k\adj \vU_k \vA_k + \vA_k\adj\vU_k\adj \vDelta_k
        }.
    \end{align}
    Meanwhile, on $E^\star$ we have the operator norm of $\vDelta_k$ and
    $\vA_k$ controlled, using again
    \Cref{eq:blocks-opnorm-controlled-overall}. Applying then 
    the triangle inequality and submultiplicativity, we obtain
    \begin{equation}
        \norm*{
            \vDelta_k\adj \vU_k \vA_k + \vA_k\adj\vU_k\adj \vDelta_k
        }
        \ltsim
        \sigma
        \left(
        1 + \sqrt{N/d}
        \right),
    \end{equation}
    again simplifying with \Cref{ass:parameter_config}.
    This shows that \Cref{eq:other-term-of-xik} is close to
    \Cref{eq:other-term-of-xik-nom}
    with deviations of the order $\ltsim \sigma(1 + \sqrt{N/d})$.

    \paragraph{Aggregating the previous results.}
    Combining our perturbation analysis above, we have established control
    \begin{align}
        &\Bigl\|
        \vM_k
        -
        \Bigl[
            \left(
            \vI - 
            \vA_k(\beta\inv \vI + \vA_k\adj \vA_k)\inv\vA_k\adj
            \right)
            \vU_k\adj
            \begin{bmatrix}
                \beta \vDelta_1
                &\hdots&
                \vZero
                &\hdots&
                \beta \vDelta_K
            \end{bmatrix}
            \\
            &\quad+
            \begin{bmatrix}
                \vZero & \hdots & \vA_k(\beta\inv \vI + \vA_k\adj
                \vA_k)\inv & \hdots & \vZero
            \end{bmatrix}
            \Bigr]\vPi
        \Bigr\|
        \\
        &\ltsim
        \sigma(1 + \sqrt{N/d})
        +
        \sqrt{K} \beta \sigma^2(1 + \sqrt{N/d} ).
    \end{align}
    It is convenient to include one additional stage of simplification here:
    namely, we use \Cref{lem:blockwise-pinv} once more to simplify the
    second term in the nominal value of $\vM_k$ appearing here.
    Namely, we have (arguing as we have above, once again)
    \begin{align}
        &\norm*{
            \begin{bmatrix}
                \vZero & \hdots & \vA_k(\beta\inv \vI + \vA_k\adj
                \vA_k)\inv & \hdots & \vZero
            \end{bmatrix}
            -
            \begin{bmatrix}
                \vZero & \hdots & \frac{1}{1+\beta\inv}\vA_k & \hdots & \vZero
            \end{bmatrix}
        }
        \\
        &\quad=
        \norm*{
            \frac{1}{1+\beta\inv}\vA_k - \vA_k\left(
            \beta\inv \vI + \vA_k\adj \vA_k
            \right)\inv
        }
        \\
        &\quad\ltsim
        \sqrt{N/d},
    \end{align}
    from which it follows
    \begin{align}
        &\Bigl\|
        \vM_k
        -
        \Bigl[
            \left(
            \vI - 
            \vA_k(\beta\inv \vI + \vA_k\adj \vA_k)\inv\vA_k\adj
            \right)
            \vU_k\adj
            \begin{bmatrix}
                \beta \vDelta_1
                &\hdots&
                \vZero
                &\hdots&
                \beta \vDelta_K
            \end{bmatrix}
            \\
            &\quad+
            \begin{bmatrix}
                \vZero & \hdots & \frac{1}{1+\beta\inv}\vA_k & \hdots & \vZero
            \end{bmatrix}
            \Bigr]\vPi
        \Bigr\|
        \\
        &\ltsim
        \sigma(1 + \sqrt{N/d})
        +
        \sqrt{K} \beta \sigma^2(1 + \sqrt{N/d} ) + \sqrt{N/d}.
    \end{align}
    Meanwhile, recall the residual of scale $\ltsim (\sigma \beta)^2$
    arising when we controlled the gradient term around $\vM_k$:
    \begin{equation}
        \norm*{
            \vU_k\adj \vZ_{\base}
            \left(
            \beta\inv \vI + (\vU_k\adj \vZ_{\base})\adj \vU_k\adj
            \vZ_{\base}
            \right)\inv
            - \vM_k
        }
        \leq C (\beta \sigma)^2.
    \end{equation}
    Combining these two bounds with the triangle inequality controls the
    gradient term around its nominal value. Now, we sum these errors over
    $k$ (again with the triangle inequality) to obtain control of the
    aggregate gradient around its nominal value.
    We introduce notation to concisely capture the
    accumulations of the (approximate) orthogonal projections arising in the
    nominal value of the main term: for each $k \in [K]$, define
    \begin{equation}
        \sP_k = 
        \sum_{k' \neq k}
        \vU_{k'}
        \left(
        \vI - 
        \vA_{k'}(\beta\inv \vI + \vA_{k'}\adj \vA_{k'})\inv\vA_{k'}\adj
        \right)
        \vU_{k'}\adj,
        \label{eq:proj-defn-1}
    \end{equation}
    and define an overall (approximate) projection operator (which acts on block matrices
    partitioned compatibly with the class sizes $N_k$, as in
    \Cref{eq:signal-model-block-form}) by
    \begin{equation}
        \sP_{U_{[K]}}
        =
        \begin{bmatrix}
            \sP_1 & \hdots & \sP_K   
        \end{bmatrix}.
        \label{eq:proj-defn-2}
    \end{equation}
    Then the above argument implies
    \begin{align}
        &\;\norm*{
            \nabla_{\vZ} R^c(\vZ_{\base} \mid \vU_{[K]})
            -
            \sP_{\vU_{[K]}}(\beta \vDelta \vPi\adj) \vPi
            - \frac{1}{1 + \beta\inv}\vX_{\natural}
        }
        \\
        \ltsim&\;
        K\left(
        \sigma^2 \beta^2 + \sigma(1 + \sqrt{N/d}) + \sqrt{K} \beta
        \sigma^2(1 + \sqrt{N/d}) + \sqrt{N/d}
        \right),
    \end{align}
    which is enough to conclude.

\end{proof}

\subsubsection{Key Auxiliary Lemmas}

In this section we state and prove two key concentration inequalities that are used in the proof of the main theorem. They rely on simpler results which will be conveyed in subsequent subsections.

\begin{lemma}\label{lem:blockwise-pinv}
    There exist universal constants \(C, C' > 0\) such that the following holds. Let \(d, p, N, K \in \N\) be such that \Cref{ass:parameter_config} holds. Let \(\vA_{k}\), \(k \in [K]\), be defined as above. Let \(E^{\star}\) be the good event defined in \Cref{eq:good_event}. If \(E^{\star}\) occurs, then for \(k \in [K]\) we have
    \begin{equation}
        \norm*{(\beta^{-1}\vI + \vA_{k}^{\top}\vA_{k})^{-1} - \frac{1}{1 +
        \beta^{-1}} \vI} 
        \leq \frac{C\sqrt{N/d}}{(1 + \beta^{-1})}.
        \label{eq:pseudoinverse-ipart-small}
    \end{equation}
    and in addition
    \begin{equation}
        \norm*{\vA_{k}(\beta^{-1}\vI + \vA_{k}^{\top}\vA_{k})^{-1} - \frac{1}{1
        + \beta^{-1}}\vA_{k}} \leq \frac{C'\sqrt{N/d}}{(1 + \beta^{-1})}.
        \label{eq:pseudoinverse-small}
    \end{equation}
\end{lemma}
\begin{proof}
    Since \(E^{\star}\) holds, for all \(k \in [K]\) we have
    \begin{equation}
        \norm{\vA_{k}} \leq 1 + C_{1}\sqrt{N/d}, \qquad \norm{\vA_{k}^{\top}\vA_{k} - \vI} \leq C_{2}\sqrt{N/d}.
    \end{equation}
    By \Cref{ass:parameter_config}, we have \(\norm{\vA_{k}^{\top}\vA_{k} - \vI} < 1\), so \(\vA_{k}^{\top}\vA_{k}\) is well-conditioned. Write
    \begin{equation}
        \vXi \doteq \vA_{k}^{\top}\vA_{k} - \vI,
    \end{equation}
    so that
    \begin{align}
        (\beta^{-1}\vI + \vA_{k}^{\top}\vA_{k})^{-1}
        &= ((1 + \beta^{-1})\vI + \vXi)^{-1} \\
        &= \frac{1}{1 + \beta^{-1}}\left(\vI + \frac{1}{1 + \beta^{-1}}\vXi\right)^{-1} \\ 
        &= \frac{1}{1 + \beta^{-1}}\sum_{j = 0}^{\infty}\left(-\frac{1}{1 + \beta^{-1}}\right)^{j}\vXi^{j} \\
        &= \frac{1}{1 + \beta^{-1}}\vI + \frac{1}{1 + \beta^{-1}}\sum_{j = 1}^{\infty}\left(-\frac{1}{1 + \beta^{-1}}\right)^{j}\vXi^{j}
    \end{align}
    by the Neumann series. This gives us
    \begin{align}
        \norm*{(\beta^{-1}\vI + \vA_{k}^{\top}\vA_{k})^{-1} - \frac{1}{1 + \beta^{-1}}\vI}
        &= \norm*{\frac{1}{1 + \beta^{-1}}\sum_{j = 1}^{\infty}\left(-\frac{1}{1 + \beta^{-1}}\right)^{j}\vXi^{j}} \\
        &\leq \frac{1}{1 + \beta^{-1}}\sum_{j = 1}^{\infty}\left(\frac{1}{1 + \beta^{-1}}\right)^{j}\norm{\vXi}^{j} \\ 
        &\leq \frac{1}{1 + \beta^{-1}}\sum_{j =
        1}^{\infty}\left(\frac{C_{2}\sqrt{N/d}}{1 + \beta^{-1}}\right)^{j} \\
        &= \frac{C_{2}\sqrt{N/d}}{(1 + \beta^{-1})(1 + \beta^{-1} - C_{2}\sqrt{N/d})}.
    \end{align}
    Meanwhile, by \Cref{ass:parameter_config}, it holds
    \begin{align}
        C_2 \sqrt{N/d} \leq \sqrt{1/6},
    \end{align}
    so it follows
    \begin{equation}
        \frac{C_{2}\sqrt{N/d}}{1 + \beta^{-1} - C_{2}\sqrt{N/d}}
        \leq
        2C_2 \sqrt{N/d}.
    \end{equation}
    By the submultiplicativity of the operator norm, we thus have
    \begin{align}
        \norm*{\vA_{k}(\beta^{-1}\vI + \vA_{k}^{\top}\vA_{k})^{-1} - \frac{1}{1 + \beta^{-1}}\vA_{k}}
        &\leq \norm{\vA_{k}}\norm*{(\beta^{-1}\vI + \vA_{k}^{\top}\vA_{k})^{-1} - \frac{1}{1 + \beta^{-1}}\vI} \\
        &\leq \frac{[1 + C_{1}\sqrt{N/d}]C_{2}\sqrt{N/d}}{(1 + \beta^{-1})(1 + \beta^{-1} - C_{2}\sqrt{N/d})} \\ 
        &\leq 2\frac{[1 + C_{1}\sqrt{N/d}]C_{2}\sqrt{N/d}}{1 + \beta^{-1}} \\
        &= 2\frac{C_{2}\sqrt{N/d} + C_{1}C_{2}N/d}{1 + \beta^{-1}}.
    \end{align}
    By \Cref{ass:parameter_config}, we have that there exists some absolute constant \(C_{3} > 0\) with \(C_{3} \cdot N/d \leq \sqrt{N/d}\), which gives 
    \begin{equation}
        \norm*{\vA_{k}(\beta^{-1}\vI + \vA_{k}^{\top}\vA_{k})^{-1} - \frac{1}{1
        + \beta^{-1}}\vA_{k}} \leq 2\frac{(C_{2} + C_{1}C_{2}C_{3}^{-1}) \sqrt{N/d}}{1 + \beta^{-1}},
    \end{equation}
    as desired.
\end{proof}

\begin{lemma}\label{lem:neumann-op-norm}
    There exist universal constants \(C_{1}, C_{2} > 0\) such that the following holds. Let \(d, p, N, K \in \N\) be such that \Cref{ass:parameter_config} holds. Let \(\vA_{k}\), \(k \in [K]\), be defined as above. Let \(\vD_{k}\) be defined as in \Cref{eq:ak-defn}. Let \(\vXi_{k}\) be defined as in \Cref{eq:xik-defn}. Let \(E^{\star}\) be the good event defined in \Cref{eq:good_event}. If \(E^{\star}\) occurs, then for \(k \in [K]\) we have
    \begin{equation}
        \norm{\vXi_{k}\vD_{k}}^{-1} \leq C_{1}\beta [\sigma^{2} + \sigma(C_{2} + \sqrt{N/d})].
    \end{equation}
\end{lemma}
\begin{proof}
    Since we have
    \begin{equation}
        \vD_{k} = \mat{\beta^{-1} \vI & & & & \\ & \ddots & & & \\ & & \beta^{-1}\vI + \vA_{k}^{\top}\vA_{k} & & \\ & & & \ddots & \\ & & & & \beta^{-1} \vI}
    \end{equation}
    it holds that 
    \begin{equation}
        \vD_{k}^{-1} = \mat{\beta \vI & & & & \\ & \ddots & & & \\ & & (\beta^{-1}\vI + \vA_{k}^{\top}\vA_{k})^{-1} & & \\ & & & \ddots & \\ & & & & \beta \vI}.
    \end{equation}
    We will use the straightforward estimate \(\norm{\vXi_{k}\vD_{k}^{-1}} \leq \norm{\vXi_{k}}\norm{\vD_{k}^{-1}}\) and bound the two matrices' operator norms individually. By the previous expression, 
    \begin{equation}
        \norm{\vD_{k}^{-1}} = \max\{\beta, \norm{(\beta^{-1}\vI + \vA_{k}^{\top}\vA_{k})^{-1}}\} \leq \beta,
    \end{equation}
    because \(\vA_{k}^{\top}\vA_{k} \succeq \vZero\), so we need only control the operator norm of \(\vXi_{k}\).
    To this end, note the convenient expression
    \begin{equation}
        \vXi_k
        =
        \vPi \vDelta\adj \vU_k \vU_k \adj \vDelta \vPi\adj
        +
        2 \Sym\left(
            (\vDelta \vPi\adj)\adj
            \begin{bmatrix}
                \vZero & \hdots & \vU_k \vA_k & \hdots & \vZero
            \end{bmatrix}
        \right),
    \end{equation}
    where $\Sym(\spcdot)$ denotes the symmetric part operator.
    By the triangle inequality, the operator norm of $\vXi_k$ is no larger
    than the sum of the operator norms of each term in the previous
    expression.
    The operator norm of the first term is no larger than
    $\norm{\vDelta}^2$, because $\vPi$ is a permutation matrix and $\vU
    \vU_k\adj$ is an orthogonal projection.
    Meanwhile, using that the symmetric part operator is the orthogonal
    projection onto the space of symmetric matrices, it follows
    \begin{equation}
        \scalebox{0.9}{\(\norm*{
            2 \Sym\left(
                (\vDelta \vPi\adj)\adj
                \begin{bmatrix}
                    \vZero & \hdots & \vU_k \vA_k & \hdots & \vZero
                \end{bmatrix}
            \right)
        }
        \leq
        2\norm*{
            (\vDelta \vPi\adj)\adj
            \begin{bmatrix}
                \vZero & \hdots & \vU_k \vA_k & \hdots & \vZero
            \end{bmatrix}
        }\)},
    \end{equation}
    and then we find as above that the RHS is no larger than $2
    \norm{\vDelta} \norm{\vA_k}$. Since the good event \(E^{\star}\) defined in \Cref{eq:good_event} holds by assumption, we have that there are constants \(C_{1}, C_{2} > 0\) such that 
    \begin{align}
        \norm{\vDelta} 
        &\leq \sigma \bp{C_{1} + \sqrt{\frac{N}{d}}} \\ 
        \norm{\vA_{k}}
        &\leq 1 + C_{2}\sqrt{\frac{N}{d}}.
    \end{align}
    By \Cref{ass:parameter_config} we have \(d \geq N\), so that \(\sqrt{N/d} \leq 1\). Therefore we have
    \begin{equation}
        \norm{\vDelta} \leq \sigma\bp{C_{1} + 1} = C_{3}\sigma
    \end{equation}
    for \(C_{3} \doteq C_{1} + 1\) another universal constant. Thus on this good event we have
    \begin{equation}
        2\norm{\vDelta}\norm{\vA_{k}} \leq C_{3}\sigma\bp{1 + C_{2}\sqrt{N/d}}.
    \end{equation}
    Therefore, we have
    \begin{align}
        \norm{\vXi_{k}}
        &\leq \norm{\vDelta}^{2} + 2\norm{\vDelta}\norm{\vA_{k}} \\
        &\leq C_{3}^{2}\sigma^{2} + C_{3}\sigma(1 + C_{2}\sqrt{N/d}) \\
        &\leq C_{4}[\sigma^{2} + \sigma(1 + C_{2}\sqrt{N/d})]
    \end{align}
    where \(C_{4} = \max\{C_{3}, C_{3}^{2}\}\) is another universal constant. Thus on \(E^{\star}\) we have
    \begin{equation}
        \norm{\vXi_{k}\vD_{k}^{-1}} \leq C_{4}\beta[\sigma^{2} + \sigma(1 + C_{2}\sqrt{N/d})] \leq C_{5}\beta [\sigma^{2} + \sigma(1 + \sqrt{N/d})]
    \end{equation}
    for \(C_{5} > 0\) another obvious universal constant.
\end{proof}

\subsubsection{Concentration Inequalities for Our Setting}

In this section we prove some simple concentration inequalities that are adapted to the problem setting. These results are used to prove the key lemmata above, and indeed are also invoked in the main theorem. They follow from even simpler concentration inequalities that are abstracted away from the problem setting, which we discuss in the following subsections.

\begin{proposition}\label{prop:delta_concentration_opnorm}
    There are universal constants \(C_{1}, C_{2}, C_{3} > 0\) such that the following holds. Let \(d, N \in \N\) be such that \Cref{ass:parameter_config} holds. Let \(\vDelta \in \R^{d \times N}\) be defined as above. Then
    \begin{equation}
        \Pr*{\norm*{\vDelta} > \sigma\bp{C_{1} + \sqrt{\frac{N}{d}}}} \leq C_{2}e^{-C_{3}d}.
    \end{equation}
\end{proposition}
\begin{proof}
    We use \Cref{prop:gaussian_concentration_opnorm} with the parameters \(q = d\), \(n = N\), and \(x = \sigma/\sqrt{d}\), which obtains
    \begin{equation}
        \Pr{\norm{\vDelta} > s} \leq C_{1}\exp\rp{-d\bc{\frac{s\sqrt{d}/\sigma - \sqrt{N}}{C_{2}\sqrt{d}} - 1}^{2}}, \qquad \forall s > \frac{\sigma}{\sqrt{d}}(\sqrt{N} + C_{2}\sqrt{d})
    \end{equation}
    Notice that we have
    \begin{equation}
        \frac{s\sqrt{d}/\sigma - \sqrt{N}}{C_{2}\sqrt{d}} - 1 = \frac{1}{C_{2}}\bp{\frac{s}{\sigma} - \sqrt{\frac{N}{d}}} - 1, \qquad \frac{\sigma}{\sqrt{d}}(\sqrt{N} + C_{2}\sqrt{d}) = \sigma \bp{\sqrt{\frac{N}{d}} + C_{2}}.
    \end{equation}
    To make the squared term equal to \(1\), we pick
    \begin{equation}
        s = \sigma \bp{\sqrt{\frac{N}{d}} + 2C_{2}},
    \end{equation}
    which gives
    \begin{equation}
        \Pr*{\norm{\vDelta} > \sigma\bp{2C_{2} + \sqrt{\frac{N}{d}}}} \leq C_{2}e^{-d}.
    \end{equation}
\end{proof}

\begin{proposition}\label{prop:xk_concentration_opnorm}
    There are universal constants \(C_{1}, C_{2}, C_{3}, C_{4} > 0\) such that the following holds. Let \(p, N, K \in \N\) be such that \Cref{ass:parameter_config} holds. Let \(\vA_{k}\), \(k \in [K]\), be defined as above. Then
    \begin{equation}
        \Pr*{\norm{\vA_{k}} > 1 + C_{1}\sqrt{\frac{N}{d}}} \leq C_{2}\exp\rp{-C_{3}\frac{N}{K}} + \frac{C_{4}}{N^{2}}
    \end{equation}
\end{proposition}
\begin{proof}
    By \Cref{prop:binom_concentration,prop:gaussian_concentration_opnorm} with parameters \(n = n\), \(k = K\), \(q = p\), and \(x = 1/\sqrt{p}\), if we define 
    \begin{equation}
        N_{\min} \doteq \floor*{\frac{N}{K} - C_{1}\sqrt{N\log N}}, \qquad
        N_{\max} \doteq \ceil*{\frac{N}{K} + C_{1}\sqrt{N \log N}}
    \end{equation}
    then we have
    \begin{align}
        \Pr{\norm{\vA_{k}} > s}
        &\leq \sum_{n = N_{\min}}^{N_{\max}}\Pr{\norm{\vA_{k}} > s \mid K_{k} = n}\Pr{K_{k} = n} + \frac{C_{2}}{N^{2}} \\
        &\leq \sum_{n = N_{\min}}^{N_{\max}}C_{3}\exp\rp{-n\bc{\frac{s\sqrt{p} - \sqrt{p}}{C_{4}\sqrt{4}} - 1}^{2}}\Pr{K_{k} = n} + \frac{C_{2}}{N^{2}},
    \end{align}
    for all \(s\) obeying
    \begin{align}
        s 
        &\geq \frac{1}{\sqrt{p}}\bp{\sqrt{p} + C_{4}\sqrt{N_{\max}}} \\ 
        &= 1 + C_{4}\sqrt{\frac{N_{\max}}{p}}.
    \end{align}
    Thus we have that the concentration holds for all \(s\) obeying
    \begin{equation}
        s \geq 1 + C_{4}\sqrt{\frac{N_{\max}}{p}}.
    \end{equation}
    In order to cancel out the most interior terms, we choose
    \begin{equation}
        s = 1 + 2C_{4}\sqrt{\frac{N_{\max}}{p}}.
    \end{equation}
    This choice obtains
    \begin{align}
        \Pr{\norm{\vA_{k}} > s}
        &\leq \sum_{n = N_{\min}}^{N_{\max}}C_{3}\exp\rp{-n\bc{\frac{s\sqrt{p} - \sqrt{p}}{C_{4}\sqrt{n}} - 1}^{2}}\Pr{K_{k} = n} + \frac{C_{2}}{N^{2}} \\
        &= \sum_{n = N_{\min}}^{N_{\max}}C_{3}\exp\rp{-n\underbrace{\bc{2\underbrace{\sqrt{\frac{N_{\max}}{n}}}_{\geq 1} - 1}^{2}}_{\geq 1}}\Pr{K_{k} = n} + \frac{C_{2}}{N^{2}} \\
        &\leq \sum_{n = N_{\min}}^{N_{\max}}C_{3}\exp\rp{-n}\Pr{K_{k} = n} + \frac{C_{2}}{N^{2}} \\
        &\leq \sum_{n = N_{\min}}^{N_{\max}}C_{3}\exp\rp{-N_{\min}}\Pr{K_{k} = n} + \frac{C_{2}}{N^{2}} \\
        &\leq C_{3}\exp(-N_{\min}) + \frac{C_{2}}{N^{2}} \\ 
        &= C_{3}\exp\rp{-\frac{N}{K} + C_{1}\sqrt{N \log N}} + \frac{C_{2}}{N^{2}} \\
        &\leq C_{3}\exp\rp{-\frac{N}{K} + \frac{1}{2}\sqrt{N \log N}} + \frac{C_{2}}{N^{2}} \\ 
        &\leq C_{3}\exp\rp{-\frac{1}{2}\cdot \frac{N}{K}} + \frac{C_{2}}{N^{2}}.
    \end{align}
    To obtain the conclusion of the theorem, note that any \(s\) such that
    \begin{equation}
        s \geq 1 + 2C_{4}\sqrt{\frac{N_{\max}}{p}} = 1 + 2C_{4}\sqrt{\frac{N/K}{p} + C_{1}\frac{\sqrt{N \log N}}{p}} = 1 + 2C_{4}\sqrt{\frac{N}{d} + C_{1}\frac{\sqrt{N \log N}}{p}}
    \end{equation}
    enjoys the same high-probability bound. By \Cref{ass:parameter_config}, we have
    \begin{align}
        &\;1 + 2C_{4}\sqrt{\frac{N}{d} + C_{1}\frac{\sqrt{N \log N}}{p}} \leq 1 +
          2C_{4}\sqrt{\frac{N}{d} + \frac{1}{2} \cdot \frac{N/K}{p}}\\
          = &\;1 +
          2C_{4}\sqrt{\frac{N}{d} + \frac{1}{2} \cdot \frac{N}{d}} = 1 + 2C_{4}\sqrt{\frac{3}{2}}\cdot \sqrt{\frac{N}{d}}
    \end{align}
    whence the ultimate conclusion is obtained by combining constants.
\end{proof}

\begin{proposition}\label{prop:xktxk_concentration_opnorm}
    There are universal constants \(C_{1}, C_{2}, C_{3}, C_{4} > 0\) such that the following holds. Let \(p, N, K \in \N\) be such that \Cref{ass:parameter_config} holds. Let \(\vA_{k}\), \(k \in [K]\), be defined as above. Then
    \begin{equation}
        \Pr*{\norm{\vA_{k}^{\top}\vA_{k} - \vI} > C_{1}\sqrt{\frac{N}{d}}} \leq C_{2}\exp\rp{-C_{3}\frac{N}{K}} + \frac{C_{4}}{N^{2}}.
    \end{equation}
\end{proposition}
\begin{proof}
    By \Cref{prop:binom_concentration,prop:gaussian_covariance_concentration_opnorm} with parameters \(n = n\), \(k = K\), \(q = p\), and \(x = 1/\sqrt{p}\), if we define 
    \begin{equation}
        N_{\min} \doteq \floor*{\frac{N}{K} - C_{1}\sqrt{N\log N}}, \qquad 
        N_{\max} \doteq \ceil*{\frac{N}{K} + C_{1}\sqrt{N \log N}}
    \end{equation}
    then we have
    \begin{align}
        &\Pr{\norm{\vA_{k}^{\top}\vA_{k} - \vI} > s} \\
        &\leq \sum_{n = N_{\min}}^{N_{\max}}\Pr{\norm{\vA_{k}^{\top}\vA_{k} - \vI} > s \mid K_{k} = n}\Pr{K_{k} = n} + \frac{C_{2}}{N^{2}} \\
        &\leq \frac{C_{2}}{N^{2}} + \sum_{n = N_{\min}}^{N_{\max}}\Pr{K_{k} = n}\cdot \scalebox{0.8}{\(\casework{C_{3}\exp\rp{-n\bc{\frac{1}{C_{4}^{2}C_{5}\sqrt{n/p}}s - 1}^{2}}, & \text{if}\ C_{4}^{2}C_{5}\sqrt{n/p} \leq s \leq C_{4}^{2} \\ C_{3}\exp\rp{-n\bc{\frac{1}{C_{4}C_{5}\sqrt{n/p}}\sqrt{s} - 1}^{2}}, & \text{if}\ s \geq C_{4}^{2}.}\)}
    \end{align}
    In order to cancel the most terms, we choose
    \begin{equation}
        s = 2C_{4}^{2}C_{5}\sqrt{\frac{N_{\max}}{p}}.
    \end{equation}
    In order to assure ourselves that this choice still has \(s \leq C_{4}^{2}\) (so that we can use the first case for all \(n\)), we have
    \begin{align}
        s 
        &= 2C_{4}^{2}C_{5}\sqrt{\frac{N_{\max}}{p}} \\
        &= 2C_{4}^{2}C_{5}\sqrt{\frac{N/K + C_{1}\sqrt{N \log N}}{p}} \\ 
        &= 2C_{4}^{2}C_{5}\sqrt{\frac{N/K + \frac{1}{2}N/K}{p}} \\ 
        &= 2\sqrt{\frac{3}{2}}C_{4}^{2}C_{5}\cdot \sqrt{\frac{N/K}{p}} \\
        &= \sqrt{6}C_{4}^{2}C_{5}\cdot \sqrt{\frac{N}{d}} \\ 
        &\leq C_{4}^{2} \quad \text{when} \quad \sqrt{6}C_{5}\sqrt{\frac{N}{d}} \leq 1.
    \end{align}
    Of course, this condition is assured by \Cref{ass:parameter_config}. Now that we have this, we know \(s\) falls in the first, and so we have
    \begin{align}
        &\Pr*{\norm{\vA_{k}^{\top}\vA_{k} - \vI} > C_{1}\sqrt{\frac{N}{d}}} \\
        &\leq \frac{C_{2}}{N^{2}} + \sum_{n = N_{\min}}^{N_{\max}}\Pr{K_{k} = n}\cdot \scalebox{0.8}{\(\casework{C_{3}\exp\rp{-n\bc{\frac{1}{C_{4}^{2}C_{5}\sqrt{n/p}}s - 1}^{2}}, & \text{if}\ C_{4}^{2}C_{5}\sqrt{n/p} \leq s \leq C_{4}^{2} \\ C_{3}\exp\rp{-n\bc{\frac{1}{C_{4}C_{5}\sqrt{n/p}}\sqrt{s} - 1}^{2}}, & \text{if}\ s \geq C_{4}^{2}}\)} \\
        &\leq \frac{C_{2}}{N^{2}} + \sum_{n = N_{\min}}^{N_{\max}}C_{3}\exp\rp{-n\bc{\frac{2C_{4}^{2}C_{5}\sqrt{N_{\max}/p}}{C_{4}^{2}C_{5}\sqrt{n/p}} - 1}^{2}}\Pr{K_{k} = n} \\ 
        &= \frac{C_{2}}{N^{2}} + \sum_{n = N_{\min}}^{N_{\max}}C_{3}\exp\rp{-n\bc{2\sqrt{\frac{N_{\max}}{n}} - 1}^{2}}\Pr{K_{k} = n} \\ 
        &\leq C_{3}\exp\rp{-\frac{1}{2} \cdot \frac{N}{K}} + \frac{C_{2}}{N^{2}}
    \end{align}
    where the last inequality follows from the exact same argument as in \Cref{prop:xk_concentration_opnorm}.
\end{proof}

\subsubsection{Generic Concentration Inequalities}

In this subsection we prove the base-level concentration inequalities used throughout the proofs in this paper.

\paragraph{Binomial concentration.}

\begin{proposition}\label{prop:binom_concentration}
    There exist universal constants \(C_{1}, C_{2} > 0\) such that the following holds. Let \(n, k \in \Z\). For each \(i \in [k]\), let \(B_{i} \sim \dBin(n, 1/k)\), such that the \(B_{i}\) are identically (marginally) distributed but not necessarily independent binomial random variables. Let \(E\) be an event. Then for any \(i \in [k]\), we have
    \begin{equation}
        \Pr{E} \leq \sum_{b = \floor*{n/k - C_{1}\sqrt{n \log n}}}^{\ceil*{n/k + C_{1}\sqrt{n \log n}}}\Pr{E \given B_{i} = b}\Pr{B_{i} = b} + \frac{C_{2}}{n^{2}}.
    \end{equation}
\end{proposition}
\begin{proof}
    We have
    \begin{equation}
        \Pr{E} = \E{\E{E \mid B_{i}}} = \sum_{b = 0}^{n}\Pr{E \mid B_{i} = b}\Pr{B_{i} = b}.
    \end{equation}
    Each \(B_{i}\) is unconditionally distributed as \(\dBin(n, 1/k)\). By union bound and Hoeffding's inequality \citep[Theorem 2.2.6]{vershynin2018high}, we have
    \begin{equation}
        \Pr*{\abs{B_{i} - n/k} \geq t} \leq 2\exp\rp{-\frac{2t^{2}}{n}}.
    \end{equation}
    Inverting this inequality obtains that there exists some (simple) universal constants \(C_{1}, C_{2} > 0\) such that 
    \begin{equation}
        \Pr*{\abs{B_{i} - n/k} \geq C_{1}\sqrt{n\log n}} \leq \frac{C_{2}}{n^{3}}.
    \end{equation}
    Thus, if we define
    \begin{equation}
        b_{\min} \doteq \floor*{\frac{n}{k} - C_{1}\sqrt{n \log n}}, \qquad b_{\max} \doteq \ceil*{\frac{n}{k} + C_{1}\sqrt{n \log n}},
    \end{equation}
    then we have
    \begin{align}
        \Pr{E}
        &= \sum_{b = 0}^{n}\Pr{E \mid B_{i} = b}\Pr{B_{i} = b} \\ 
        &= \sum_{b = 0}^{b_{\min} - 1}\underbrace{\Pr{E \mid B_{i} = b}}_{\leq 1}\underbrace{\Pr{B_{i} = b}}_{\leq C_{2}/n^{3}} + \sum_{b = b_{\max} + 1}^{n}\underbrace{\Pr{E \mid B_{i} = b_{1}}}_{\leq 1}\underbrace{\Pr{B_{i} = b}}_{\leq C_{2}/n^{3}} \\ 
        &\qquad + \sum_{b = b_{\min}}^{b_{\max}}\Pr{E \mid B_{i} = b}\Pr{B_{i} = b} \\ 
        &\leq \sum_{b = 0}^{b_{\min} - 1}\frac{C_{2}}{n^{3}} + \sum_{b = b_{\max} + 1}^{n}\frac{C_{2}}{n^{3}} + \sum_{b = b_{\min}}^{b_{\max}}\Pr{E \mid B_{i} = b}\Pr{B_{i} = b} \\
        &\leq \sum_{b = 0}^{n}\frac{C_{2}}{n^{3}} + \sum_{b = b_{\min}}^{b_{\max}}\Pr{E \mid B_{i} = b}\Pr{B_{i} = b} \\
        &= \frac{C_{2}}{n^{2}} + \frac{C_{2}}{n^{3}} + \sum_{b = b_{\min}}^{b_{\max}}\Pr{E \mid B_{i} = b}\Pr{B_{i} = b} \\
        &\leq \frac{2C_{2}}{n^{2}} + \sum_{b = b_{\min}}^{b_{\max}}\Pr{E \mid B_{i} = b}\Pr{B_{i} = b}.
    \end{align}
\end{proof}
\begin{remark}
    Notice that a simple adaptation of this argument can turn the additive probability \(C_{3}/n^{2}\) into \(C_{3}^{\prime}/n^{z}\) for any positive integer \(z \in \N\) (where \(C_{3}^{\prime}\) depends on \(z\)). However, trying to replace it with \(C_{3}^{\prime}e^{-C^{\prime}n}\) is more difficult.
\end{remark}

\paragraph{Gaussian concentration.}

\begin{proposition}\label{prop:gaussian_concentration_opnorm}
    There are universal constants \(C_{1}, C_{2}, C_{3} > 0\) such that the following holds. Let \(n, q \in \N\), and let \(\vM \in \R^{q \times n}\) be such that \(M_{ij} \simiid \dNorm(0, x^{2})\). Then 
        \begin{align}
            \Pr{\norm{\vM} > s}  
            &\leq C_{1}\exp\rp{-n\bc{\frac{s/x - \sqrt{q}}{C_{2}\sqrt{n}} - 1}^{2}}, \qquad \forall s > x\bc{\sqrt{q} + C_{2}\sqrt{n}} \\
            \Pr{\norm{\vM} > s}  
            &\leq C_{1}\exp\rp{-q\bc{\frac{s/x - \sqrt{n}}{C_{3}\sqrt{q}} - 1}^{2}}, \qquad \forall s > x\bc{\sqrt{n} + C_{3}\sqrt{q}}.
        \end{align}
\end{proposition}
\begin{proof}
    Define \(\ol{\vM} \doteq \frac{1}{x}\vM\), so that \(M_{ij} \simiid \dNorm(0, 1)\). By \citet[Example 2.5.8, Lemma 3.4.2]{vershynin2018high}, we see that each row \(\ol{\vM}_{i}\) has Orlicz norm \(\norm{\ol{\vM}_{i}}_{\psi_{2}} \leq C_{1}\) for some universal constant \(C_{1} > 0\). 

    By \citet[Theorem 4.6.1]{vershynin2018high} we have for some other universal constant \(C_{2} > 0\) that for all \(t > 0\),
    \begin{equation}
        \sqrt{q} - C_{1}^{2}C_{2}(\sqrt{n} + t) \leq \sigma_{\min(n, q)}(\ol{\vM}) \leq \sigma_{1}(\ol{\vM}) \leq \sqrt{q} + C_{1}^{2}C_{2}(\sqrt{n} + t)
    \end{equation}
    with probability at least \(1 - 2e^{-t^{2}}\). Defining \(C_{3} \doteq C_{1}^{2}C_{2}\) and noting that \(\norm{\cdot} = \sigma_{1}(\cdot)\), we have with the same probability that
    \begin{equation}
        \norm{\ol{\vM}} - \sqrt{q} \leq C_{3}\bp{\sqrt{n} + t}.
    \end{equation}
    Simplifying, we obtain
    \begin{align}
        \norm{\ol{\vM}} - \sqrt{q}
        &\leq C_{3}\bp{\sqrt{n} + t} \\
        \frac{1}{x}\norm{\vM} - \sqrt{q} 
        &\leq C_{3}\bp{\sqrt{n} + t} \\
        \norm{\vM} - x\sqrt{q} 
        &\leq C_{3}x\bp{\sqrt{n} + t} \\
        \norm{\vM} 
        &\leq x \bc{\sqrt{q} + C_{3}\bp{\sqrt{n} + t}}.
    \end{align}
    Define \(s > 0\) by 
    \begin{equation}
        s \doteq x\bc{\sqrt{q} + C_{3}\bp{\sqrt{n} + t}} \iff t = \frac{s/x - \sqrt{q}}{C_{3}} - \sqrt{n}.
    \end{equation}
    Note that the range of validity is
    \begin{equation}
        t > 0 \iff s > x\bc{\sqrt{q} + C_{3}\sqrt{n}}.
    \end{equation}
    For \(s\) in this range, we have
    \begin{equation}
        \Pr{\norm{\vM} > s} \leq 2\exp\rp{-\bc{\frac{s/x - \sqrt{q}}{C_{3}} - \sqrt{n}}^{2}} = 2\exp\rp{-n\bc{\frac{s/x - \sqrt{q}}{C_{3}\sqrt{n}} - 1}^{2}}.
    \end{equation}
    The other inequality follows from applying this inequality to \(\vM^{\top}\).
\end{proof}

\begin{proposition}\label{prop:gaussian_covariance_concentration_opnorm}
    There are universal constants \(C_{1}, C_{2}, C_{3} > 0\) such that the following holds. Let \(n, q \in \N\), and let \(\vM \in \R^{q \times n}\) be such that \(M_{ij} \simiid \dNorm(0, x^{2})\). Then 
    \begin{align}
        &\Pr*{\norm*{\vM^{\top}\vM - qx^{2}\vI} > s} \\
        &\leq \casework{C_{1}\exp\rp{-n\bc{\frac{1}{C_{2}^{2}C_{3}\sqrt{nq}x^{2}}s - 1}^{2}}, & \text{if}\ C_{2}^{2}C_{3}\sqrt{nq}x^{2} \leq s \leq C_{2}^{2}qx^{2} \\ C_{1}\exp\rp{-n\bc{\frac{1}{C_{2}C_{3}\sqrt{n}x}\sqrt{s} - 1}^{2}}, & \text{if}\ s \geq C_{2}^{2}qx^{2}.}
    \end{align}
\end{proposition}
\begin{proof}
    Define \(\ol{\vM} \doteq \frac{1}{x}\vM\), so that \(\ol{M}_{ij} \simiid \dNorm(0, 1)\). By \citet[Example 2.5.8, Lemma 3.4.2]{vershynin2018high}, we see that each row has Orlicz norm \(\norm{\ol{\vM}_{i}}_{\psi_{2}} \leq C_{1}\) for some universal constant \(C_{1} > 0\). 

    By \citet[Eq.~4.22]{vershynin2018high} we have for some other universal constant \(C_{2} > 0\) that for all \(t > 0\),
    \begin{equation}
        \norm*{\frac{1}{q}\ol{\vM}^{\top}\ol{\vM} - \vI} \leq C_{1}^{2}\max\{\delta, \delta^{2}\} \quad \text{where} \quad \delta \doteq C_{2}\frac{\sqrt{n} + t}{\sqrt{q}}.
    \end{equation}
    with probability at least \(1 - 2e^{-t^{2}}\). Simplifying, we obtain
    \begin{align}
        \norm*{\frac{1}{q}\ol{\vM}^{\top}\ol{\vM} - \vI} 
        &\leq C_{1}^{2}\max\{\delta, \delta^{2}\} \\
        \norm*{\ol{\vM}^{\top}\ol{\vM} - q\vI} 
        &\leq C_{1}^{2}q\max\{\delta, \delta^{2}\} \\
        \norm*{(x^{-1}\vM)^{\top}(x^{-1}\vM) - q\vI} 
        &\leq C_{1}^{2}q\max\{\delta, \delta^{2}\} \\
        \norm*{x^{-2}\vM^{\top}\vM - q\vI} 
        &\leq C_{1}^{2}q\max\{\delta, \delta^{2}\} \\
        \norm*{\vM^{\top}\vM - qx^{2}\vI} 
        &\leq C_{1}^{2}qx^{2}\cdot\max\{\delta, \delta^{2}\}.
    \end{align}
    Now from simple algebra and the fact that \(n \geq 1\), we have
    \begin{align}
        \max\bc{\delta, \delta^{2}} = \delta 
        &\iff 0 \leq t \leq C_{2}^{-1}\sqrt{q} - \sqrt{n} \\
        \max\bc{\delta, \delta^{2}} = \delta^{2}
        &\iff t \geq C_{2}^{-1}\sqrt{q} - \sqrt{n}.
    \end{align}
    Now define \(s \geq 0\) by
    \begin{equation}
        s \doteq C_{1}^{2}qx^{2} \cdot \max\{\delta, \delta^{2}\}.
    \end{equation}
    Thus in the first case we have
    \begin{equation}
        s \doteq C_{1}^{2}C_{2}\sqrt{q}x^{2}(\sqrt{n} + t) \iff t = \frac{1}{C_{1}^{2}C_{2}\sqrt{q}x^{2}}s - \sqrt{n},
    \end{equation}
    and in particular the first case holds when
    \begin{equation}
        \max\{\delta, \delta^{2}\} = \delta  \iff 0 \leq t \leq C_{2}^{-1}\sqrt{q} - \sqrt{n} \iff C_{1}^{2}C_{2}\sqrt{nq}x^{2} \leq s \leq C_{1}^{2}qx^{2}.
    \end{equation}
    Meanwhile, in the second case, we have
    \begin{equation}
        s \doteq C_{1}^{2}C_{2}^{2}(\sqrt{n} + t)^{2}x^{2} \iff t = \frac{1}{C_{1}C_{2}x}\sqrt{s} - \sqrt{n}
    \end{equation}
    where we obtain only one solution to the quadratic equation by requiring \(t \geq 0\), and in particular the second case holds when 
    \begin{equation}
        \max\{\delta, \delta^{2}\} = \delta \iff t \geq C_{2}^{-1}\sqrt{q} - \sqrt{n} \iff s \geq C_{1}^{2}qx^{2}.
    \end{equation}
    Thus we have
    \begin{align}
        &\Pr{\norm{\vM^{\top}\vM - qx^{2}\vI} > s} \\
        &\leq \casework{2\exp\rp{-\bc{\frac{1}{C_{1}^{2}C_{2}\sqrt{q}x^{2}}s - \sqrt{n}}^{2}}, & \text{if}\ C_{1}^{2}C_{2}\sqrt{nq}x^{2} \leq s \leq C_{1}^{2}qx^{2} \\ 2\exp\rp{-\bc{\frac{1}{C_{1}C_{2}x}\sqrt{s} - \sqrt{n}}^{2}}, & \text{if}\ s \geq C_{1}^{2}qx^{2}} \\
        &= \casework{2\exp\rp{-n\bc{\frac{1}{C_{1}^{2}C_{2}\sqrt{nq}x^{2}}s - 1}^{2}}, & \text{if}\ C_{1}^{2}C_{2}\sqrt{nq}x^{2} \leq s \leq C_{1}^{2}qx^{2} \\ 2\exp\rp{-n\bc{\frac{1}{C_{1}C_{2}\sqrt{n}x}\sqrt{s} - 1}^{2}}, & \text{if}\ s \geq C_{1}^{2}qx^{2}.}
    \end{align}
\end{proof}

\subsection{Companion to \Cref{sub:crate_mae_layer}}
\label{app:discretization}

In this section, we justify the scaling applied to \(\nabla R^{c}\) in \Cref{sub:crate_mae_layer} and supply the discretization scheme.

First, suppose that \(\vZ_{\natural}^{\ell}\) satisfies \Cref{model:gaussian_tokens}, and \(\vZ_{t} \doteq \vZ_{\natural}^{\ell} + \sigma_{t}\vW\), where \(\vW\) is a standard Gaussian matrix, so that \(\vZ_{t}\) satisfies \Cref{model:gaussian_tokens_noise} with noise level \(\sigma_{t} > 0\). Let \(q_{t}\) be the density of \(\vZ_{t}\). Theoretical analysis from \citep{lu2023mathematical} and empirical analysis from \citep{Song2019-ww} demonstrates that under generic conditions, we have that 
\begin{equation}
    \norm{\nabla q_{t}(\vZ_{t})}_{2} \propto \frac{1}{\sigma_{t}^{2}},
\end{equation}
ignoring all terms in the right-hand side except for those involving \(\sigma_{t}\). On the other hand, from the proof of \Cref{lem:inverse-term}, we obtain that \(-\nabla R^{c}(\vZ_{t})\) has constant (in \(\sigma_{t}\)) magnitude with high probability. Thus, in order to have them be the same magnitude, we need to divide \(-\nabla R^{c}(\vZ_{t})\) by \(\sigma_{t}^{2}\) to have it be a drop-in replacement for the score function, as alluded to in \Cref{sub:unification,sub:crate_mae_layer}.

Second, we wish to explicitly state our discretization scheme given in \Cref{sub:crate_mae_layer}. To wit, we provide a discretization scheme that turns the structured diffusion ODE \Cref{eq:compression_ode_reverse} into its discretized analogue; the other discretization of the structured denoising ODE \Cref{eq:compression_ode} occurs similarly. To begin with, define the shorthand notation 
\begin{equation}
    f(t, \vY(t)) \doteq \nabla R^{c}(\vY(t) \mid \vU_{[K]}(T - t)),
\end{equation}
so that we have
\begin{equation*}
    \odif{\vY(t)} = \frac{1}{2t}f(t, \vY(t))\odif{t}. \tag{\ref{eq:compression_ode}}
\end{equation*}
Fix \(L\), and let \(0 < t_{1} < t_{2} < \cdots < t_{L} = T\), such that \(t_{1}\) is small. (These will be specified shortly in order to supply the discretization scheme.) A suitable first-order discretization is given by
\begin{equation}
    \vY^{\ell + 1} \approx \vY^{\ell} + \frac{t_{\ell + 1} - t_{\ell}}{2t_{\ell}}f(t_{\ell}, \vY^{\ell}).
\end{equation}
Thus it remains to set \(t_{1}, \dots, t_{L}\) such that
\begin{equation}
    \frac{t_{\ell + 1} - t_{\ell}}{2t_{\ell}} = \kappa
\end{equation}
for some constant \(\kappa\), we observe that we must set 
\begin{equation}
    t_{\ell + 1} = (1 + 2\kappa)t_{\ell},
\end{equation}
so that the time grows exponentially in the index. The reverse process time decays exponentially in the index, which matches practical discretization schemes for ordinary diffusion models \citep{Song2019-ww}. Finally, we have \(T = t_{L} = (1 + 2\kappa)^{L}t_{1}\), so that \(t_{1} = \frac{T}{(1 + 2\kappa)^{L}}\).

\newpage
\section{Additional Experiment Details}\label{app:experiments}

\subsection{Experiment Details and Clarifications}\label{app:experiment_clarifications}

In all setups, as is standard practice \citep{he2022masked}, we append a trainable class token \(\vz_{\cls}^{1} \in \R^{d}\) after masking and linear projection, namely, 
\begin{equation}
    f^{\pre}(\vX) = [\vz_{\cls}^{1}, \vW^{\pre}\vX + \vE^{\pos}].
\end{equation}
Everything else goes through with \(Z\) having \(N + 1\) instead of \(N\) tokens, indexing the \(N\) image tokens' intermediate representations as \(\vz_{i}^{\ell}\), etc., and indexing the class token intermediate representation as \(\vz_{\cls}^{\ell}\). At the end, the post-processing map is 
\begin{equation}
    g^{\post}(\vY^{L}) = \vW^{\post}\vY_{1:N}^{L}
\end{equation}
where \(\vY_{1:N}^{L} \in \R^{d \times N}\) is the matrix whose columns are the columns of \(\vY^{L}\) corresponding to image tokens, i.e., the second through last columns of \(\vY^{L}\). Thus, the class token \(\vz_{\cls}^{1}\) (and its representation \(\vz_{\cls}\), i.e., the first column of \(\vZ\)) are neither masked or reconstructed.

For training using masked autoencoding, we follow a modified recipe of \citep{he2022masked}. We mask a fixed percentage \(\mu \in [0, 1]\) of randomly selected tokens in \(\vX\); that is, we randomly set \((1 - \mu)N\) image tokens \(\vx_{i}\) to \(\vZero \in \R^{D}\), obtaining \(\overline{\vX}\). Then \(\overline{\vX}\) is the input to the encoder. Unlike in \citep{he2022masked}, the decoder receives no special treatment, and operates on all token representations. The loss is computed only on the masked image patches. Refer to \citet{he2022masked} for more MAE implementation details.

For fine-tuning using supervised classification, we use the representation \(\vz_{\cls}\) of the so-far-vestigial class token as the feature used for classification. Namely, we obtain the unnormalized log-probabilities for the classes as \(\vu \doteq \vW^{\head}\operatorname{LN}(\vz_{\cls})\), where \(\operatorname{LN}\) is a trainable layer-norm and \(\vW^{\head} \in \R^{C \times d}\) is a trainable weight matrix, where \(C\) is the number of classes. The output \(\vu \in \R^{C}\) is the input to the softmax cross-entropy loss. All model parameters are trainable during fine-tuning, while in linear probing only the weight matrix \(\vW^{\head}\) is trainable (and in fact learned via full-batch logistic regression). 

In all training setups, we average the loss over all samples in the batch.

We pre-train \ours{} on ImageNet-1K \citep{deng2009imagenet}. We employ the AdamW optimizer \citep{loshchilov2017decoupled}. We configure the learning rate as \(3\times 10^{-5}\), weight decay as \(0.1\), and batch size as \(4,096\). 

We fine-tune and linear probe our pre-trained \ours{} on the following target datasets: CIFAR10/CIFAR100~\citep{krizhevsky2009learning}, Oxford Flowers-102~\citep{nilsback2008automated}, Oxford-IIIT-Pets~\citep{parkhi2012cats}. 
For each fine-tuning task, we employ the AdamW optimizer~\citep{loshchilov2017decoupled}. 
We configure the learning rate as \(5 \times 10^{-5}\), weight decay as \(0.01\), and batch size as \(256\). 
For each linear probing task, we use the linear probing functionality in Scikit-Learn~\citep{scikit-learn}. For each evaluation we choose several regularizers \(C \in \{10^{0}, 10^{1}, 10^{2}, 10^{3}, 10^{4}, 10^{5}\}\), train a logistic regression model on features from the whole dataset, and choose the logistic regression model with the best performance. All numbers are reported on the test sets.

For experiments, we use the model configurations reported in \Cref{tab:model_configs}. From this table, there are two unspecified hyperparameters, namely \(\lambda\) and \(\eta\). In all experiments we fix \(\eta = 0.1\), sincce it only multiplies with a trainable matrix and \(\lambda\). In numerical experiments we use \(\lambda = 0.5\), while figures are generated from models with hyperparameter \(\lambda = 5.0\), though the difference in numerical performance and figure quality is marginal between the two settings (\Cref{tab:performance_for_different_lambda}).

To allow transfer learning, in all training and evaluations setups we first resize our input data to \(224\) height and width. For data augmentations during pre-training and fine-tuning, we also adopt several standard techniques: random cropping, random horizontal flipping, and random augmentation (with number of transformations \(n=2\) and magnitude of transformations \(m=14\)).

\subsection{PyTorch-Like Pseudocode}\label{app:experiment_pseudocode}

\begin{lstlisting}[language=Python, caption=PyTorch-Like Code for MSSA and ISTA]
class ISTA:
    # initialization
    def __init__(self, dim, hidden_dim, dropout = 0., step_size=0.1, lambd=0.1):
        super().__init__()
        self.weight = nn.Parameter(torch.Tensor(dim, dim))
        with torch.no_grad():
            init.kaiming_uniform_(self.weight)
        self.step_size = step_size
        self.lambd = lambd
    # forward pass
    def forward(self, x):
        x1 = F.linear(x, self.weight, bias=None)
        grad_1 = F.linear(x1, self.weight.t(), bias=None)
        grad_2 = F.linear(x, self.weight.t(), bias=None)
        grad_update = self.step_size * (grad_2 - grad_1) - self.step_size * self.lambd
        output = F.relu(x + grad_update)
        return output

class MSSA:
    # initialization
    def __init__(self, dim, heads = 8, dim_head = 64, dropout = 0.):
        inner_dim = dim_head * heads
        project_out = not (heads == 1 and dim_head == dim)
        self.heads = heads
        self.scale = dim_head ** -0.5
        self.attend = Softmax(dim = -1)
        self.dropout = Dropout(dropout)
        self.qkv = Linear(dim, inner_dim, bias=False)
        self.to_out = Sequential(Linear(inner_dim, dim), Dropout(dropout)) if project_out else nn.Identity()
    # forward pass
    def forward(self, x):
        w = rearrange(self.qkv(x), 'b n (h d) -> b h n d', h = self.heads)
        dots = matmul(w, w.transpose(-1, -2)) * self.scale
        attn = self.attend(dots)
        attn = self.dropout(attn)
        out = matmul(attn, w)
        out = rearrange(out, 'b h n d -> b n (h d)')
        return self.to_out(out)
\end{lstlisting}

\begin{lstlisting}[language=Python, caption=PyTorch-Like Code for \ours{} Encoder]
class CRATE_Encoder:
    # initialization
    def __init__(self, dim, depth, heads, dim_head, mlp_dim, dropout = 0.):
        self.layers = []
        self.depth = depth
        for _ in range(depth):
            self.layers.extend([LayerNorm(dim), MSSA(dim, heads, dim_head, dropout)])
            self.layers.extend([LayerNorm(dim), ISTA(dim, mlp_dim, dropout)])
    # forward pass
    def forward(self, x):
        for ln1, attn, ln2, ff in self.layers:
            x_ = attn(ln1(x)) + x
            x = ff(ln2(x_))
        return x
\end{lstlisting}

\begin{lstlisting}[language=Python, caption=PyTorch-Like Code for \ours{} Decoder]
class CRATE_Decoder:
    # initialization
    def __init__(self, dim, depth, heads, dim_head, mlp_dim, dropout = 0.):
        # define layers
        self.layers = []
        self.depth = depth
        for _ in range(depth):
            self.layers.extend([LayerNorm(dim), Linear(in_features=dim, out_features=dim, bias=False)])
            self.layers.extend([LayerNorm(dim), MSSA(dim, heads, dim_head, dropout)])
    # forward pass
    def forward(self, x):
        for ln1, f_linear, ln2, attn in self.layers:
            x_ = f_linear(ln1(x))
            x = ln2(x_) - attn(ln2(x_))
        return x
\end{lstlisting}

\subsection{Visualization Methodology} \label{app:viz_methodology}

In this subsection we formally describe the procedures we used to generate the visualizations used to evaluate the segmentation property of \ours{} in \Cref{fig:pca_attn}. Much of this evaluation is the same as in \citet{yu2023emergence}, which initially demonstrates the emergent segmentation properties of white-box architectures.

\subsubsection{PCA Visualizations}\label{app:pca}

We recapitulate the method to visualize the patch representations in \Cref{fig:pca} using PCA from \citet{amir2021deep, oquab2023dinov2,yu2023emergence}.

We first select $J$ images that belong to the same class, $\{\vX_j\}_{j=1}^{J}$, and extract the token representations for each image at layer \(\ell\), i.e., $\mat{{\vz}_{j, 1}^{\ell}, \dots, \vz_{j, N}^{\ell}}$ for $j \in [J]$. In particular, $\vz_{j, i}^{\ell}$ represents the $i^{\mathrm{th}}$ token representation at the $\ell^{\mathrm{th}}$ layer for the $j^{\mathrm{th}}$ image. We then compute the first principal components of $\widehat{\vZ}^{\ell} = \{\widehat{\vz}_{1, 1}^{\ell}, \dots, \widehat{\vz}_{1, N}^{\ell}, \dots, \widehat{\vz}_{J, 1}^{\ell}, \dots, \widehat{\vz}_{J, N}^{\ell} \}$, and use $\widehat{\vz}_{j, i}^{\ell}$ to denote the aggregated token representation for the $i$-th token of $\vX_j$, i.e., $\widehat{\vz}_{j, i}^{\ell}=[(\vU_1^{*}\widehat{\vz}_{j, i}^{\ell})^{\top}, \dots, (\vU_K^{*}\widehat{\vz}_{j, i}^{\ell})^{\top}]^{\top} \in \R^{(p\cdot K)\times 1}$. 
We denote the first eigenvector of the matrix $\widehat{\vZ}^{*}\widehat{\vZ}$ by $\vu_0$ and compute the projection values as $\{\sigma_{\lambda}(\langle \vu_0, \vz_{j, i}^{\ell} \rangle)\}_{i, j}$,
where $\sigma_{\lambda}(x)=\begin{cases}x, & \abs{x} \geq \lambda \\ 0, & \abs{x} < \lambda\end{cases}$ is the hard-thresholding function. 
We then select a subset of token representations from $\widehat{\vZ}$ with $\sigma_{\lambda}(\langle \vu_0, \vz_{j, i}^{\ell} \rangle)>0$. 
which correspond to non-zero projection values after thresholding, and we denote this subset as $\widehat{\vZ}_{s}\subseteq \widehat{\vZ}$. 
This selection step is used to remove the background~\citep{oquab2023dinov2}.
We then compute the first three right singular vectors of  $\widehat{\vZ}_{s}$ with the first three eigenvectors of the matrix $\widehat{\vZ}_{s}^{*}\widehat{\vZ}_{s}$ denoted as $\{\vu_1, \vu_2, \vu_3\}$. We define the RGB tuple for each token as:
\begin{equation}
    [r_{j,i}, g_{j,i}, b_{j,i}] = [\langle\vu_1, \vz_{j, i}^{\ell} \rangle, \langle \vu_2, \vz_{j, i}^{\ell} \rangle, \langle \vu_3, \vz_{j, i}^{\ell} \rangle], \quad i \in [N], \, j \in [J], \vz_{j, i}^{\ell}\in \widehat{\vZ}_s.
\end{equation}
Next, for each image $\vX_j$ we compute $\bm{R}_j, \bm{G}_j, \bm{B}_j$, where $\bm{R}_j=[r_{j,1}, \dots, r_{j,N}]^{\top}\in \R^{d\times 1}$ (similar for $\bm{G}_j $ and $\bm{B}_j$). 
Then we reshape the three matrices into $\sqrt{N}\times\sqrt{N}$ and visualize the ``principal components'' of image $\vX_j$ via the RGB image $(\bm{R}_j, \bm{G}_j, \bm{B}_j)\in \R^{3 \times \sqrt{N}\times\sqrt{N}}$. 

\subsubsection{Visualizing Attention Maps}\label{app:attn_maps}

We recapitulate the method to visualize attention maps in \citet{abnar2020quantifying, caron2021emerging}.

For the \(k\th\) head at the \(\ell\th\) layer of the encoder of \ours{}, we compute the \textit{self-attention map} $\vA_{k}^{\ell}\in\R^{N}$ defined as follows:
\begin{equation}\label{eq:def_attention_map_crate}
    \vA_{k}^{\ell}= \mat{A_{k, 1}^{\ell} \\ \vdots \\ A_{k, N}^{\ell}} \in \R^{N}, \quad \text{where} \quad A_{k, i}^{\ell} = \frac{\exp(\langle \vU_{k}^{\ell*}\vz_{i}^{\ell}, \vU_{k}^{\ell*}\vz_{\cls}^{\ell}\rangle)}{\sum_{j=1}^{N}\exp(\langle \vU_{k}^{\ell*}\vz_{j}^{\ell}, \vU_{k}^{\ell*}\vz_{\cls}^{\ell}\rangle)}.   
\end{equation}
where $\vz_{\cls}^{\ell}$ is the \(\ell\th\) layer representation of the class token.

For each image, we reshape the attention matrix \(\vA_{k}^{L - 1}\) for the penultimate layer \(L - 1\) into a \(\sqrt{N}\times\sqrt{N}\) matrix and visualize the heatmaps as shown in \Cref{fig:attention_maps}. For example, the \(i\th\) row and the \(j\th\) column element of each heatmap in \Cref{fig:attention_maps} corresponds to the \(m\th\) component of \(\vA_{k}^{\ell}\), where \(m = (i - 1) \cdot \sqrt{N} + j\). 
In \Cref{fig:attention_maps}, for each image we select one attention head \(k\) of \ours{} and visualize the attention matrix \(\vA_{k}^{L - 1}\).

\subsection{Additional Experiments} \label{app:additional_experiments}

In this section we perform more experiments to explore properties of the \ours{} architecture. 

\begin{table*}[t!]
    \centering
    \caption{\small \textbf{Top-1 classification accuracy of \ours{-Base}
        when pre-trained on ImageNet-1K and fine-tuned on classification for various datasets.} We compare a fine-tuned model which was pre-trained on the MAE task with a model trained from scratch on classification (``\textit{random init}'') using exactly the same experimental conditions. Our results show that the representation learning occurring during pre-training substantially improves performance on downstream tasks. 
    }
    \label{tab:comparison_with_random_init}
    \small
    \setlength{\tabcolsep}{13.6pt}
        \begin{tabular}{@{}lccccc@{}}
            \toprule
            \textbf{Classification Performance} & CIFAR 10 & CIFAR 100 & Oxford Flowers & Oxford-Pets
            \\ 
            \midrule
            \midrule
            \ours{-Base} (trained) & 96.8 & 80.3 &
            78.5 & 76.7 \\
            \ours{-Base} (random init) & 85.1 & 58.8 & 38.0  & 28.8 \\ 
            \bottomrule
        \end{tabular}%
\end{table*}

First, in \Cref{tab:comparison_with_random_init}, we compare the fine-tuning performance of an \ours{-Base} model with MAE pretraining against an \ours{-Base} model with no pretraining at all (i.e., randomly initialized). We apply the same fine-tuning process to both models, and we observe a massive disparity in performance, where the pre-trained model succeeds while the randomly initialized model performs poorly. This indicates that the organized representations of a pre-trained \ours{-Base} model are a strong starting point when fine-tuning for downstream tasks.

\begin{table*}[t!]
    \centering
    \caption{\small \textbf{Average reconstruction loss over the training and validation sets of ImageNet-1K for both \ours{-Base} and ViT-MAE-Base}. We see that the performance of \ours{-Base}, while a bit worse than ViT-MAE-Base, obtains promising performance on the challenging masked autoencoding task.
    }
    \label{tab:reconstruction_loss}
    \small
    \setlength{\tabcolsep}{13.6pt}
        \begin{tabular}{@{}lcc@{}}
            \toprule
            \textbf{Reconstruction Loss} & Training Loss & Validation Loss 
            \\ 
            \midrule
            \midrule
            \ours{-Base} & 0.265 & 0.302 \\
            ViT-MAE-Base & 0.240 & 0.267 
            \\
            \bottomrule
        \end{tabular}%
\end{table*}

Next, in \Cref{tab:reconstruction_loss}, we evaluate the reconstruction loss (measured in mean-squared error) of \ours{-Base} versus ViT-MAE-Base, evaluated on the training and test sets of ImageNet-1K. We observe that while \ours{-Base} performs slightly worse at masked reconstruction, the performance is still very reasonable.

\begin{table*}[t!]
    \centering
    \caption{\small \textbf{Top-1 classification accuracy of \ours{-Base} when pre-trained on ImageNet-1K and linear probed for classification on CIFAR-10, when pre-trained using different mask percentage (i.e., number of masked tokens in each sample).} This shows that \ours{} models with 75\% of the tokens masked during training tend to have the most structured representations that are useful for downstream tasks, an empirical conclusion that echoes \citet{he2022masked}, but a wide range of mask percentages result in good representations. (\textit{Note:} This table uses the hyperparameter setting \(\lambda = 5.0\).)
    }
    \label{tab:performance_for_different_mask_sizes}
    \small
    \setlength{\tabcolsep}{13.6pt}
        \begin{tabular}{@{}lcccc@{}}
            \toprule 
            \textbf{Classification Accuracy} & 25\% Masked & 50\% Masked & 75\% Masked & 90\% Masked \\ 
            \midrule 
            \midrule 
            \ours{-Base} & 69.78 & 75.97 & 75.99 & 73.45 \\
            \bottomrule
        \end{tabular}%
\end{table*}

In \Cref{tab:performance_for_different_mask_sizes}, we check the downstream performance and feature learning of \ours{-Base} models when varying the mask size. We use test accuracy of linear probing on CIFAR10 as a proxy for the quality of the learned features. Our results show that the performance is maximized when the number of masked tokens is 75\% of the number of total tokens, i.e., when 75\% of tokens are maxed out. This confirms the experiments in \citet{he2022masked}. 

\begin{table*}[t!]
    \centering
    \caption{\small \textbf{Top-1 classification accuracy of \ours{-Base} when pre-trained on ImageNet-1K and linear probed for classification on various datasets, when pre-trained using different \(\lambda\).} This shows that \ours{} models perform reasonably well at \ours{-Base} scale across different values of \(\lambda\).
    }
    \label{tab:performance_for_different_lambda}
    \small
    \setlength{\tabcolsep}{13.6pt}
        \begin{tabular}{@{}lccc@{}}
            \toprule 
            \textbf{Classification Accuracy} & \(\lambda = 0.1\) & \(\lambda = 0.5\) & \(\lambda = 5.0\) \\ 
            \midrule 
            \midrule 
            \ours{-Base} & 83.33 & 80.87 & 75.99 \\
            \bottomrule
        \end{tabular}%
\end{table*}

In \Cref{tab:performance_for_different_lambda}, we check the downstream performance and feature learning of \ours{-Base} models when varying the hyperparameter \(\lambda\). We again use test accuracy of linear probing on CIFAR10 as a proxy for the quality of the learned features. Our results show that the performance is maximized when \(\lambda = 0.1\), but all models perform reasonably well at \ours{-Base} scale.

\begin{figure}[t!]
    \centering
    \includegraphics[width=\textwidth]{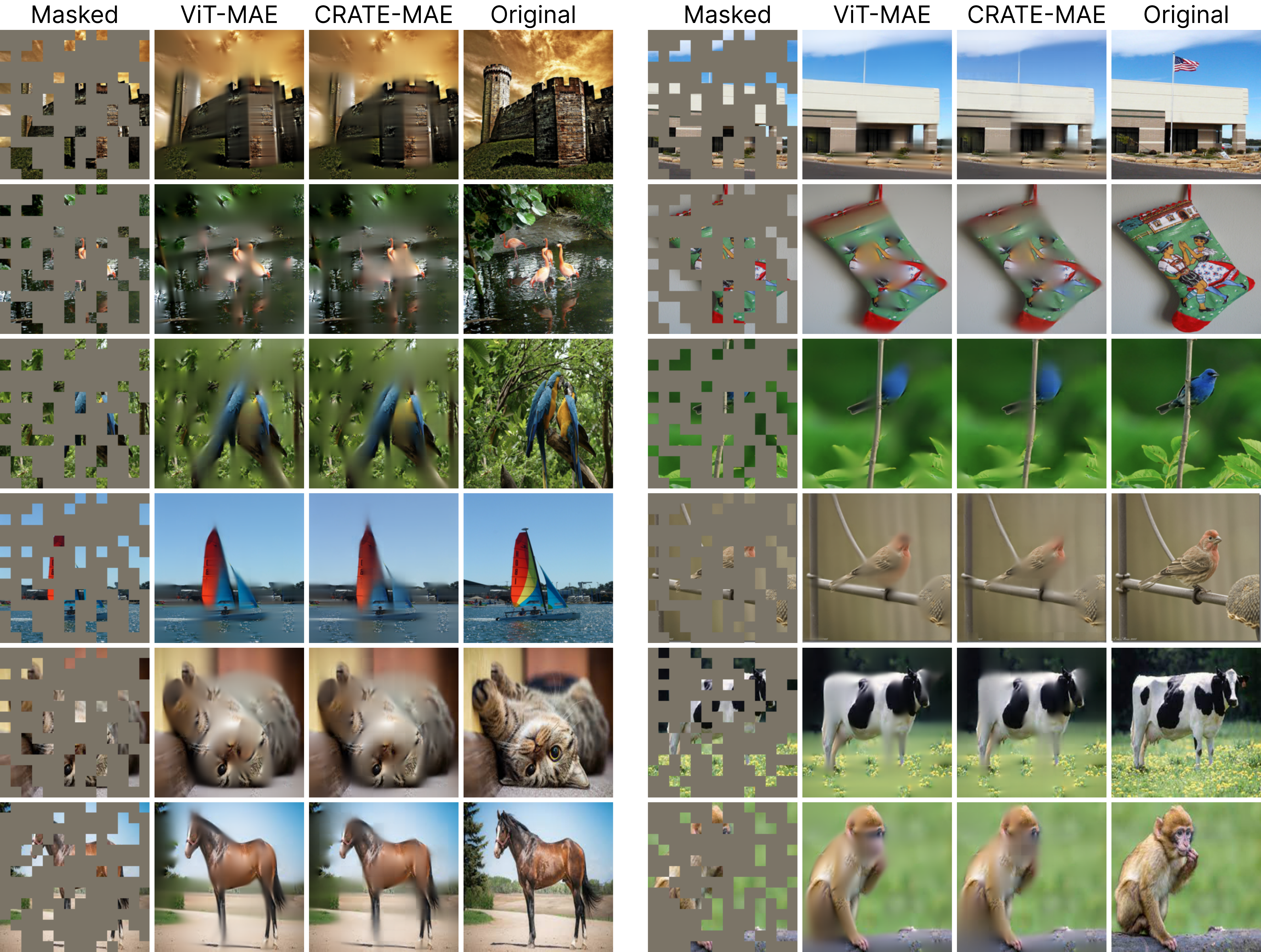}
    \caption{\small \textbf{More instances of parity between \ours{-Base} and ViT-MAE-Base in the masked autoencoding task.} Echoing the message of \Cref{fig:mae_autoencoding}, we find that \ours{-Base} and ViT-MAE-Base have similar performance on the masked autoencoding task, even as the \ours{-Base} model is significantly more parameter-efficient.}
    \label{fig:mae_autoencoding_extended}
\end{figure}

In \Cref{fig:mae_autoencoding_extended} we provide more examples of the masked autoencoding efficacy of both \ours{-Base} and ViT-MAE-Base. Our results confirm those of \Cref{fig:mae_autoencoding}, namely \ours{-Base} achieves parity with the much larger ViT-MAE-Base on the challenging masked autoencoding task.

\begin{figure}[t!]
    \centering 
    \includegraphics[width=\textwidth]{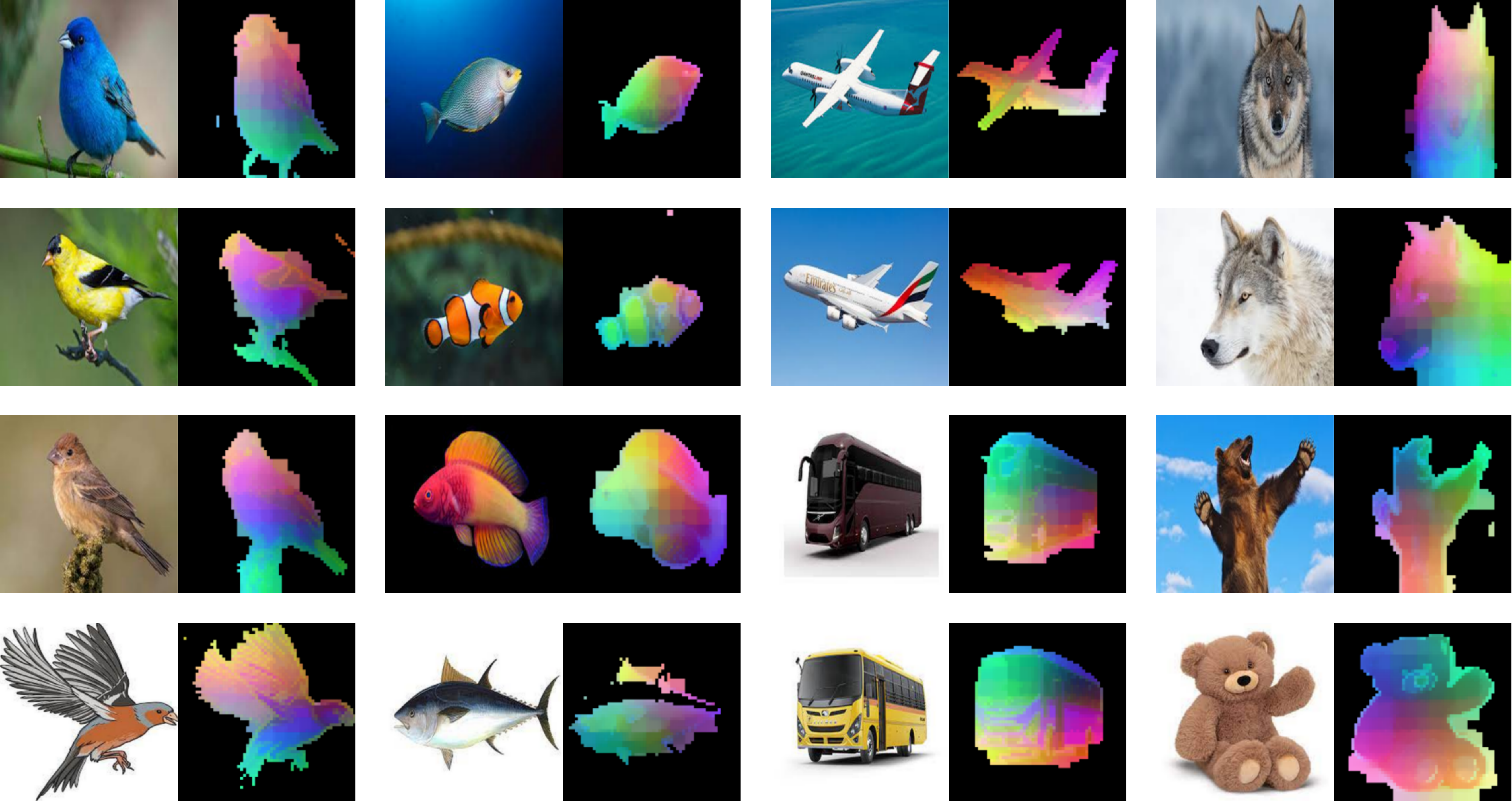}
    \caption{\small \textbf{More examples of linearized representations in \ours{-Base}.} Echoing the message of \Cref{fig:pca}, we find that \ours{-Base} has a linear feature space. In particular, the first three principal components strongly correlate with the main semantic content of the image.}
    \label{fig:pca_full}
\end{figure}

In \Cref{fig:pca_full}, we provide more examples of the linearity of the representations within \ours{-Base}. We see that over a wide variety of images, the first three principal components of each class correlate strongly with semantically meaningful patches of the input image. These results extend \Cref{fig:pca}.

\begin{figure}[t!]
    \centering 
    \includegraphics[width=\textwidth]{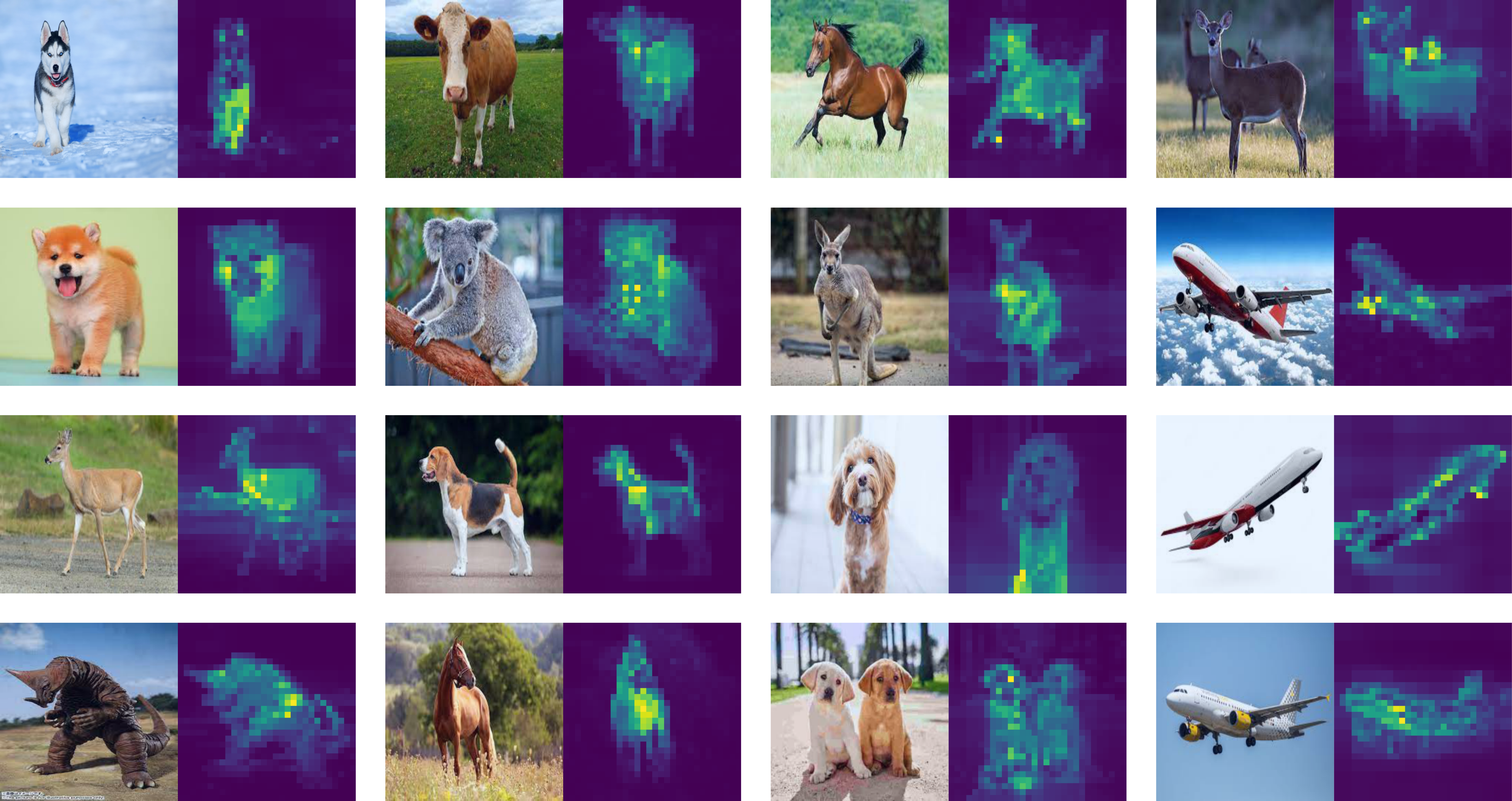}
    \caption{\small \textbf{More examples of interpretable attention maps in \ours{-Base}.} Echoing the message of \Cref{fig:attention_maps}, we find that \ours{-Base} has human-interpretable attention maps which semantically segment the foreground of input images.}
    \label{fig:attention_maps_full}
\end{figure}

In \Cref{fig:attention_maps_full}, we provide more examples of the interpretability of selected attention outputs within \ours{-Base}. We see that over a wide variety of images, the attention map succeeds at capturing many semantics of the input image. These results extend \Cref{fig:attention_maps}.

\begin{figure}[t!]
    \centering 
    \begin{subfigure}{0.475\textwidth}
        \centering 
        \includegraphics[width=\textwidth]{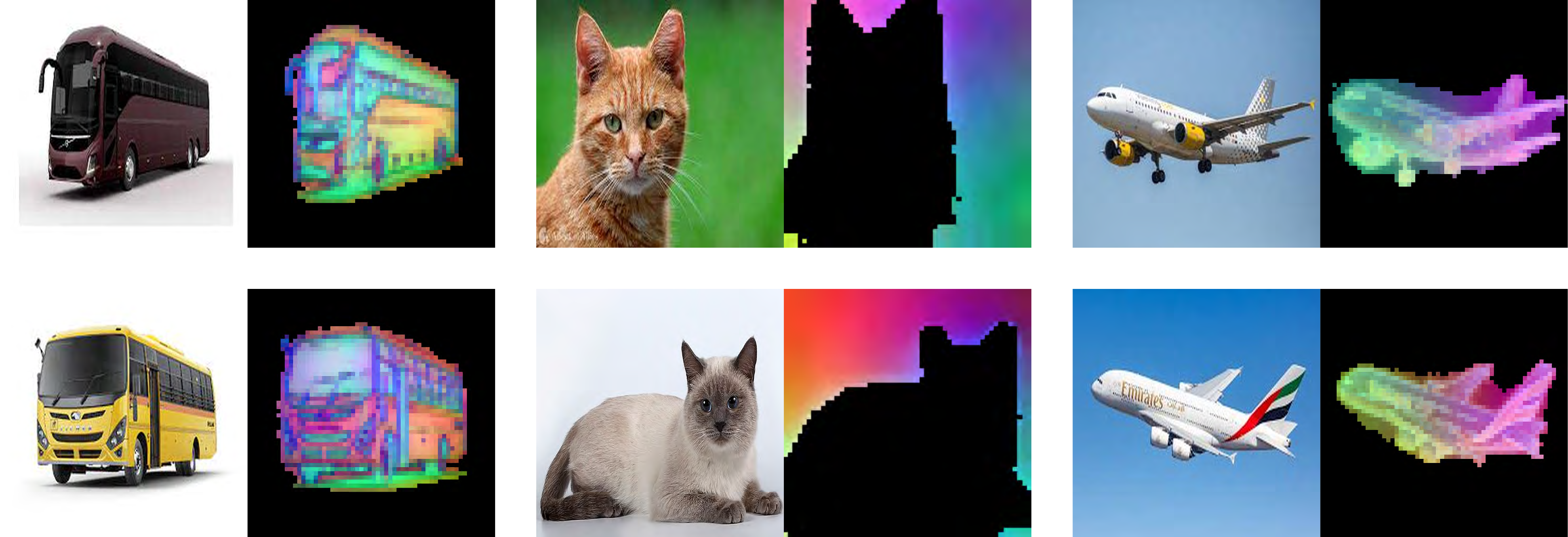}
        \caption{Visualizing PCA of ViT-MAE-Base.}
        \label{fig:vit_pca}
    \end{subfigure}
    \hspace{0.025\textwidth}
    \begin{subfigure}{0.475\textwidth}
        \centering 
        \includegraphics[width=\textwidth]{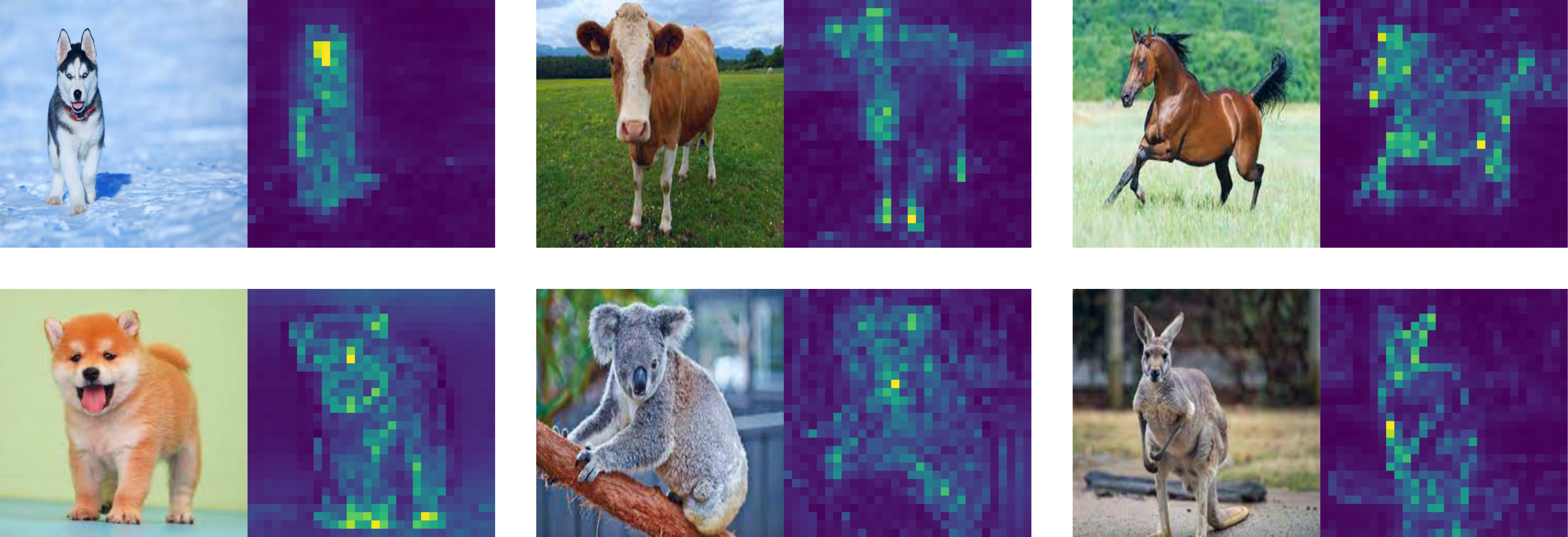}
        \caption{Visualizing attention maps of ViT-MAE-Base.}
        \label{fig:vit_attention}
    \end{subfigure}
    \caption{\small \textbf{A comparison of CRATE to ViT-MAE in the setting of \Cref{fig:pca_attn}.} We use the visualizations of PCA on the token representations and attention maps, introduced in \Cref{fig:pca_attn}, to qualitatively evaluate the representation quality of the ViT-MAE-Base. By comparing this figure (\Cref{fig:vit_pca_attn}) and \Cref{fig:pca_attn}, we observe that \ours{-Base} attention maps contain more clear semantics than those from ViT-MAE-Base, while both \ours{-Base} and ViT-MAE-Base have nearly-linear representation spaces wherein semantic concepts correspond to the first three principal components.}
    \label{fig:vit_pca_attn}
\end{figure}

In \Cref{fig:vit_pca_attn} we apply (nearly) the same methodology involved in constructing \Cref{fig:pca_attn}, a qualitative demonstration of the feature quality of the \ours{-Base} encoder, to evaluate the feature quality of the ViT-MAE-Base encoder. One necessary difference is that to evaluate the attention maps, the ViT does not have the \(\vU_{[K]}^{\ell}\) matrices that CRATE does, but instead has three sets of matrices \(\vQ_{[K]}^{\ell}\), \(\vK_{[K]}^{\ell}\), and \(\vV_{[K]}^{\ell}\); thus, we construct the attention maps via the following equation:
\begin{equation}
    A_{k, i}^{\ell} = \frac{\exp(\langle \vK_{k}^{\ell}\vz_{i}^{\ell}, \vQ_{k}^{\ell}\vz_{\cls}^{\ell} \rangle)}{\sum_{j = 1}^{N}\exp(\langle \vK_{k}^{\ell}\vz_{j}^{\ell}, \vQ_{k}^{\ell}\vz_{\cls}^{\ell} \rangle)}.
\end{equation}
Overall, \Cref{fig:pca_attn} and \Cref{fig:vit_pca_attn} demonstrate that, at least empirically, the attention semantics in \ours{-Base} are significantly better and clearer than ViT-MAE-Base. The reason that \ours{} models have semantically meaningful attention maps may be due to our white-box design, namely the fact that in \ours{} we have \(\vQ_{[K]}^{\ell} = \vK_{[K]}^{\ell} = \vV_{[K]}^{\ell} = \vU_{[K]}^{\ell *}\); indeed, the fact that setting \(\vQ_{[K]}^{\ell} = \vK_{[K]}^{\ell} = \vV_{[K]}^{\ell} = \vU_{[K]}^{\ell *}\) yields semantically meaningful attention maps has been shown in other work, albeit in a different setting \citep{yu2023emergence}. The reason that the ViT-MAE has semantically meaningful attention maps, albeit not as clear and worse than \ours{}, may be due to several different factors such as using the class token as a register \citep{darcet2023vision}. Nevertheless, the features in both models have linear structures, or at least each class' representations have three principal components which correlate strongly with the semantics of the class.

\end{document}